\documentclass[journal,12pt,cspaper,draftclsnofoot,onecolumn]{IEEEtran}

\usepackage{times}
\usepackage[ruled]{algorithm2e}
\usepackage{epsfig}
\usepackage{graphicx}
\usepackage{dsfont}
\usepackage{amsthm}
\usepackage[cmex10]{amsmath}
\usepackage{amssymb}
\usepackage{fixltx2e}

\ifCLASSOPTIONcompsoc
  \usepackage[caption=false,font=normalsize,labelfont=sf,textfont=sf]{subfig}
\else
  \usepackage[caption=false,font=footnotesize]{subfig}
\fi

\usepackage[nocompress]{cite}

\usepackage{color}
\usepackage{multirow}
\usepackage{mdwmath}
\usepackage{mdwtab}
\usepackage{booktabs}

\newtheorem{theorem}{Theorem}

\newtheorem{corollary}[theorem]{Corollary}

\newenvironment{definition}[1][Definition]{\begin{trivlist}
\item[\hskip \labelsep {\bfseries #1}]}{\end{trivlist}}

\begin{document}

\title{Rigid and Non-rigid Shape Evolutions for Shape Alignment and Recovery in Images}

\author{Junyan Wang*~\IEEEmembership{Member,~IEEE,}~and~Kap Luk Chan~\IEEEmembership{Member,~IEEE,}
\IEEEcompsocitemizethanks{\IEEEcompsocthanksitem Junyan Wang and Kap Luk Chan are with School of Electrical \& Electronic Engineering, Nanyang Technological University, Singapore.\protect\\
E-mail: \{wa0009an,eklchan\}@ntu.edu.sg
}
\thanks{}}

\IEEEcompsoctitleabstractindextext{%
\begin{abstract}
The same type of objects in different images may vary in their shapes because of rigid and non-rigid shape deformations, occluding foreground as well as cluttered background. The problem concerned in this work is the shape extraction in such challenging situations. We approach the shape extraction through shape alignment and recovery. This paper presents a novel and general method for shape alignment and recovery by using one example shapes based on deterministic energy minimization. Our idea is to use general model of shape deformation in minimizing active contour energies. Given \emph{a priori} form of the shape deformation, we show how the curve evolution equation corresponding to the shape deformation can be derived. The curve evolution is called the prior variation shape evolution (PVSE). We also derive the energy-minimizing PVSE for minimizing active contour energies. For shape recovery, we propose to use the PVSE that deforms the shape while preserving its shape characteristics. For choosing such shape-preserving PVSE, a theory of shape preservability of the PVSE is established. Experimental results validate the theory and the formulations, and they demonstrate the effectiveness of our method.
\end{abstract}
}
\maketitle
\IEEEdisplaynotcompsoctitleabstractindextext
\section{Introduction}
Object segmentation has been formulated as a contour optimization problem, which is known as the active contour model. There exist active contours based on various criteria such as edge detection \cite{kass88snakes}, region grouping by regional homogeneity \cite{Zhusongchun96RegComp} or regional dissimilarity \cite{Baris06GPAC} etc. Most of the early active contour formulations did not incorporate shape prior modeling. It has been observed that such active contour method can segment the object of interest based on the detectable boundary or region of the object, but it is not able to recover the object shape if partial of the object boundary or region is undetectable because of, for example, occlusion or scene clutters. The accurate object shapes are desired in many vision applications requiring shape analysis, such as the medical image analysis \cite{Leventon00Statisticalshape} and tracking \cite{Cremers06LStracking}.

In real images, the occlusion and scene clutters can cause missing parts on the objects of interest, or merging with surrounding objects. The missing-part phenomenon is illustrated in Figs. \ref{Fig:Intro_Ocl} (d-f), and the object-overlap phenomenon is illustrated in \ref{Fig:Intro_Exp} (d-f). Fig. \ref{Fig:GenSitu} shows that the object of interest is surrounded by scene clusters (Fig. \ref{Fig:GenSitu}(a)) or is also partially occluded (Fig. \ref{Fig:GenSitu}(b)). The objects of interest are all deformable. We expect to obtain the underlying shape silhouettes in Figs \ref{Fig:Intro_Ocl} (a-c)  and \ref{Fig:Intro_Exp} (a-c) when only given the observed shapes in Figs. \ref{Fig:Intro_Ocl} (d-f) and \ref{Fig:Intro_Exp} (d-f). Besides, we desire this to be achieved in real images with complex foregrounds or backgrounds. This task may be achieved by aligning a shape template, i.e. the prior shape model, to the object of interest in the image, which we refer to as a process of \emph{shape registration}. However, during the conventional shape registration process for deformable object, such as in \cite{Charpiat2007GG}, the shape of the template is allowed to change but it may not be able to recover the object shape if there are missing parts or shape overlaps. In this paper, we deal with simultaneous \emph{shape registration and recovery}. Fig. \ref{Fig:HorsRegSys} illustrates such a desired process in which the deformed horse overlapped with its rider is registered and recovered using the selected prototype shape(s). Complex forgrounds/backgrounds are also considered.
\begin{figure}[t]
\centering \subfloat[]{\label{Fig:Monkey_org}\includegraphics[height=0.8in]{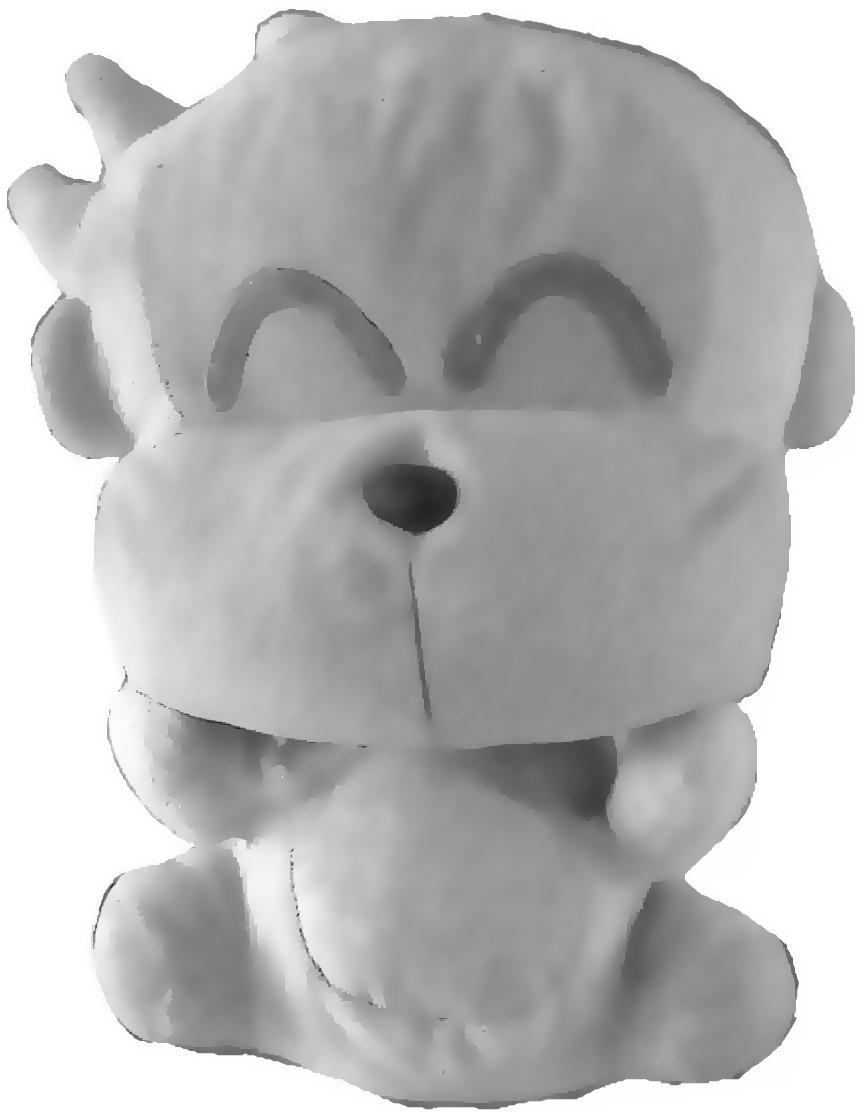}}\hspace{0.33in}
\subfloat[]{\label{Fig:RTMonkey}\includegraphics[height=0.8in]{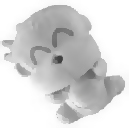}}\hspace{0.33in}
\subfloat[]{\label{Fig:NRTMonkey}\includegraphics[height=0.8in]{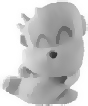}}\\
~~\subfloat[]{\label{Fig:OCMonkey}\includegraphics[height=0.8in]{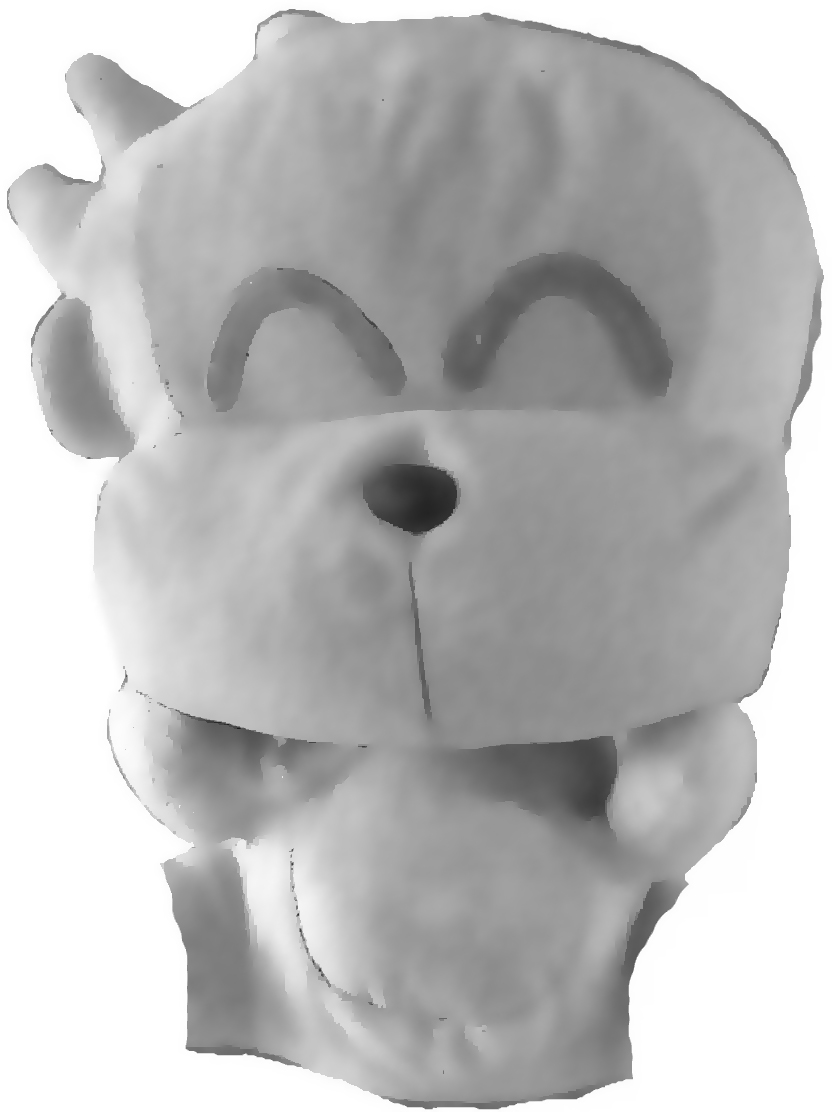}}\hspace{0.33in}
\subfloat[]{\label{Fig:RTOCMonkey}\includegraphics[height=0.8in]{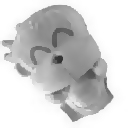}}\hspace{0.33in}
\subfloat[]{\label{Fig:NRTOCMonkey}\includegraphics[height=0.8in]{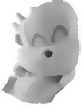}}
\caption{(a) A toy monkey; (b) Toy monkey after a rigid transformation; (c)
Toy monkey after a non-rigid transformation; (d) (e) and (f) are the toy monkeys in (a) (b) and (c) with parts missing.}\label{Fig:Intro_Ocl}
\end{figure}
\begin{figure}[t]
\centering \subfloat[]{\label{Fig:Horse_org}\includegraphics[height=0.8in]{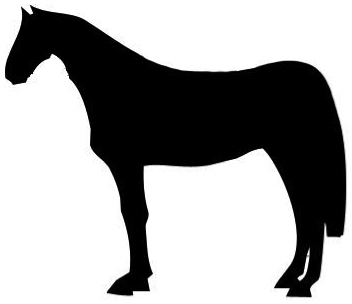}}\hspace{0.33in}
\subfloat[]{\label{Fig:RTHorse}\includegraphics[height=0.8in]{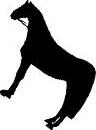}}\hspace{0.33in}
\subfloat[]{\label{Fig:NRTHorse}\includegraphics[height=0.8in]{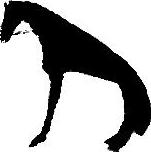}}\\
\subfloat[]{\label{Fig:ExpHorse}\includegraphics[height=0.8in]{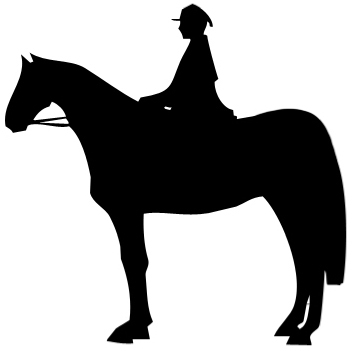}}\hspace{0.33in}
\subfloat[]{\label{Fig:RTExpHorse}\includegraphics[height=0.8in]{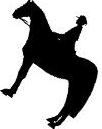}}\hspace{0.33in}
\subfloat[]{\label{Fig:NRTExpHorse}\includegraphics[height=0.8in]{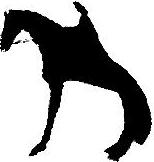}}
\caption{(a) A horse;
(b) The rigidly transformed horse; (c) The non-rigidly transformed horse; (d) (e) and (f) are the horses in (a) (b) and (c) merged with the rider.}\label{Fig:Intro_Exp}
\end{figure}
\begin{figure}[t]
\centering
  \subfloat[]{\includegraphics[height=2cm]{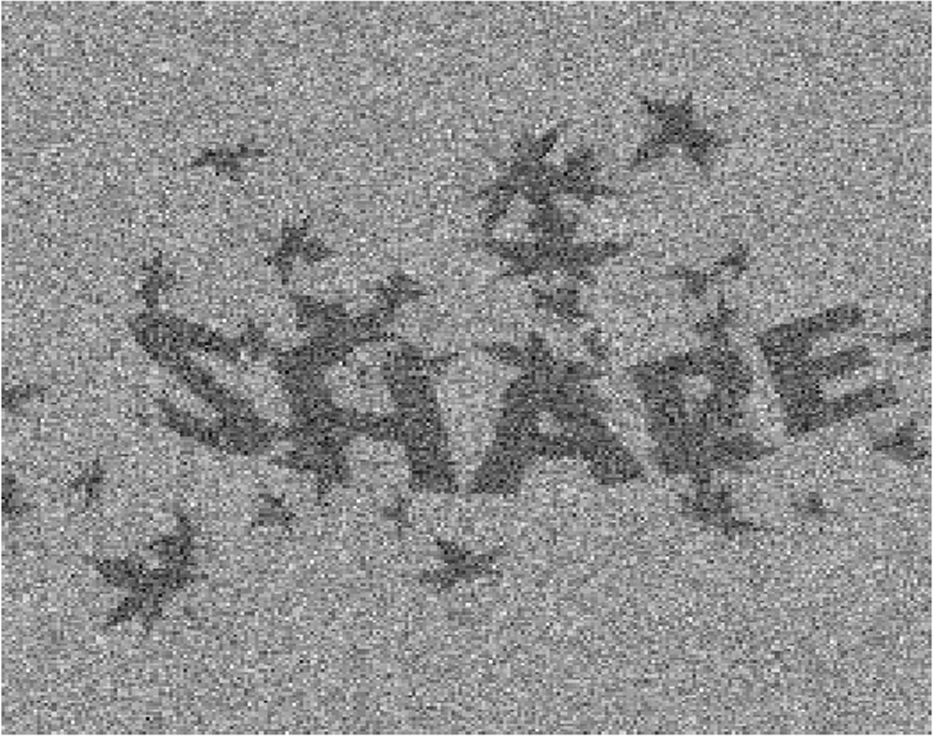}}~\subfloat[]{\includegraphics[height=2cm]{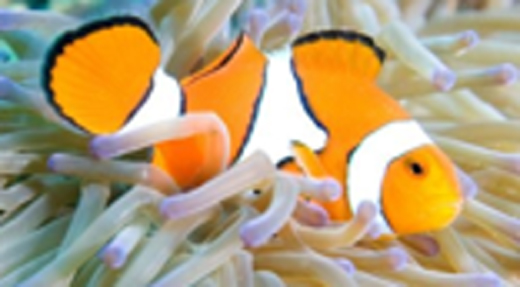}}
  \caption{More general situations}\label{Fig:GenSitu}
\end{figure}

\begin{figure}[t]
\centering
  \includegraphics[width=0.8\columnwidth]{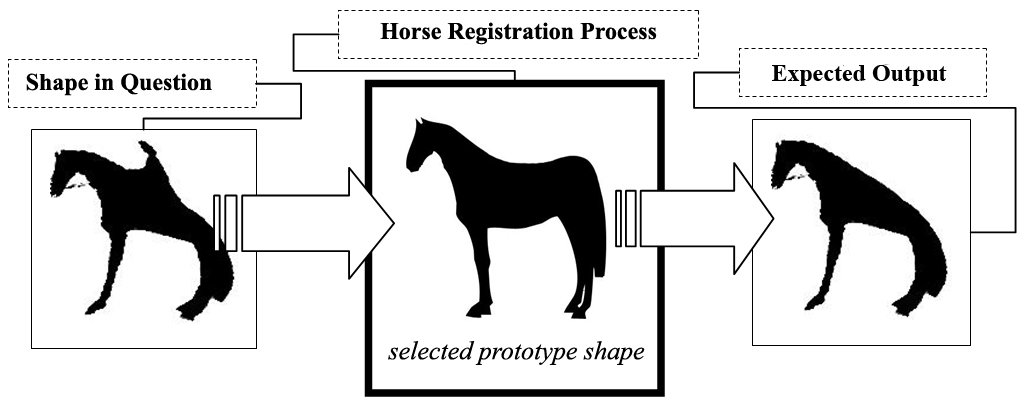}\\
  \caption{A horse registration and recovery process}\label{Fig:HorsRegSys}
\end{figure}

In many recent works \cite{Chenym02Usingprior, Leventon00Statisticalshape, Rousson02Shapepriors, Cremers03Shapestatistics, Cremers04kerneldensity, Etyngier07Shapepriors, Prisacariu2011GPLVM_LSSP, Dambreville09KPCAShape, Joshi09BAC}, people have shown that the active contour frameworks with shape prior modeling can achieve the shape registration that is capable of recovering shapes from those with missing parts and overlapping shapes. The main challenge lies in the modeling of shape prior. Some of the early works addressed the simple shape transformation, such as the similarity transformation \cite{Chenym02Usingprior}. General non-rigid shape deformations were handled by using the statistical shape prior model built with the data obtained in advance \cite{Leventon00Statisticalshape,Rousson02Shapepriors,Cremers04kerneldensity,Cremers03Shapestatistics,Prisacariu2011GPLVM_LSSP,Dambreville09KPCAShape,Joshi09BAC}. We shall call these models the \emph{data-based} shape prior models. The data-based shape prior modeling requires a large shape dataset for training. Since the large shape dataset may be unavailable in real applications, we propose a \emph{knowledge-based} shape prior model as an alternative to the conventional \emph{data-based} shape prior models to remove the hard requirement of training data. The resultant method can achieve the shape registration and recovery in images in the presence of rigid and non-rigid shape deformations, missing parts and overlapping shapes by using a limited number, e.g. $1$, of example shapes. Due to its difficulties, the research efforts made to this task is sparse.

We propose to adopt the general rigid or non-rigid shape deformations expressed in closed form as the knowledge-based shape prior model. Such model is also known as the shape warping model. However, it is unknown how to use the general shape warping for minimizing the active contour energies. Suppose the form of the shape warping is known \emph{a priori}, we show in this paper how the curve evolution equation corresponding to the shape warping can be derived. The resultant novel curve evolution equation is called the \emph{Prior Variation Shape Evolution} (PVSE) equation. We then derive the energy-minimizing PVSE equation for minimizing general active contour energies. This novel general derivation is named the calculus of prior variations. We use the energy-minimizing PVSE to achieve shape registration for extracting rigidly and non-rigidly deformed objects in images. Moreover, we formulate a theory of shape preservability for selecting a shape-preserving PVSE whereby shape recovery can be achieved additionally. According to our theory, the PVSE corresponding to a well-known non-rigid shape warping model is shape-preserving, and this PVSE is used for shape extraction in the work. This paper extends the work reported in \cite{Wang09PVCE} in several aspects. There are another two important contributions in this paper. First, we showed that the PVSE equation can be derived from a standard model of shape warping. This can be viewed as an interpretation of the PVSE. Second, we present a geometrical theory on shape preservability. This theory explains why the non-rigid PVSE in the chosen form can achieve shape recovery.

An interesting method closely related to our method was proposed by Charpiat et al. \cite{Charpiat2007GG}. However, their work did not address the challenging issues on modeling the deformations of the objects of interest. Thus, it is unclear how their method can be used for handling the deformations of the objects.

The rest of this paper is organized as follows. In section \ref{SEC:REL}, we briefly review the related and inspiring works. In section \ref{SEC:Derive_PVSE}, we derive the PVSE equation from shape warping. In section \ref{SEC:CoPV}, we derive the energy-minimizing PVSE, and we also apply the derivation of the energy-minimizing PVSE to general active contour energy. In section \ref{SEC:Theory_ShaPre}, we present a theory of shape preservability of PVSE. The theory used to justify the PVSE model of non-rigid shape deformations. In section \ref{SEC:Exp}, we evaluate the PVSE for modeling shape deformation, and we also evaluate the PVSE for shape registration and recovery on images. Both qualitative and quantitative results are reported to validate our formulations and algorithm. We conclude the paper in section \ref{SEC:Con} with discussions.

\section{Related works}\label{SEC:REL}

\subsection{Active contour with shape priors}

In the literature of shape prior based active contours, simple shape transformations, such as similarity, affine and projective shape transformations are modeled by using the knowledge-based shape prior models. Among the pioneering works, Chen {et al.} \cite{Chenym02Usingprior} proposed Geodesic Active Contour model with a prior term of scaling-, rotation- and translation-invariant shape dissimilarity for simultaneous shape registration and recovery. Cremers {et al.} \cite{Cremers03Shapestatistics,Cremers04kerneldensity} adopted a similar approach. However, only the similarity transform was considered in these works. There are existing works on constraining the contour motion in active contour by parametric Euclidean transformation \cite{Mansouri04LieActiveContour}, affine and projective \cite{Mukherjee07LieAC_Affine_Projective} transformations.

The non-rigid shape variations were handled by using data-based shape prior models \cite{Leventon00Statisticalshape, Cremers03Shapestatistics,Cremers04kerneldensity,Etyngier07Shapepriors,Joshi09BAC}. Leventon {et al.} \cite{Leventon00Statisticalshape} modeled shape variations using Principal Component Analysis (PCA), and they measured the shape dissimilarity by shape distance from a mean shape. More recently, Etyngier {et al.} \cite{Etyngier07Shapepriors} formulated the space of prior shapes for curve evolution as a shape manifold to which the well established nonlinear dimensionality reduction techniques are applicable. Some other works on manifold based shape prior modeling in active contour framework have been reported in \cite{Prisacariu2011GPLVM_LSSP,Dambreville09KPCAShape,Joshi09BAC}.

The data-based prior shape models generally require a dataset of training shapes, and every single reliable sample silhouette shape is to be delineated manually to constitute the dataset, which is laborious, making this approach inconvenient and even inapplicable in practice. An immediate advantage of the knowledge-based formulation in our context is that it allows the user to determine the mathematical form of the deformation based on the physical properties of the object or domain-specific prior knowledge. Thus, it removes the hard requirement of training data in the modeling phase. The problem with the existing deterministic models for active contours is mainly that only simple transformations, such as similarity, affine and projective transforms, have been considered, while the large variety of non-rigid shape deformations have yet to be dealt with.

\subsection{Shape warping as a deformable shape model}
Shape warping is a general knowledge-based deformable shape model \cite{Amit91StructuralImage,Jain96Objectmatching,Glasbey98Areviewofimagewarping}. In shape warping, the shape deformation is viewed as a result of point displacement on the shape. The point displacement is expressed as a warping mapping. Similarity and affine transformations are special cases of warping mapping. Warping mapping has been used for modeling non-rigid shape deformations in the framework of \emph{deformable template matching} for locating object boundary \cite{Jain96Objectmatching}. Warping mapping can represent both rigid and non-rigid deformations. Comparing with the aforementioned data-based shape prior models, the advantage of warping is that it does not require training samples. Only one example shape is required. However, the segmentation model in the classic deformable template matching were based on local image features, such as edges \cite{Jain96Objectmatching,Felzenszwalb05DeformableModel} and Gabor features \cite{Wu07ActiveBasis}, which could be restrictive for shape registration in images. Besides, random optimization schemes were often used to jump out the too many local optimal solutions. Consequently, the computations for a single image can be impractical, and the convergence is difficult to determine.

For effective shape registration in images, we require a better segmentation model, and for the sake of efficiency we require deterministic algorithms. Active contour models are advanced optimization models for segmentation. However it was unknown how active contour energies can be minimized deterministically via general rigid and non-rigid warping. It was also unknown how the general warping can deal with missing parts and shape overlapping.

\subsection{Active contour and shape warping}
Charpiat et al. \cite{Charpiat2007GG} developed a method for minimizing active contour energy by incorporating warping mappings. Their method was proposed based on an interesting observation that the functional energy could be decreased by using a generalized negative gradient. The generalized negative gradient could be induced by a positive definite, i.e. Riemannian, metric. Charpiat et al. showed how the positive definite metric could be constructed in an ad-hoc way by using the original $L_2$ functional gradients and priors. The similarity transformation, affine transformation and \emph{Sobolev} norm were used to construct the generalized gradients in \cite{Charpiat2007GG}. In a similar work \cite{Mansouri04LieActiveContour}, Mansouri et al. considered shape evolution of Lie transformation groups.

Nevertheless, these works did not address the challenging issues on modeling of general, especially non-rigid, shape deformations. It has not been shown how the generalized gradients induced by positive metrics, proposed by Charpiat et al. \cite{Charpiat2007GG}, can model the deformation of a given object shape. The derivations by Mansouri et al. \cite{Mansouri04LieActiveContour} were only suitable for Lie transformation groups. On the contrary, we aim at a sound and universal method for modeling general shape deformations of a given object and for shape registration in images. Moreover, we approach the shape recovery through a notion of shape preservability of shape deformations, which have not been considered previously.

\section{Deriving curve evolution equation from shape warping}\label{SEC:Derive_PVSE}
The shape warping is considered as our model of shape deformation. In a real application, the detailed shape warping model for a given object can be formulated in advance. This section presents the derivation of the curve evolution equation from a given shape warping represented in a general form.
%

\subsection{Infinitesimal warping iteration}
This subsection presents our general formulation of shape warping. We require the mathematical definition of the warping mapping before presenting our model of shape warping.
\begin{definition} [Warping mapping]\label{DEF:ID}
\emph{The warping mapping is a differentiable parametric function mapping $\mathbf{\tilde{x}}=\mathbf{f}(\mathbf{x},\theta)$, where $\mathbf{x}\in\mathds{R}^2$ is the image coordinate, $\theta\in\mathds{R}^n$ is the vector of warping parameter, and $\mathbf{f}:\mathds{R}^{2+n}\mapsto\mathds{R}^2$, such that the \emph{identity} property, i.e. $\mathbf{x}=\mathbf{f}(\mathbf{x},\mathbf{0})$, holds.}
\end{definition}
The identity property is necessary for the warping mapping due to the following \emph{consistency} property according to the mean value theorem:
\begin{equation}
\begin{split}
\Big\|\mathbf{f}(\mathbf{x},\theta)-\mathbf{x}\Big\| &= \Big\|\mathbf{f}(\mathbf{x},\theta)-\mathbf{f}(\mathbf{x},\mathbf{0})\Big\|\\
&\leq \big\|\nabla\mathbf{f}(\mathbf{x},\alpha\theta+(1-\alpha)\mathbf{0})\big\| \big\|\theta\big\|\\
&\leq \max_{0\leq\alpha\leq1}\left(\big\|\nabla\mathbf{f}(\mathbf{x},\alpha\theta)\big\|\right) \big\|\theta\big\|,
\end{split}
\end{equation}
which ensures that small parameters yield small shape deformations. We shall not impose the invertibility of the warping mapping. This makes our model more general than the Lie group transformations/diffeomorphisms. The nice closure property of Lie group may not be needed in the real situations.

Let us consider an evolving contour shape $C$ defined by the points $\mathbf{x}(p,t)$ on the shape, where $p$ is the parametrization along the curve and $t$ is the an artificial time. We may interchangeably use $C(p,t)$ or $\mathbf{x}(p,t)$ to denote a point on the curve at a specific time henceforth. A shape deformation due to warping within an infinitesimal time interval $\tau$ at a fixed time $t'$ can be defined as follows
\begin{equation}\label{EQ:IWI}
\mathbf{x}(p,t'+\tau)=\mathbf{f}(\mathbf{x}(p,t'),\theta(\tau)),~ \|\theta(\tau)\|\leq\epsilon.
\end{equation}

We call Eq. (\ref{EQ:IWI}) the \emph{infinitesimal warping iteration}. To follow the identity property of warping, we require $\theta(0)=0$. Besides, $\epsilon$ is also an infinitesimal positive constant. The model tells that the current shape is generated by warping the shape in the near past. This general deformable shape model can represent a large variety of shape variations according to the relatively arbitrary form of the warping mapping $\mathbf{f}$.


\subsection{Prior variation shape evolution}
In the following we show how the shape warping, in the form of infinitesimal warping iteration, can be written as a curve evolution.

In the infinitesimal warping iteration, the small displacement due to small $\theta$ allows for the Taylor expansion of the warping for $\tau\rightarrow0$ at each iteration as follows:
\begin{equation}
\begin{split}
&\mathbf{x}(p,t'+\tau)\\
&\approx\mathbf{x}(p,t')+\left.\left[{\mathbf{D}\mathbf{f}(\mathbf{x}(p,t'),\theta)\over\mathbf{D}\theta}\right]\right|_{\theta=\mathbf{0}}\theta(\tau).
\end{split}
\end{equation}

Accordingly, we may try to approximate the infinitesimal warping iteration by curve evolution. Taking derivative w.r.t. $\tau$ at $\tau=0$ on the two sides of the above, we obtain the differential form of the model as follows:
\begin{equation}
\begin{split}
\left.{\partial\mathbf{x}(p,t'+\tau)\over\partial \tau}\right|_{\tau=0}=\left.\left[{\mathbf{D}\mathbf{f}(\mathbf{x}(p,t'),\theta)\over\mathbf{D}\theta}\right]\right|_{\theta=\mathbf{0}}\left.{d\theta\over d\tau}\right|_{\tau=0},
\end{split}
\end{equation}

By replacing $\mathbf{x}(p,t'+\tau)$ with the contour points $C(p,t'+\tau)$, and let $t=t'+\tau$ we obtain the following curve evolution equation of the infinitesimally warping iteration:
\begin{equation}\label{EQ:CE_DDSM}
\left.{\partial C(p,t)\over\partial t}\right|_{t=t'}={\Big[\mathbf{V}^\theta\Big]}\left.{d\theta\over dt}\right|_{t=t'},
\end{equation}
where ${d\over dt} = {d\over d\tau}$ for a fixed $t'$, $\theta$ has become a function of $\tau$, which is equivalently a function of $t$ for fixed $t'$, and
\begin{equation}\label{EQ:V_DDSM}
\Big[\mathbf{V}^\theta\Big]=\left.\left[{\mathbf{D}\mathbf{f}(C(p,t),\theta)\over\mathbf{D}\theta}\right]\right|_{\theta=\mathbf{0}},
\end{equation}
where we use $[\mathbf{V}^\theta]$ to emphasize that $\mathbf{V}^\theta$ is a $2\times n $ matrix-valued function, $n$ is the dimension of $\theta$. We may omit the subscript $t=t'$ since $t'$ can be arbitrary.

Note that $[\mathbf{V}^\theta]$ depends only on the contour location. Hence, its form can be decided \emph{a priori}, and Eq. (\ref{EQ:CE_DDSM}) is named the \emph{Prior Variation Shape Evolution} (PVSE) equation. The PVSE equation can be implemented via the level set method to avoid the artificial reparametrization. The level set equation corresponding to Eq. (\ref{EQ:CE_DDSM}) is the following:
\begin{equation}\label{EQ:LS_PVSE}
\begin{split}
{\partial \phi\over \partial t} = -\nabla\phi\cdot{\partial C(p,t)\over \partial t}=-\left\langle\nabla\phi,{\Big[\mathbf{V}^\theta\Big]}{d\theta\over dt}\right\rangle,&\\
\phi(x,y,0)=\phi_o(x,y),&\\
\end{split}
\end{equation}
where $\|\nabla\phi(x,y,t)\|=1$, $\phi$ is the signed distance function.

For example, we can consider the similarity transformation as the underlying shape warping:
\begin{equation}\label{EQ:Simi_T}
\mathbf{\tilde{x}}  =\left[
             \begin{array}{cc}
               e^\lambda\cos(\omega) & \sin(\omega) \\
               -\sin(\omega) &  e^\lambda\cos(\omega)\\
             \end{array}
           \right]\mathbf{x}+\left[
                               \begin{array}{c}
                                 a \\
                                 b \\
                               \end{array}
                             \right],
\end{equation}
where $\lambda,\omega$ and $[a,b]^T$ are the parameters of a similarity transformation. The prior variations of the similarity PVSE, which is referred to as the \emph{similarity variations}, are the following:
\begin{subequations}
\begin{equation}\label{EQ:Simi_V1}
\mathbf{V}^{\lambda} = [x, y]^T, \mathbf{V}^{\omega} = [y, -x]^T,
\end{equation}
\begin{equation}\label{EQ:Simi_V2}
\mathbf{V}^{a} = [1,0]^T, \mathbf{V}^{b} =[0,1]^T.
\end{equation}
\end{subequations}

For the affine transformation, we have
\begin{equation}\label{EQ:Aff_T}
\begin{split}
\mathbf{\tilde{x}}  =\left[\begin{array}{cc}
                            1+a_{11} & a_{12} \\
                            a_{21} & 1+a_{22} \\
                            \end{array}
                       \right]\mathbf{x}+\left[\begin{array}{c}
                                                                         a \\
                                                                         b
                                                                       \end{array}
                                                                     \right],
\end{split}
\end{equation}
where $a_{11},a_{12},a_{21},a_{22}$ and $[a, b]^T$ are the parameters of an affine transformation. The prior variations of the PVSE, which are referred to as the \emph{affine variations}, are the following:
\begin{subequations}
\begin{equation}\label{EQ:Aff_V1}
\mathbf{V}^{\mathbf{A}}(x)= \Big[[x,0]^T ,[y,0]^T, [0,x]^T, [0,y]^T\Big]^T,
\end{equation}
\begin{equation}\label{EQ:Aff_V2}
\mathbf{V}^{\mathbf{b}}(x)= \Big[[1,0]^T ,[0,1]^T\Big]^T.
\end{equation}
\end{subequations}
We can substitute these prior variations into (\ref{EQ:CE_DDSM}) to complete our PVSE equations. To understand the effect of PVSE more intuitively, we also visualize the curve evolution of similarity transformation and affine transformation in Figs. \ref{FIG:simi_A} and \ref{FIG:aff_A} respectively.

\begin{figure*}
  \setlength{\tabcolsep}{0pt}
  \subfloat{\fbox{\includegraphics[width=0.15\textwidth]{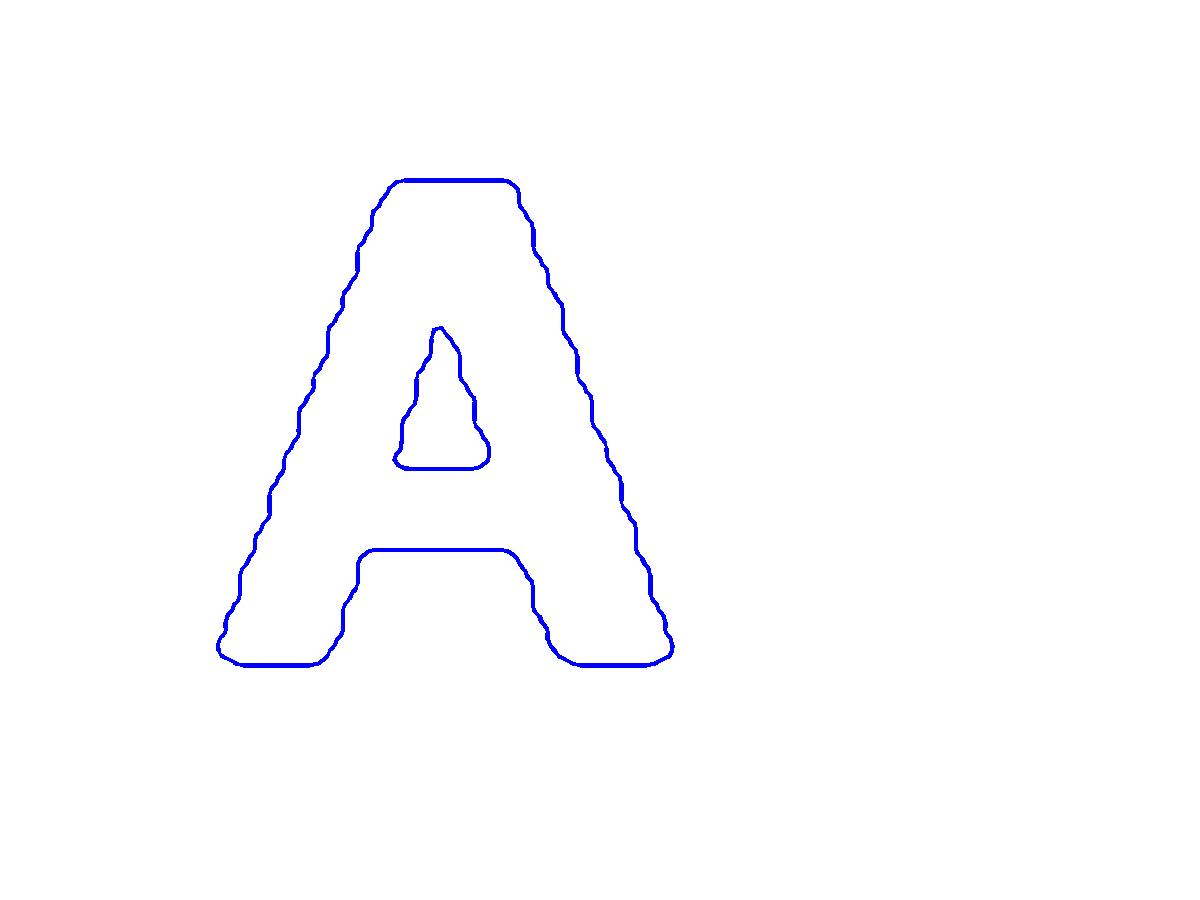}}}
  \subfloat{\fbox{\includegraphics[width=0.15\textwidth]{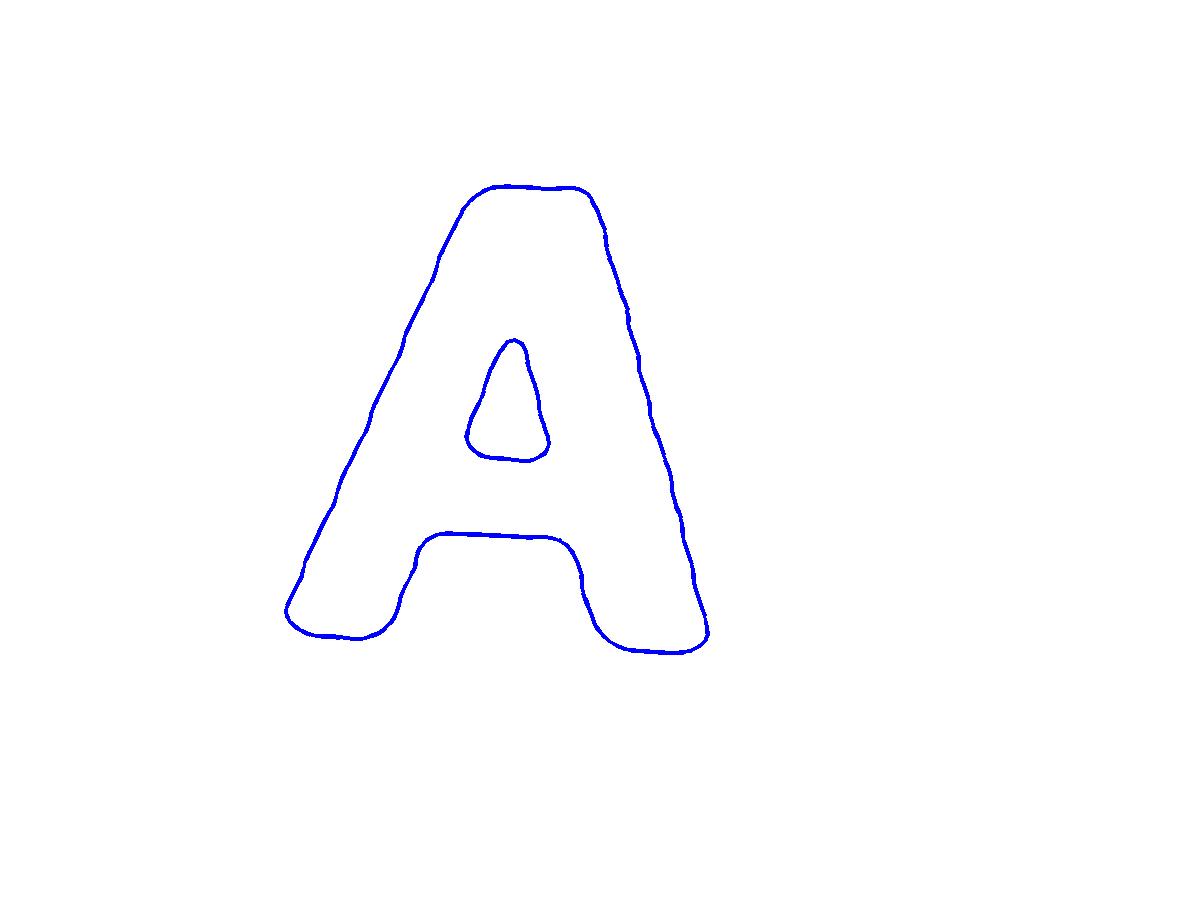}}}
  \subfloat{\fbox{\includegraphics[width=0.15\textwidth]{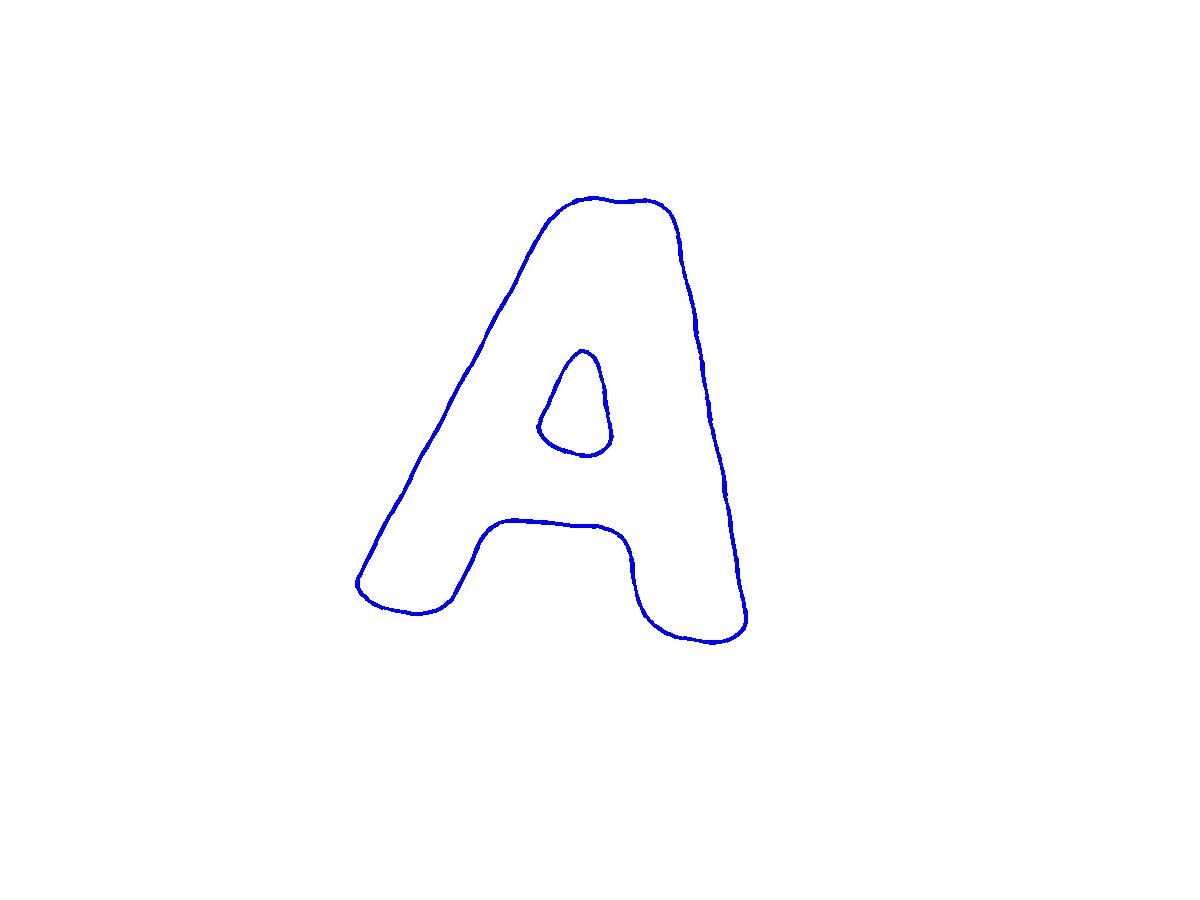}}}
  \subfloat{\fbox{\includegraphics[width=0.15\textwidth]{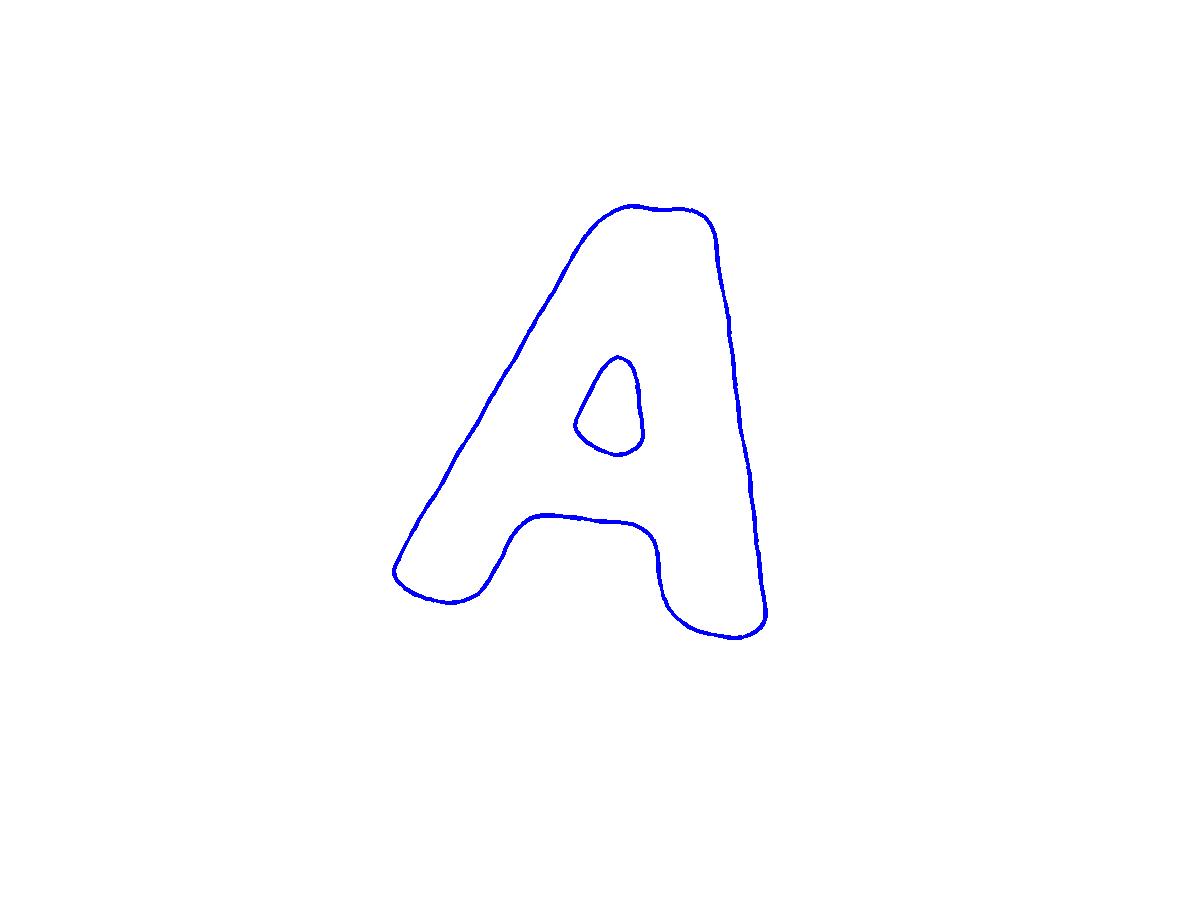}}}
  \subfloat{\fbox{\includegraphics[width=0.15\textwidth]{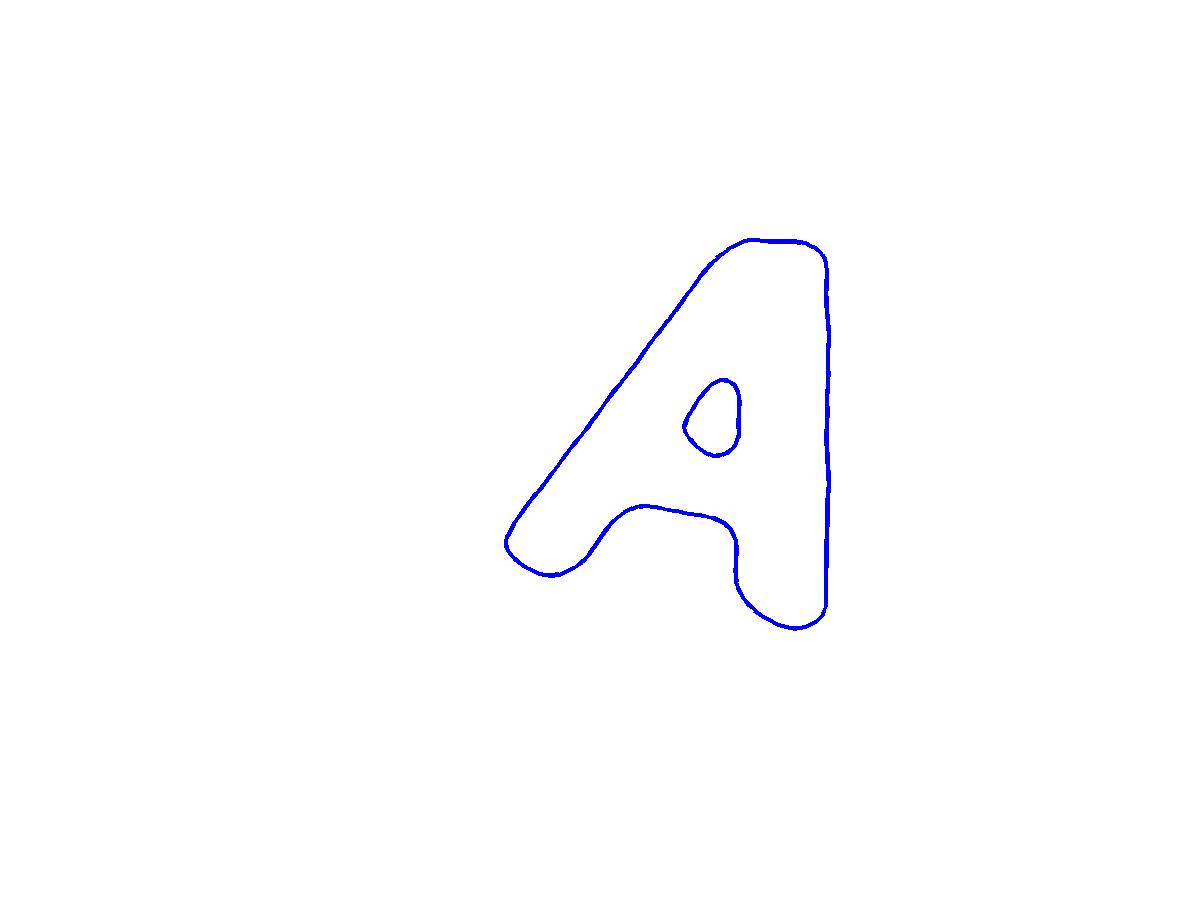}}}
  \subfloat{\fbox{\includegraphics[width=0.15\textwidth]{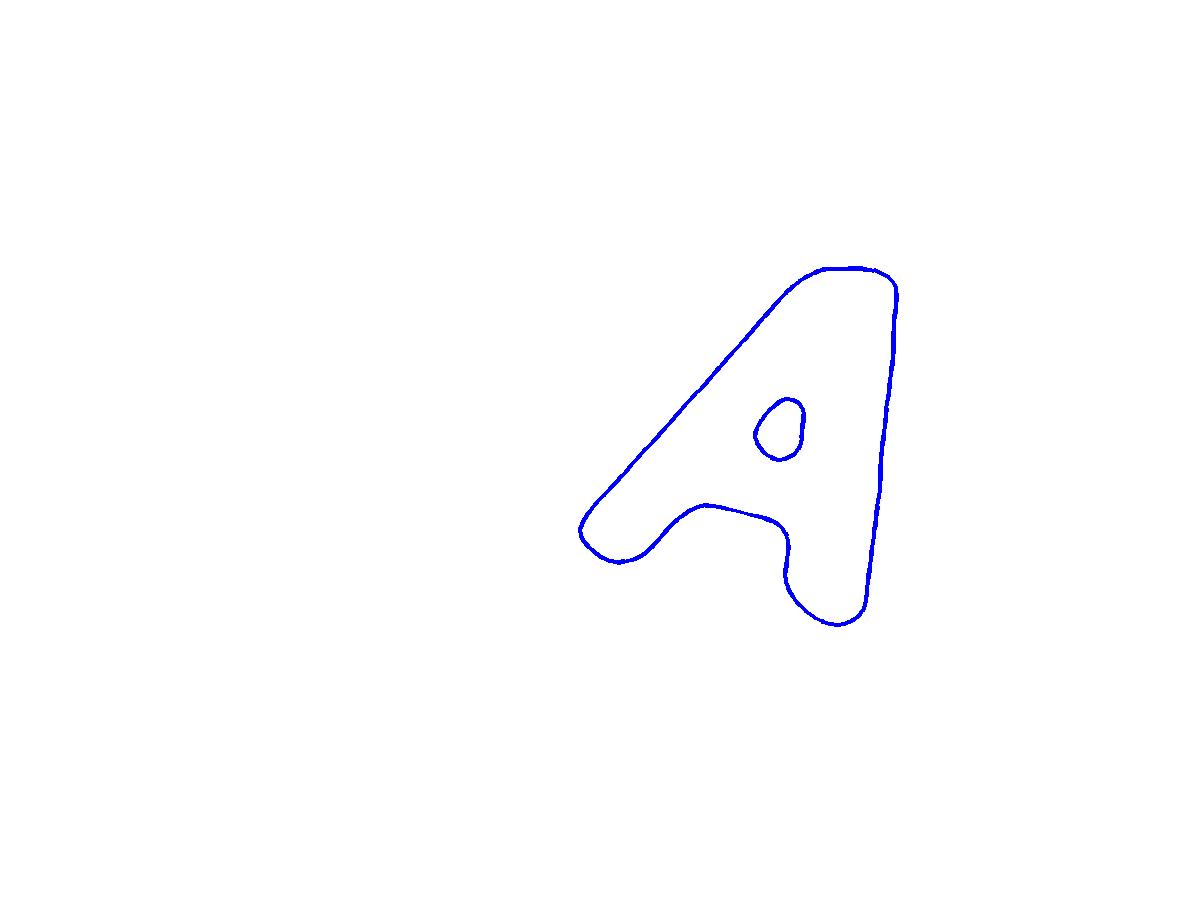}}}
  \caption{The curve evolution for concurrently rotation, scaling and translation of ``$\mathbb{A}$''.}\label{FIG:simi_A}
\end{figure*}

\begin{figure*}
  \setlength{\tabcolsep}{0pt}
  \subfloat{\fbox{\includegraphics[width=0.15\textwidth]{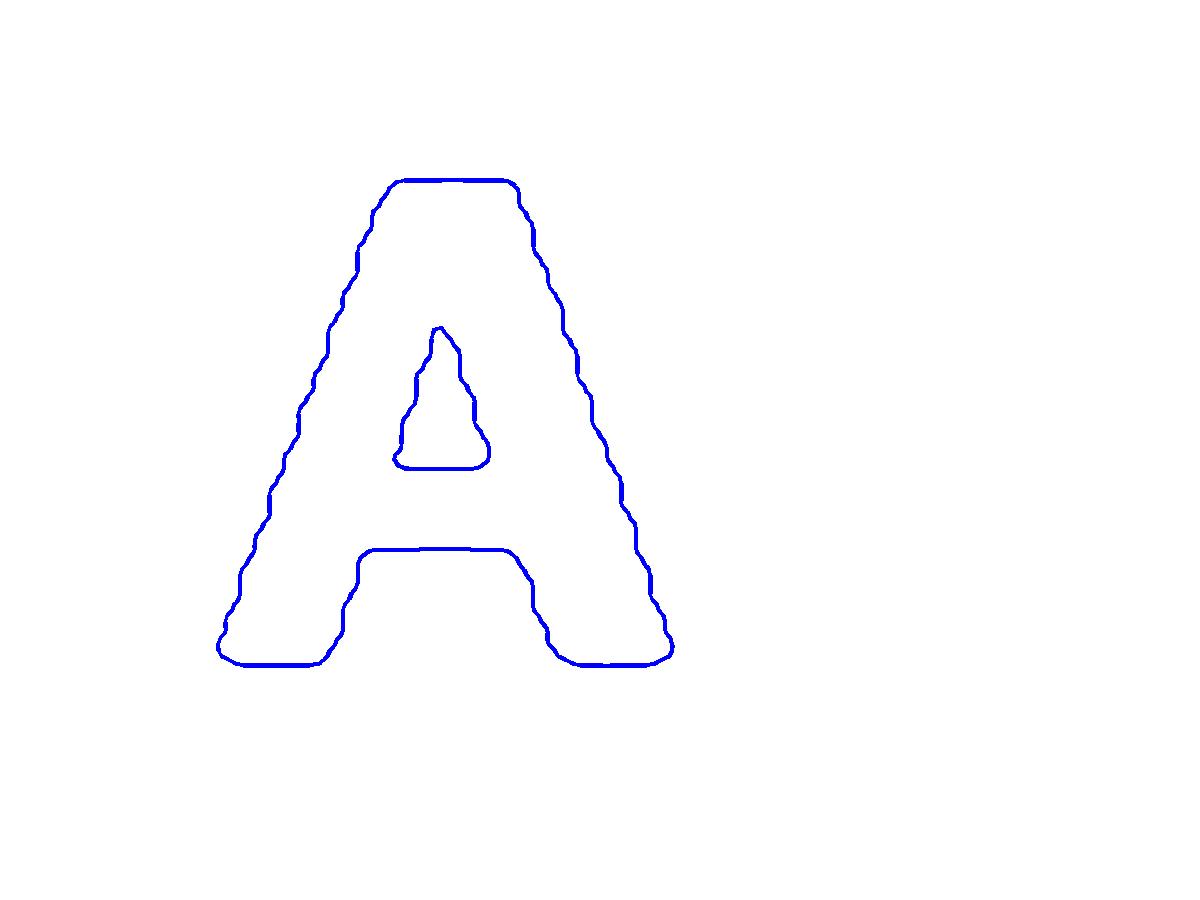}}}
  \subfloat{\fbox{\includegraphics[width=0.15\textwidth]{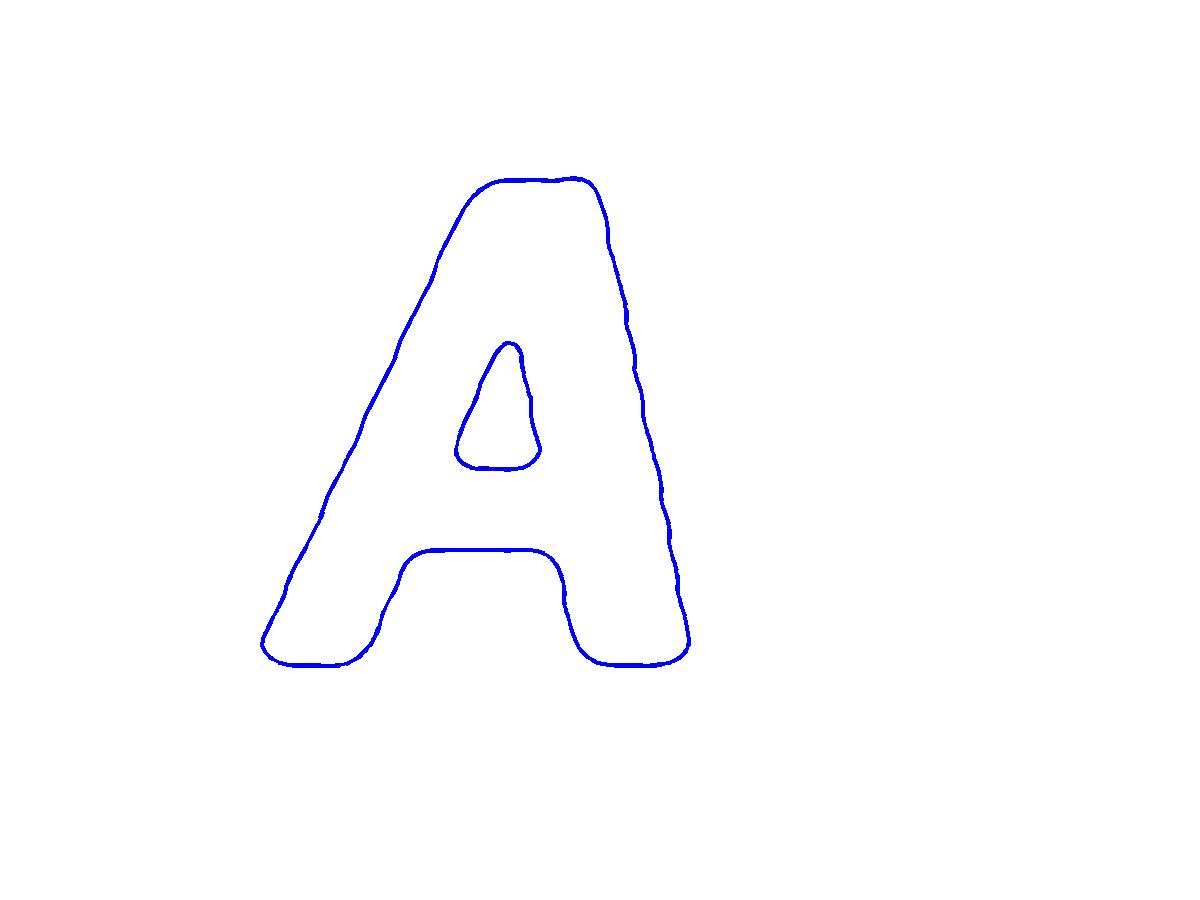}}}
  \subfloat{\fbox{\includegraphics[width=0.15\textwidth]{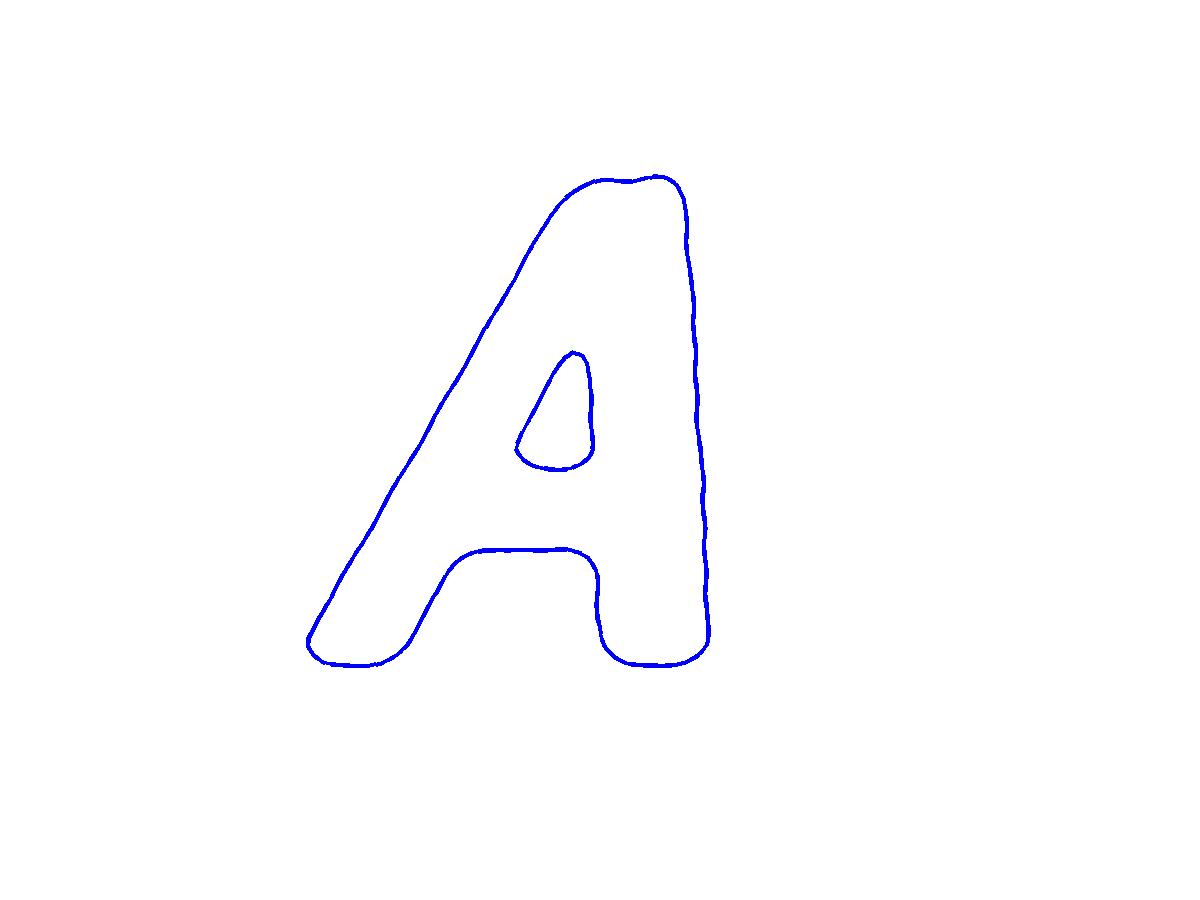}}}
  \subfloat{\fbox{\includegraphics[width=0.15\textwidth]{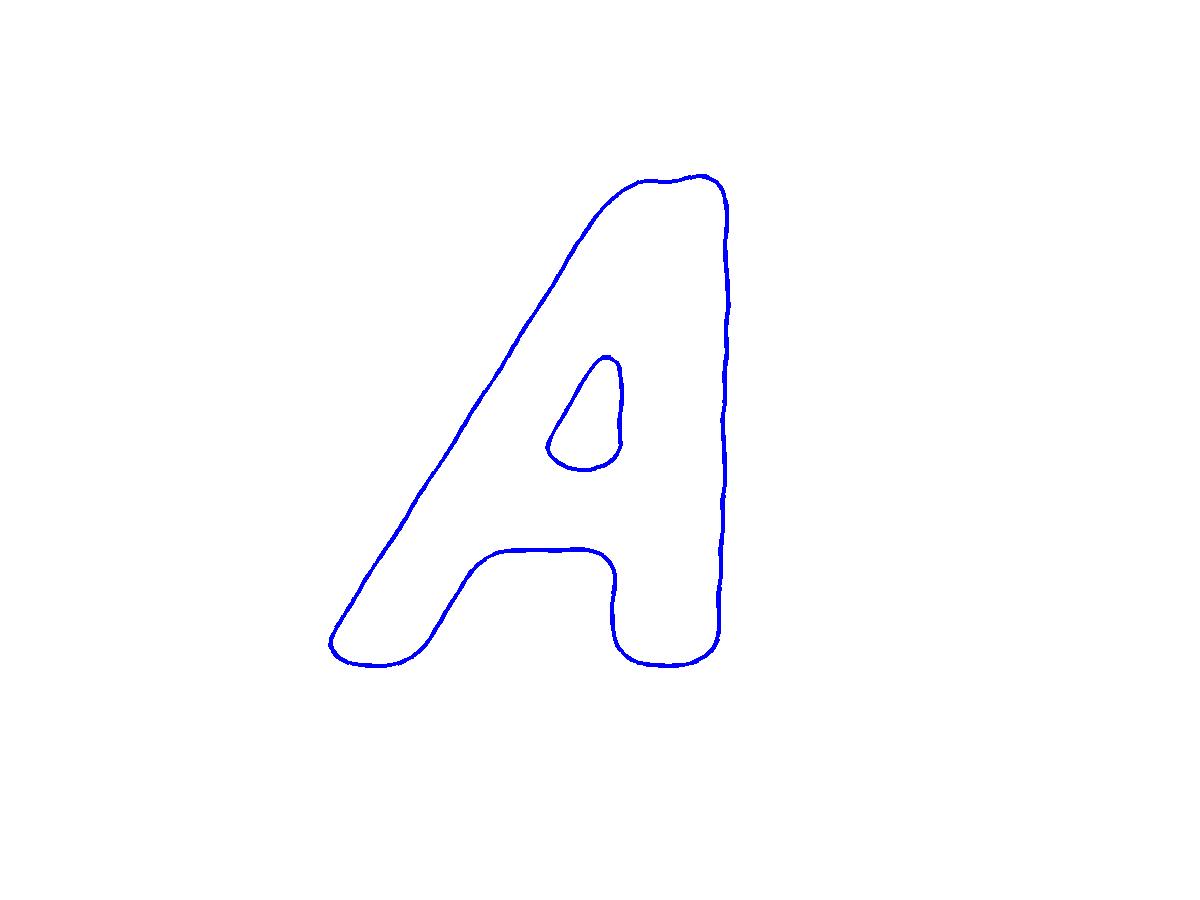}}}
  \subfloat{\fbox{\includegraphics[width=0.15\textwidth]{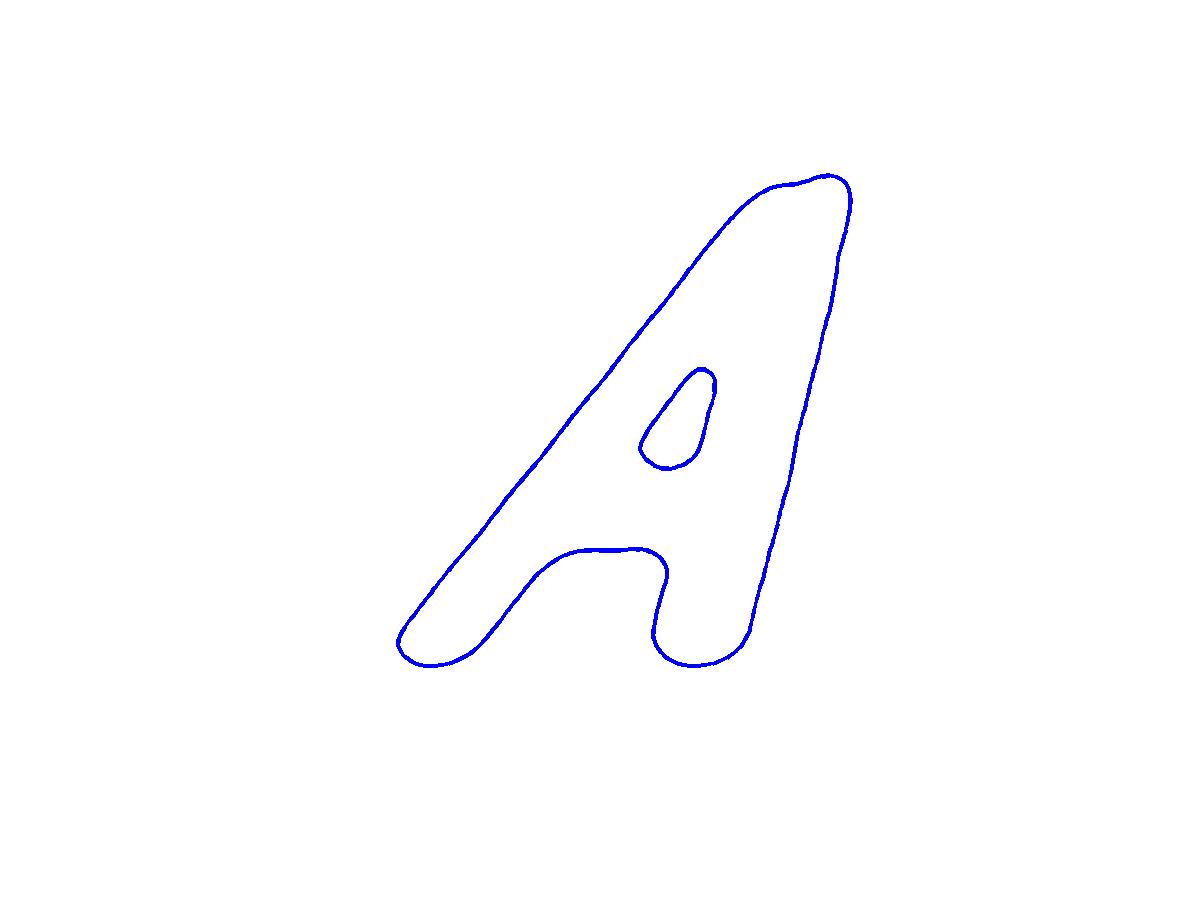}}}
  \subfloat{\fbox{\includegraphics[width=0.15\textwidth]{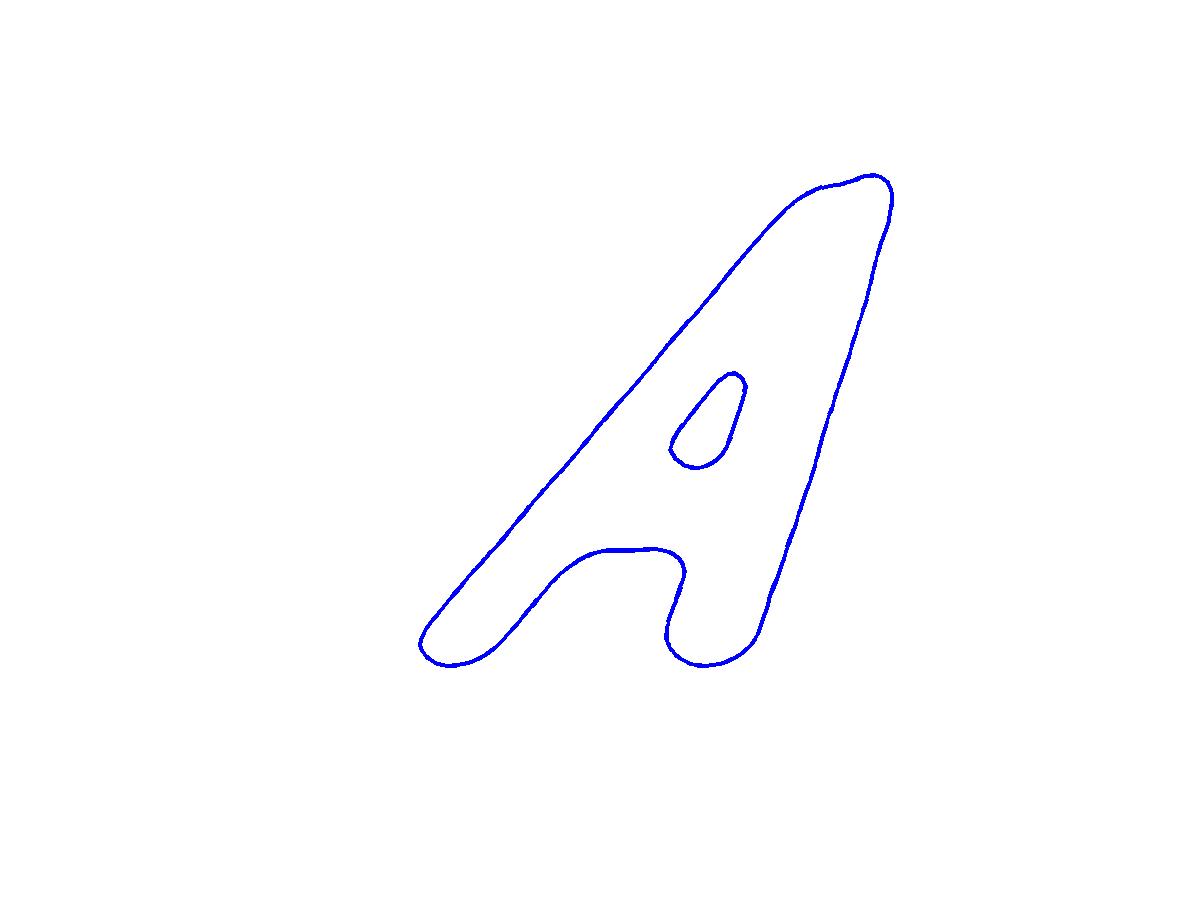}}}
  \caption{The curve evolution for affine transformation of ``$\mathbb{A}$''.}\label{FIG:aff_A}
\end{figure*}



\section{Minimizing active contour energy by PVSE}\label{SEC:CoPV}
In this section, we consider how the PVSE can be used for minimizing active contour energies, such that we can use PVSE to locate object boundaries. To this effect, we propose a new method called \emph{calculus of prior variations}. The calculus of prior variations can be used as an alternative to the conventional calculus of variations for deriving the curve evolution equations if the form of the variation is known \emph{a priori}.

The conventional calculus of variations, which can be represented by the $G\hat{a}teaux$ derivative, is as follows:
\begin{equation}\label{EQ:CoV_1}
\begin{split}
{d \mathcal{J}\over d t}
                &=\left\langle\nabla \mathcal{J},{\partial C\over\partial t}\right\rangle^*_p,
\end{split}
\end{equation}
where $\nabla \mathcal{J}$ is the functional gradient, $\langle\cdot,\cdot\rangle_p^*$ is the functional inner product of two 2-D vector functions.

Substituting the PVSE equation (\ref{EQ:CE_DDSM}) into the $G\hat{a}teaux$ derivative in Equation (\ref{EQ:CoV_1}), we obtain the following:
\begin{equation}\label{EQ:CoV_2}
\begin{split}
{d \mathcal{J}\over d t}
                &=\left\langle\nabla \mathcal{J},{\Big[\mathbf{V}^\theta\Big]}{d\theta\over dt}\right\rangle^*_p\\
                &=\sum_{i=1}^n\left\langle\nabla \mathcal{J},\mathbf{V}^{\theta_i}\right\rangle^*_p{d\theta_i\over dt}={\mathbf{D} J\over \mathbf{D}\theta}{d\theta \over dt},
\end{split}
\end{equation}
where $\theta=[\theta_1,\theta_2,...,\theta_n]^T$, $n$ is the number of parameters and we omit $t=t'$. We can obtain the gradient descent equation for $\theta$ as follows:
\begin{equation}\label{EQ:CoV_theta1}
\begin{split}
{d\theta\over dt} &= -{\mathbf{D} J\over \mathbf{D}\theta}=-\int_C\Big[\mathbf{V}^\theta\Big]\nabla \mathcal{J}dp\\
&=-\left[\left\langle\mathbf{V}^{\theta_i},\nabla \mathcal{J}\right\rangle_p^*\Bigg|i=1,2,...,n\right]^T.
\end{split}
\end{equation}
%

By back substituting (\ref{EQ:CoV_theta1}) into the PVSE equation (\ref{EQ:CE_DDSM}), we obtain the following:
%
%
\begin{equation}\label{EQ:PVSE_general}
\begin{split}
{\partial C(p,t)\over\partial t} =-\sum_{i=1}^n\left\langle\nabla \mathcal{J},\mathbf{V}^{\theta_i}\right\rangle^*_p\mathbf{V}^{\theta_i}&\\
C(p,0)=C_o(p).&
\end{split}
\end{equation}

The equation of energy minimizing PVSE (\ref{EQ:PVSE_general}) finalizes our approach for shape registration. We may substitute the functional gradients of various active contour energies into (\ref{EQ:PVSE_general}) as well as the corresponding level set equation (\ref{EQ:LS_PVSE}). For example, the gradient descent equation of GAC \cite{caselles97GAC,Yezzi97GAC_J} is the following:
\begin{equation}
{\partial C\over\partial t} = -\nabla \mathcal{J}_{_{GAC}}(C)= g\kappa\mathbf{N}-\langle\nabla g,\mathbf{N}\rangle\mathbf{N},
\end{equation}
where $\kappa$ is the contour curvature, $g$ is an edge indicator function and $\mathbf{N}$ is the contour normal.

The general form of the corresponding PVSE equation is the following:
\begin{equation}\label{EQ:GAC_PVSE}
{\partial C\over\partial t}= -\sum_{\theta_i\in\theta}\langle\nabla \mathcal{J}_{_{GAC}},\mathbf{V}^{\theta_i}\rangle_p^* \mathbf{V}^{\theta_i}.
\end{equation}
Thus, we may call Eq. (\ref{EQ:GAC_PVSE}) the GAC PVSE equation.

Likewise, we may rewrite the gradient descent equation of Chan-Vese active contour \cite{ChanVese01ActiveCon} or the region competition \cite{Zhusongchun96RegComp} as follows:
\begin{equation}
\begin{split}
{\partial C\over\partial t} &= -\nabla \mathcal{J}_{_{CV}}(C)\\
    &= \left(\mu \kappa-\nu+\lambda_2|u_0-c_2|^2-\lambda_1|u_0-c_1|^2\right)\mathbf{N},\\
\end{split}
\end{equation}
where $\mu,\nu,\lambda_1,\lambda_2$ are positive penalty coefficients, $u_0$ is the (feature) image, $c_1, c_2$ are defined below, all other notations are defined previously:
\begin{equation}
c_1 = \frac{\int\limits_\Omega u_0 H(\phi)dxdy}{\int\limits_\Omega H(\phi)dxdy},~~
c_2 = \frac{\int\limits_\Omega u_0 (1-H(\phi))dxdy}{\int\limits_\Omega(1-H(\phi))dxdy}.
\end{equation}

The general form of the corresponding PVSE equation is the following:
\begin{equation}\label{EQ:CV_PVSE}
{\partial C\over\partial t}= -\sum_{\theta_i\in\theta}\langle\nabla \mathcal{J}_{_{CV}},\mathbf{V}^{\theta_i}\rangle_p^* \mathbf{V}^{\theta_i}.
\end{equation}
We may call (\ref{EQ:CV_PVSE}) the Chan-Vese PVSE equation.


We can observe from the PVSE equations above that the velocity of the energy-minimizing curve evolution is generated by \emph{projecting} the negative functional gradient to the subspace spanned by the prior variations. The derivations for the energy-minimizing PVSE equations for all the existing active contours are similar. The rest is to substitute the rigid and non-rigid prior variations into the PVSE equations for implementation. The energy-minimizing PVSE equation is to be implemented by numerical solvers, such as the level set method. The convergence can be detected if the short-time, e.g. 20 iterations, average of the contour motion approaches zero. The implementation is summarized in Algorithm \ref{Alg:PVSE}. We adopt the approximate Dirac delta function $\delta_\epsilon$ used in \cite{ChanVese01ActiveCon}. The re-initialization is achieved by using high-order ENO scheme
\cite{OsherBookDynamic}.

\begin{algorithm}
\DontPrintSemicolon
\LinesNumbered
\SetKwInOut{Inputalg}{Input}\SetKwInOut{Outputalg}{Output}
\Inputalg{Initial ROI $\Omega_o$, Prior Variations $\big[\mathbf{V}^\theta_{rg}\big]$ or $\big[\mathbf{V}^\theta_{nrg}\big]$, Input image $I(\cdot)$}
\Outputalg{$\phi^*$}
\Begin{
    $\phi^{0}(x,y)\Leftarrow \left\{\begin{array}{rl}
                                -1,& \hbox{if } [x,y]^T\in\Omega_o\\
                                1,&  \hbox{if }[x,y]^T\in \overline{\Omega}_o
                              \end{array}\right.$\;\vspace{5pt}
    $\big[\mathbf{V}^\theta\big]\Leftarrow\big[\mathbf{V}^\theta_{rg}\big]$ or $\big[\mathbf{V}^\theta_{nrg}\big]$\;
    $k\Leftarrow1$\;
    \Repeat{Convergence}{
        $\phi^k\Leftarrow$ \texttt{Reinitialize} $\phi^{k-1}$\;
        $\nabla\mathcal{J}^k\Leftarrow\nabla \mathcal{J}_{_{AC}}(\phi^k,I)$\;\vspace{5pt}
        $\mathcal{P}^k_{_{\mathbf{V}^\theta}}\Leftarrow\sum\limits_{{\theta_i}\in\theta}\langle\mathbf{V}^{\theta_i}, \nabla \mathcal{J}^k\rangle_p^*\mathbf{V}^{\theta_i}$\;
        $\phi^{k+1}\Leftarrow \phi^{k}-\Delta t \left\langle\mathcal{P}^k_{_{\mathbf{V}^\theta}}, \nabla\phi^k\right\rangle\delta_\epsilon(\phi^{k})$\;
        $k\Leftarrow k+1$\;
    }
    $\phi^*\Leftarrow\phi^k$\;
}
\caption{The energy-minimizing PVSE}\label{Alg:PVSE}
\end{algorithm} 
\section{A theory of shape preservability of PVSE for shape recovery}\label{SEC:Theory_ShaPre}
In this section, we address the feasibility of modeling shape deformation by PVSE. The prior variations can be chosen as the prior variations of either the similarity or affine transformations for modeling similarity or affine transformations. The corresponding prior variations are as shown in Eqs. (\ref{EQ:Simi_V1})(\ref{EQ:Simi_V2}) and (\ref{EQ:Aff_V1})(\ref{EQ:Aff_V2}). Our major difficulty of the deformable shape modeling lies in choosing the prior variations for modeling non-rigid deformations. The question is: what prior variations \emph{can} be used to model non-rigid deformations of an object? \footnote{Note that the theory established in \cite{Wang09PVCE} was not correct.}

\subsection{Shape preservability as the invariance of shape characteristics}\label{SUBSEC:CurvatureBounds}
We propose to choose the prior variations that can achieve shape recovery, meaning that the shape characteristics are always preserved during the deformation. An important shape characteristic is the zero-crossing of curvature in the curvature scale space \cite{Abbasi99CSS,Mokhtarian03CSSbook}. Therefore, we investigate how such zero-crossings, or feature points, may vary during the PVSE. A fundamental result is the following.
\begin{theorem}\label{THM:Curvature_D2}
Given the shape evolution governed by (\ref{EQ:CE_DDSM}), the dynamics of the curvature $\kappa$ at the positions where $\kappa=0$, i.e. the zero-crossings of the curvature, satisfies the following:
\begin{equation}\label{EQ:DKappa_final}
\left\|\left.{\partial \kappa\over\partial t}\right|_{\kappa=0}\right\| \leq \left\|{d\theta\over dt}\right\|\sum_i^n\left\|{\mathbf{D}^2\mathbf{V}^{\theta_i}\over\mathbf{D}\mathbf{x}^2}\right\|.
\end{equation}
\end{theorem}
In words, the dynamics of the curvature at the zero-crossing points is bounded by the norm of the second order derivatives of the prior variations of the PVSEs. Besides, we show that the distance between the contour parameters at the zero-crossings is also bounded by the norm of the second order derivative of the prior variations.
\begin{corollary}\label{COR:F_bound}
Given two zero-crossing points $C(p_1,t)$ and $C(p_2,t)$ at time $t$, the dynamic of the distance between the parameters, $p_1$ and $p_2$, of the feature points is bounded as follows:
\begin{equation}\label{EQ:UPB_d_dist/dt}
\left|{d\over dt}|p_1-p_2|\right|\leq{2\over\|\kappa_sC_p\|}\left\|{d\theta\over dt}\right\|\sum_i^n\left\|{\mathbf{D}^2\mathbf{V}^{\theta_i}\over\mathbf{D}\mathbf{x}^2}\right\|.
\end{equation}
\end{corollary}
The proofs of Theorem \ref{THM:Curvature_D2} and Corollary \ref{COR:F_bound} are deferred to the Appendix. The theoretical findings above suggest using PVSEs with prior variations of small second order derivatives for modeling non-rigid shape variations.

Similarity and affine transformations are natural shape preserving deformations. Therefore, it is necessary to examine whether this theory is valid for the rigid type transformations, such as the similarity and affine transformations. By looking at Eqs. (\ref{EQ:Simi_V1}) (\ref{EQ:Simi_V2}) (\ref{EQ:Aff_V1}) and (\ref{EQ:Aff_V2}), we can immediately see that the second order derivatives of these prior variations vanish, meaning that the zero-crossings of the curvature do not change. Accordingly, the established theory is validated in the case of similarity and affine transformations.

\subsection{The prior vibrations}
In choosing prior variations for modeling the non-rigid deformations, we are particularly interested in the following family of functions:
\begin{equation}\label{EQ:Harmo}
\left[\begin{array}{c}
e^1_{mn}(\mathbf{x})\\
e^2_{mn}(\mathbf{x})
\end{array}\right] =
\left[\begin{array}{c}
 {\sin(\pi nx)\cos(\pi my)\over{\pi^2}(n^2+m^2)}\\
 {\cos(\pi mx)\sin(\pi ny)\over{\pi^2}(n^2+m^2)}
\end{array}\right],
\end{equation}
where $x$ and $y$ are the coordinates, and the coordinates are normalized such that $\mathbf{x}=[x,y]^T\in[0,1]\times[0,1]$ as in \cite{Amit91StructuralImage}.

This class of functions are the eigenfunctions of the Laplace operator. They have been used for representing non-rigid shape deformations in the deformable template matching \cite{Amit91StructuralImage,Jain96Objectmatching}. Variants of this function class have also been used, such as in \cite{Staib92DeformableModel,Krinidis05VibrationDeformableModel}. This means that this function class is useful, though the theory behind its usefulness has not been revealed.

We may consider the shape warping defined by using this function class. The corresponding PVSE has the prior variations in the form of this function class as follows:
\begin{align}
&\Big[\mathbf{V}^\theta_{nrg}\Big] = \Big[\mathbf{V}_{mn}\Big|0<m\leq M,0<n\leq N\Big],\label{EQ:V_Harmo1}\\
&\mathbf{V}_{mn} = \left[\begin{array}{c}
{e}_{mn}^1(C(p,t)) \\
{e}_{mn}^2(C(p,t)) \\
\end{array}\right]^T,\label{EQ:V_Harmo2}
\end{align}
which is named as the \emph{prior vibrations} due to the periodic oscillating structure of the basis functions.

In the following, we show that the PVSE with prior vibrations is shape-preserving with a small order $M$ and $N$.
\begin{theorem}\label{THM:BD_PriVib}
For the function class defined in (\ref{EQ:V_Harmo1}) and (\ref{EQ:V_Harmo2}), we have the following bound of its second order derivatives:
\begin{equation}
\sum_{m,n=1}^{MN}\left\|{\mathbf{D}^2\mathbf{V}_{mn}\over\mathbf{D}\mathbf{x}^2}\right\|_\infty\leq 4MN,
\end{equation}
where $M$ and $N$ are defined in (\ref{EQ:V_Harmo1}). $\|\|_\infty$ is the $\infty$-norm of matrix.
\end{theorem}

The proof of Theorem \ref{THM:BD_PriVib} is deferred to the Appendix. According to our theory of shape preservability presented in Section \ref{SUBSEC:CurvatureBounds}, we can choose small values of $M$ and $N$, namely using the truncated series, to preserve the shape characteristic. Likewise, the smoothness of the prior variations, in terms of the magnitude of the first order differential of $\mathbf{V}_{mn}$, can also be guaranteed by using small values of $M$ and $N$, which also means that prior vibrations give smooth shape deformations.

Accordingly, we can substitute the prior vibrations in (\ref{EQ:V_Harmo2}) into the PVSE equation (\ref{EQ:CE_DDSM}) to lead to a shape-preserving non-rigid PVSE. We shall call the prior vibrations in (\ref{EQ:Harmo}) the $1^{st}$ order prior vibration if $M+N=1$ , and likewise the rest are the $2^{nd}$, $3^{rd}$,...,$(M+N)^{th}$ order prior vibrations.

\section{Experimental results}\label{SEC:Exp}

\subsection{Evaluation of the PVSE equation for modeling shape transformation}
In this subsection, we present the experimental results for evaluating the performance of PVSE in modeling similarity and affine shape transformation.

\subsubsection{Experiment configurations}
To evaluate the performance of the PVSE for modeling shape deformation, we propose to examine if the PVSE achieves the desired deformation. We therefore propose to measure the distance between the evolving shapes due to PVSE and the shapes deformed by using the given deformation. We consider the similarity and affine transformations. The shape distance is defined as follows:
\begin{equation}\label{EQ:S_dist}
\rho(\mathbf{S},\mathbf{S}_o) = \min\limits_{\mathbf{A},\mathbf{b}}~d(\mathbf{S},\mathbf{A}\circ\mathbf{S}_o+\mathbf{b}).
\end{equation}
We may use a natural distance between point sets, i.e. the Hausdorff distance, to compute the distance $d$. The Hausdorff distance may be sensitive to noise, but the shapes are clean in this experiment. $\mathbf{A},\mathbf{b}$ are the parameters of similarity or affine transformation. They can be solved by using shape matching, such as \cite{Array05ShapeContext}. We emphasize that the shape matching used here is only for producing a quantitative measure of the performance of our method.

In principle, if the PVSEs with prior variations (\ref{EQ:Simi_V1})(\ref{EQ:Simi_V2}) and (\ref{EQ:Aff_V1})(\ref{EQ:Aff_V2}) can achieve the corresponding shape transformations, the shape distance according to the corresponding similarity or affine geometry will be small. We also require a baseline result for comparison. In our experiments, we use the bolded \textbf{A} in Fig. \ref{FIG:A_shapes} (left) as the prototype, and we use the slanted \textit{A} in Fig. \ref{FIG:A_shapes} (middle) as the baseline shape. The size of both images is $64\times64$. We also present the optimally aligned shape against the target shape in Fig. \ref{FIG:A_shapes}(right). The affine invariant Hausdorff shape distance between the optimally aligned shape and the reference shape is about $4.5$ according to Eq. (\ref{EQ:S_dist}).
\begin{figure}[htb]
\centering
  \subfloat{\includegraphics[width=0.166\textwidth]{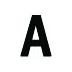}}
  \subfloat{\includegraphics[width=0.166\textwidth]{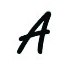}}
  \subfloat{\includegraphics[width=0.166\textwidth]{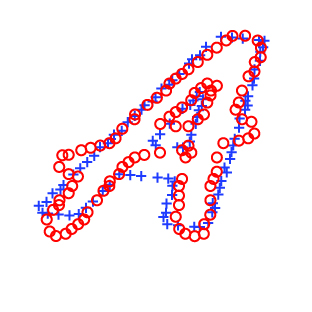}}
  \caption{The prototype (left) and baseline (middle) shapes used for evaluation, and the right shows the affine shape context matching of the baseline (blue crosses) by the prototype (red circles). }\label{FIG:A_shapes}
\end{figure}

Our experiment proceeds as follows.
\renewcommand{\theenumi}{\alph{enumi}}
{\emph{\begin{enumerate}
  \item Generate random velocities ${d\theta\over dt}$ for the prior variations;
  \item Run the corresponding PVSE with a prototype initial shape and record the evolving shapes;
  \item Perform the shape context matching for the recorded shapes and estimate the similarity/affine transformation parameters;
  \item Compute the shape distances defined by (\ref{EQ:S_dist}).
\end{enumerate}}}

In detail, we generate $10$ random vectors of ${d\theta\over dt}$ for the two types of PVSEs. We run each of the PVSEs for $50$ iterations to generate a collection of $10\times2\times50=1000$ shapes for evaluation.

\subsubsection{Experimental results and analysis}

The quantitative results of this experiment are shown in Figs. \ref{FIG:H_dists_simi_aff}. The results are visualized by boxplots for all iterations. The top and bottom of each box are the $25^{th}$ and $75^{th}$ percentiles of the distribution respectively. The line in the middle of each box is the median, and the red crosses are outliers. By inspecting Fig. \ref{FIG:H_dists_simi_aff}, we observe that the shape distances are small comparing with the baseline value and the curve evolution does not increase the shape distance, which is good. To conclude, the PVSE is effective for simulating the similarity and affine transformations. Some examples of the simulated PVSEs are visualized in Fig. \ref{FIG:PVSE_vs_ShapeContext} for better understanding the experiment. Table \ref{TB:H_dist} summarizes the quantitative results corresponding to Fig. \ref{FIG:PVSE_vs_ShapeContext}.
\begin{figure}
\centering
  \subfloat{\includegraphics[width=0.45\columnwidth]{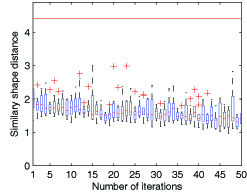}}
  \subfloat{\includegraphics[width=0.45\columnwidth]{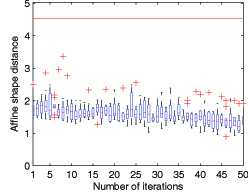}}
  \caption{The shape distance for the 50 iterations of similarity (left) and affine (right) PVSEs. The red line is the shape distance between the two ``A''s in Fig. \ref{FIG:A_shapes}.}\label{FIG:H_dists_simi_aff}
\end{figure}


\begin{figure}[htb]
\centering
 {\includegraphics[width=0.3\columnwidth]{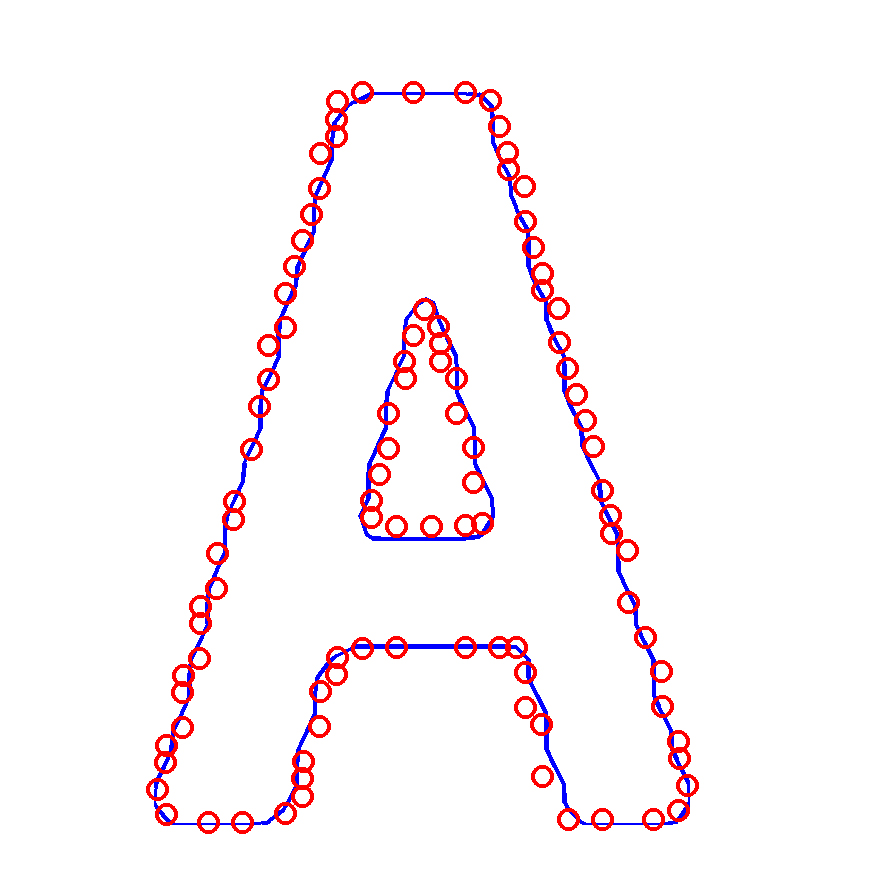}}
  {\includegraphics[width=0.3\columnwidth]{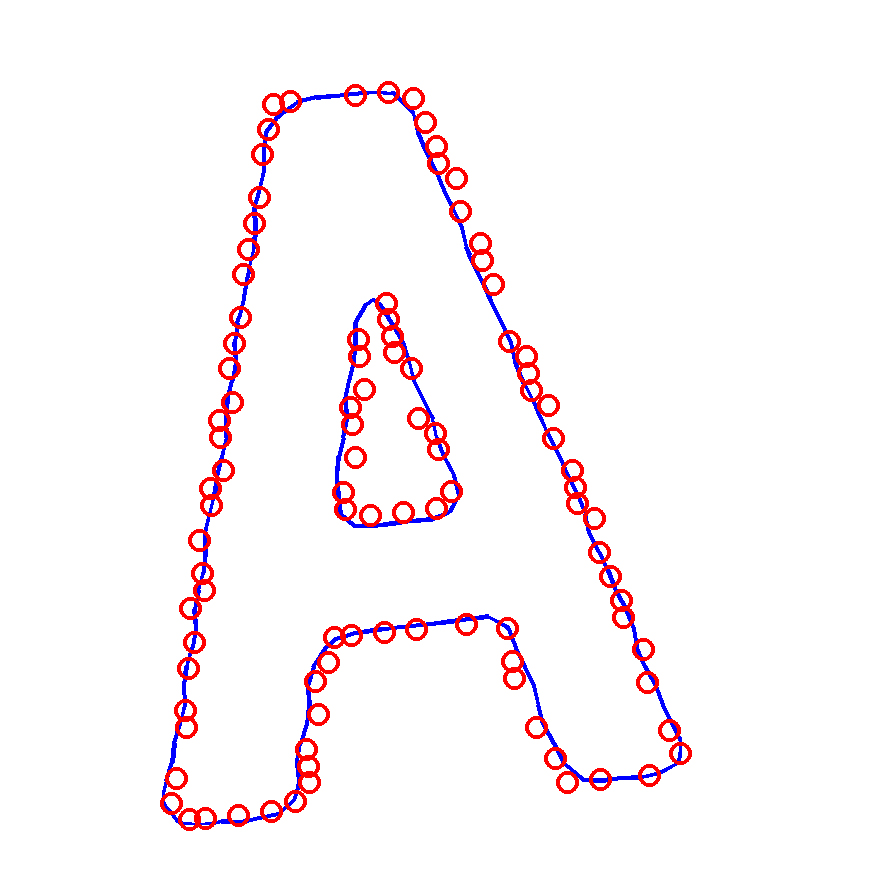}}
  {\includegraphics[width=0.3\columnwidth]{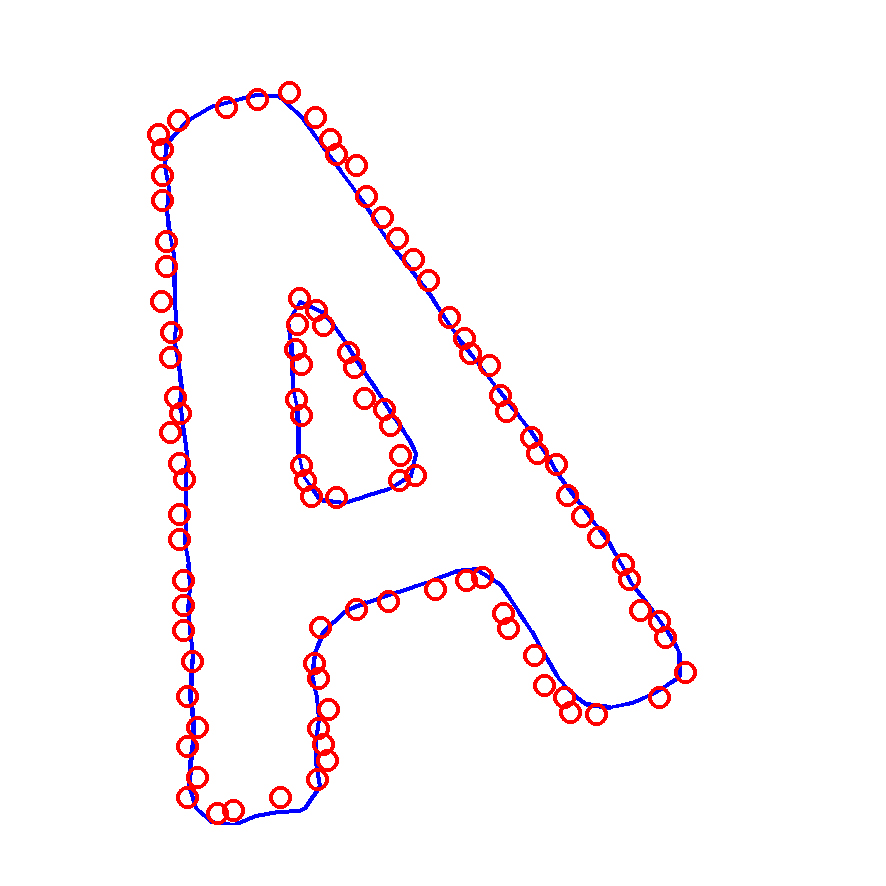}}\\
  {\includegraphics[width=0.3\columnwidth]{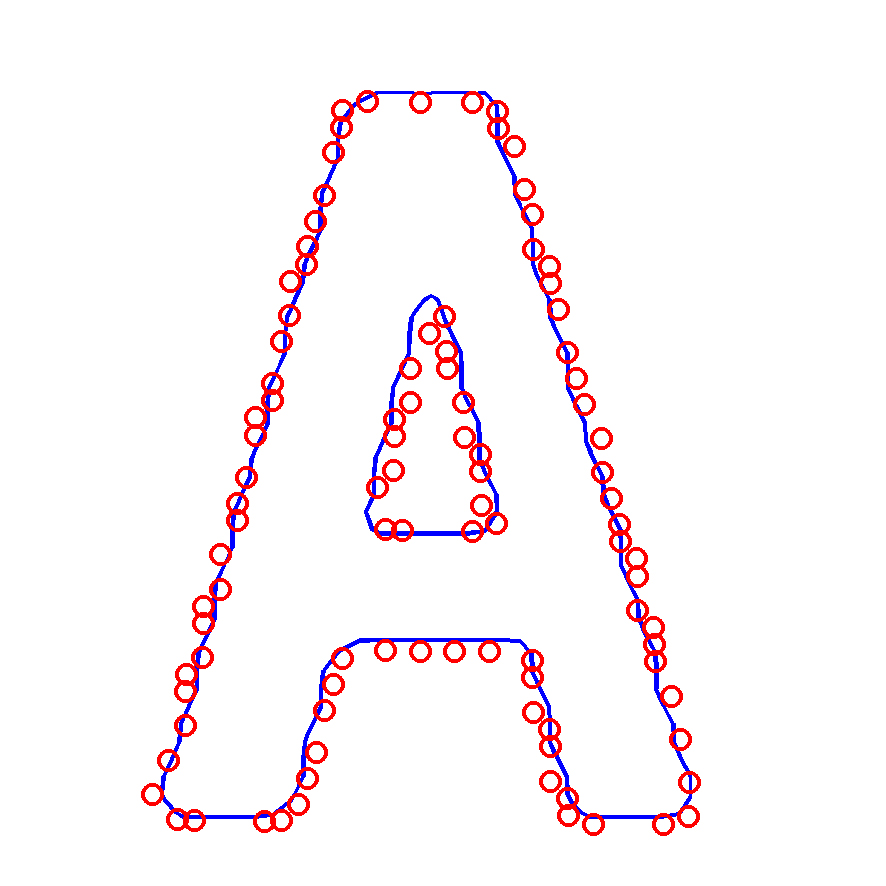}}
  {\includegraphics[width=0.3\columnwidth]{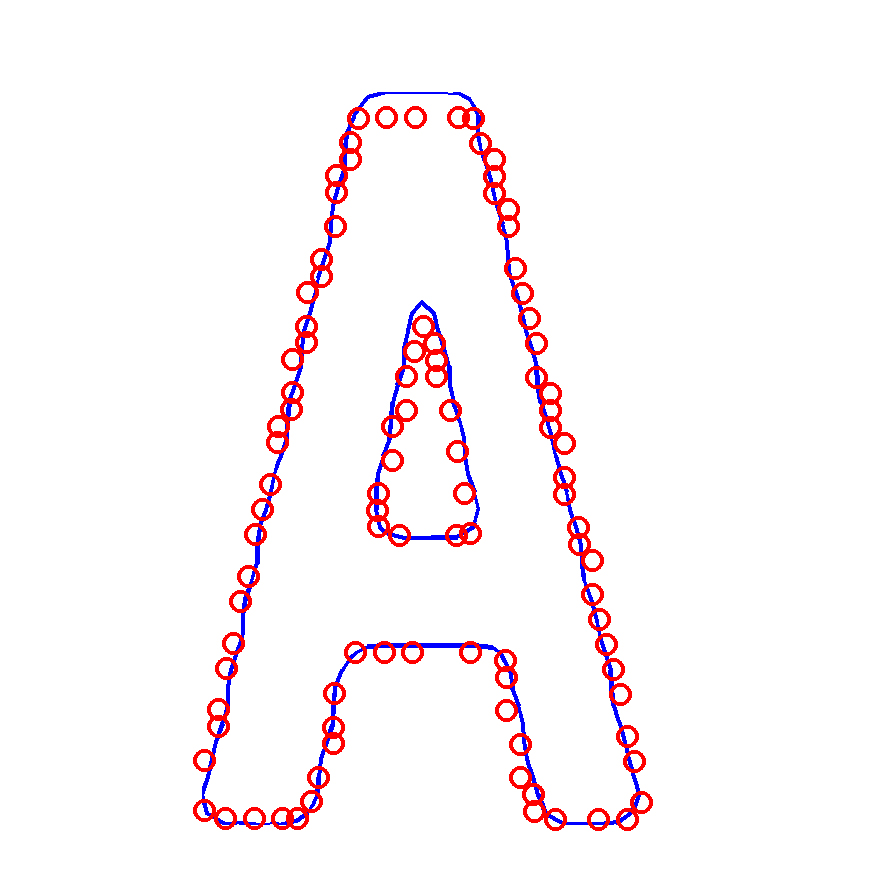}}
  {\includegraphics[width=0.3\columnwidth]{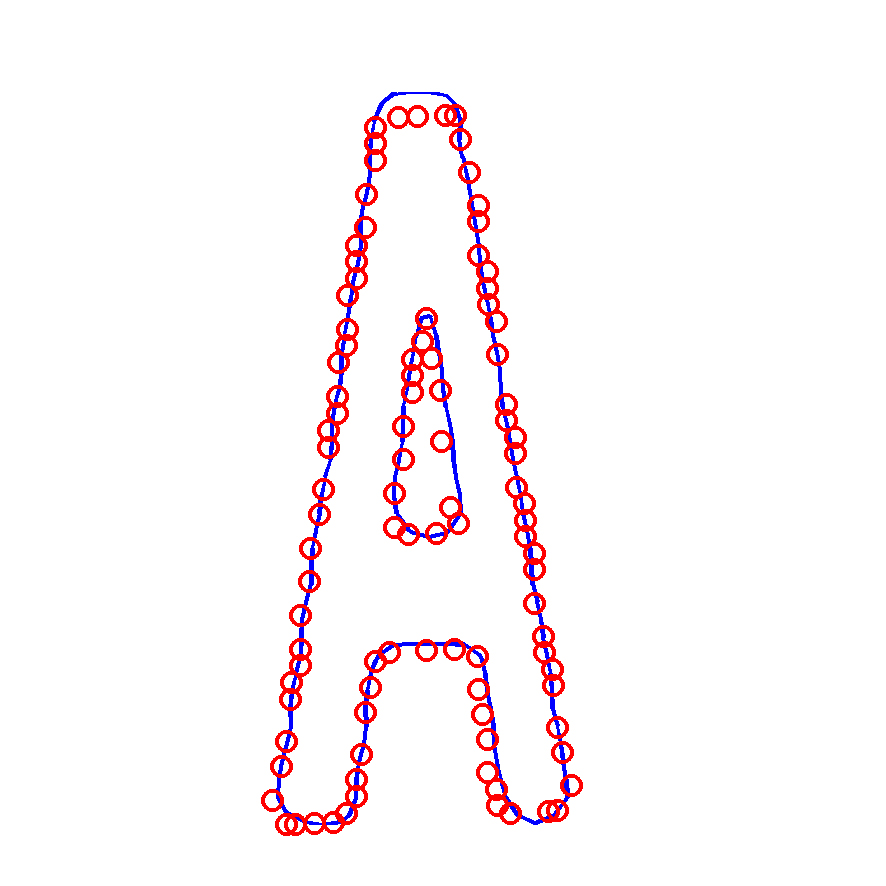}}\\
\caption{The PVSE and the affine shape matching by shape context. The blue curves are the contour curves due to PVSE, the red crosses are the results of shape matching by shape context. The first row corresponds to a PVSE by similarity variations;
The last row corresponds to a PVSE by $3^{rd}$ order prior vibrations.}
\label{FIG:PVSE_vs_ShapeContext}
\end{figure}

\begin{table}
\centering
\caption{Haussdorff shape distance corresponding to Fig. \ref{FIG:PVSE_vs_ShapeContext}}\label{TB:H_dist}
\begin{tabular*}{0.5\textwidth}{@{\extracolsep{\fill}}ccccc}
\toprule
  Prior Variation          & Dist. type &\multicolumn{3}{c}{Dist.}  \\  \cline{3-5}
               && (a) &  (b) & (c) \\ \hline
  Similarity   & Similarity &1.85 &  1.36  & 1.18 \\
  Affine       & Affine &2.44 &  1.54  & 1.66 \\
\bottomrule
\end{tabular*}
\end{table}

\subsection{Evaluation of the PVSE for shape extraction}
In this subsection, we evaluate the proposed method for shape extraction in the presence of rigid/nonrigid deformations, missing parts and/or shape overlapping.

\subsubsection{Experiment configurations}\label{EXP:Monkey_Horse}
In the experiments, we assume that the type of warping is known. We experiment on simple images to ensure that the active contour model is correct for modeling the object boundary/region, and our focus is on the deformable object model, namely PVSE. 

The size of the image for the monkey and the monkey with missing parts, the horse and the horse merged with the rider as shown in Figs. \ref{Fig:Intro_Ocl} and \ref{Fig:Intro_Exp}, is $128\times128$. We use Figs. \ref{Fig:Intro_Ocl}(a) and \ref{Fig:Intro_Exp}(a) as the template shapes. We randomly generate the parameters for $50$ similarity transformations, we also generate $50$ random non-rigid warping mappings in the form of Eq.(\ref{EQ:Harmo}). Then we apply the corresponding rigid and non-rigid deformations to the monkey and horse shapes in Figs. \ref{Fig:Intro_Ocl}(d) and \ref{Fig:Intro_Exp}(d) to obtain a total of $200 + 200 = 400$ shapes for evaluation. The probability density distributions that are used to generate random parameters are summarized in Table \ref{TB:Par}. In the experiment, the Chan-Vese active contour is chosen for modeling the object region. The prior variations are chosen as the ones corresponding to the deformations used for generating the test samples. We also use the deformed template shapes as the ground truth results for evaluation purpose.
\begin{table}
  \centering
\caption{The distributions of randomly generated parameters from uniform distributions}\label{TB:Par}
{\footnotesize\begin{tabular}{ccc}
\toprule
  Transformations & Parameters & Distributions \\\hline
  \multirow{3}{*}{\begin{tabular}{c}
                    Rigid, \\
                    $\vec{x}\in[-1,1]\times[-1,1]$ \\
                  \end{tabular}
  } & $\theta$ & $U(-\pi/4,\pi/4)$ \\\cline{2-3}
   & $\lambda$ & $U(-0.9,1.1)$ \\\cline{2-3}
   & $a$,$b$ & $U(-0.1,0.1)$ \\\hline
  \multirow{3}{*}{\begin{tabular}{c}
                    Non-rigid, \\
                    $\vec{x}\in[0,1]\times[0,1]$ \\
                  \end{tabular}} & $\xi^k_{mn},$ & \multirow{3}{*}{$U(-0.125,0.125)$} \\
  & $1\leq m+n\leq3,$&\\
  &$k=1,2$&\\
  \bottomrule
\end{tabular}}
\end{table}

\subsubsection{Statistical analysis of the performance of shape extraction}
When presenting the results, we used some abbreviations for the datasets. The abbreviations are summarized in Table \ref{TB:Label}.

The performance is measured by comparing the results with the ground truth, and the performance measure we adopted is the Jaccard's similarity measure. The Jaccard's similarity measure is defined as the overlapping ratio of two sets. We desire large value of the measure. We have no comparisons for rigid transformed shapes. We first present the shape variation in each test set in Fig. \ref{Fig:boxplot_cmp}. The large shape variations imply the difficulty of the shape registration.

The averages of the Jaccard similarity measures for our rigid PVSE on MRM and HRM are 0.96 and 0.94 respectively, and the corresponding standard deviations are 0.0033 and 0.033. For the non-rigidly deformed shapes, we have compared our method with the $H_1$ and $H_2$ gradients in \cite{Charpiat2007GG} as well as the rigid PVSE. The results are shown in Fig. \ref{Fig:boxplot_cmp}. The results indicate that our method significantly outperforms other methods. Figs. \ref{Fig:Var} and \ref{Fig:boxplot_cmp} jointly prove that the proposed method is robust to large shape variations.
\begin{table}
\centering
\caption{Abbreviations}\label{TB:Label}
{\begin{tabular}{cl}
\hline
\textbf{MR}:     & rigidly transformed ground truth of monkey\\
\textbf{MRM}: & rigidly transformed monkey with missing parts\\
\textbf{MNR}:    & non-rigidly transformed ground truth of monkey \\
\textbf{MNRM}:& non-rigidly transformed monkey with missing parts\\
\textbf{HR}:     & rigidly transformed ground truth of horse\\
\textbf{HRM}: & rigidly transformed horse merged with rider\\
\textbf{HNR}:    & non-rigidly transformed ground truth of horse\\
\textbf{HNRM}:& non-rigidly transformed horse merged with rider\\
\hline
\end{tabular}}
\end{table}

We also present some examples of the shape extraction results in Figs. \ref{Fig:CMP_horse&monkey}. From the figure, we are able to understand the behavior of each method. We can observe that the rigid PVSE can retrieve the correct size and main orientation of the object shape. The $H_1$ and $H_2$ gradient based shape evolution generally do not fit to the non-rigid deformation well, although the $H_1$ gradient based shape evolution can preserve the shape to some extent. The $H_2$ gradient based shape evolution does not preserve the original shape, but the converged shapes have a particular style. We can conclude that both of them are not suitable for the deformations appeared in this experiment.

\begin{figure}[htb]
\centering
  \includegraphics[width=0.8\columnwidth]{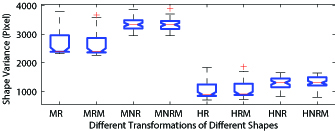}
\caption{Shape Variation: the distances between the initial prototype shapes and the shapes to be registered.}\label{Fig:Var}
\end{figure}

\begin{figure}[htb]
\centering
  \subfloat[]{\includegraphics[width=0.45\columnwidth]{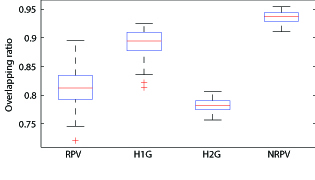}}~
  \subfloat[]{\includegraphics[width=0.45\columnwidth]{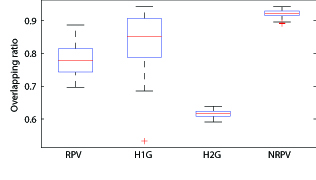}}
\caption{Comparison of different methods for shape extraction under non-rigid deformations, missing parts or overlapping shapes. (a) is the result for the non-rigidly deformed toy monkey missing feet and one ear; (b) is the result for the non-rigidly deformed horse overlapped with its rider.}\label{Fig:boxplot_cmp}
\end{figure}

\begin{figure*}
\centering
\subfloat{\includegraphics[height=0.4in]{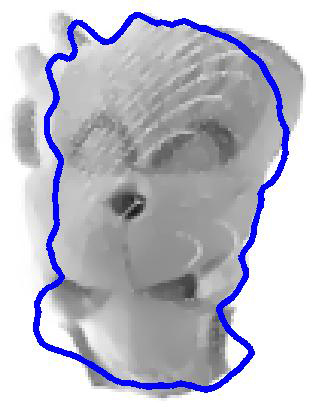}
\includegraphics[height=0.4in]{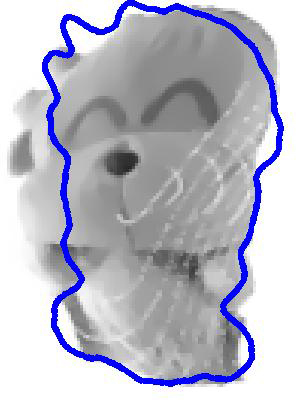}
\includegraphics[height=0.4in]{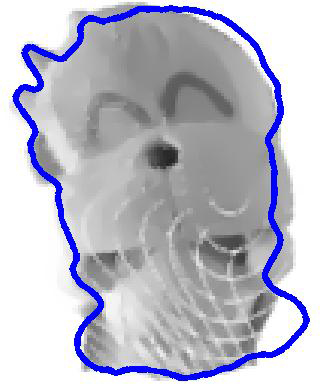}
\includegraphics[height=0.4in]{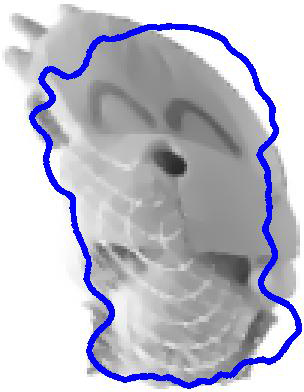}
\includegraphics[height=0.4in]{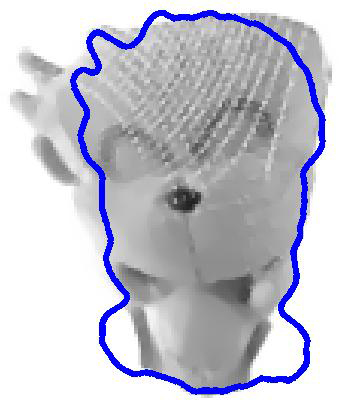}
\includegraphics[height=0.4in]{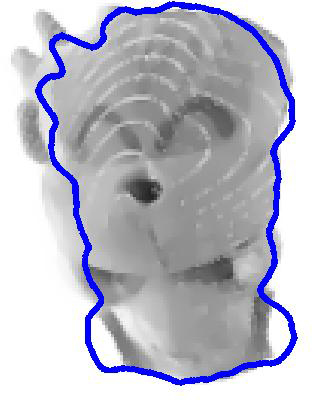}
\includegraphics[height=0.4in]{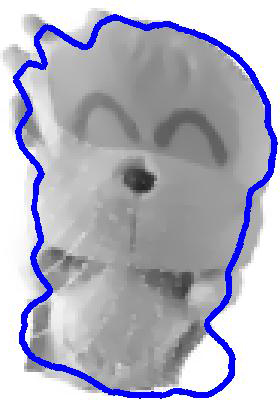}}~
\subfloat{\includegraphics[height=0.4in]{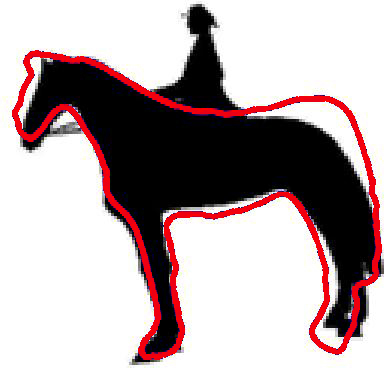}
\includegraphics[height=0.4in]{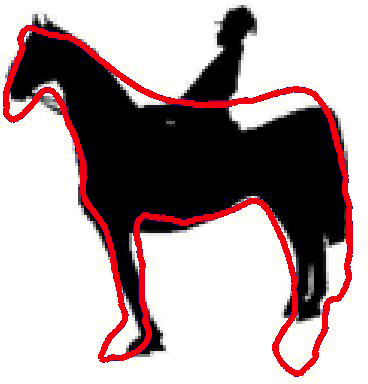}
\includegraphics[height=0.4in]{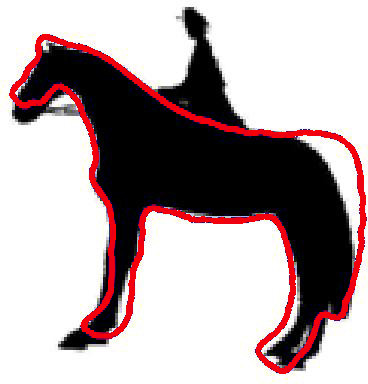}
\includegraphics[height=0.4in]{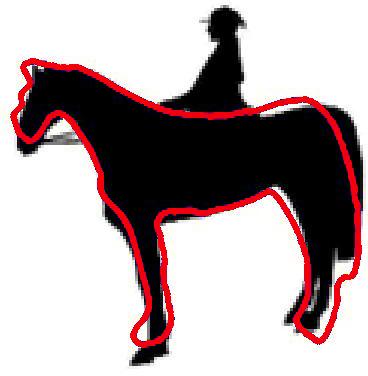}
\includegraphics[height=0.4in]{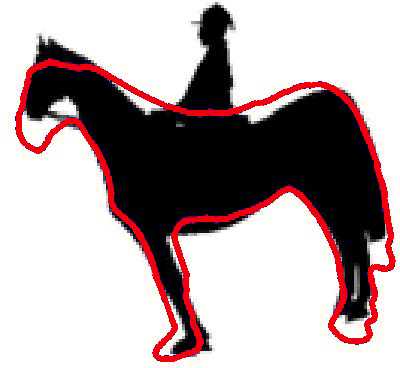}
\includegraphics[height=0.4in]{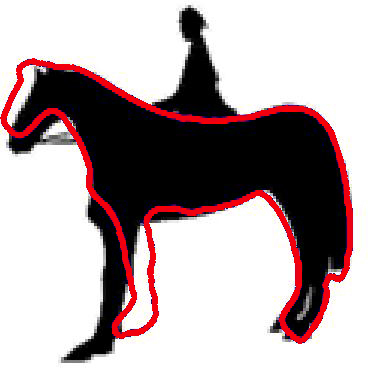}
\includegraphics[height=0.4in]{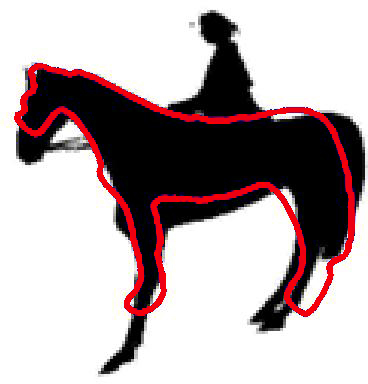}}\\
\subfloat{\includegraphics[height=0.4in]{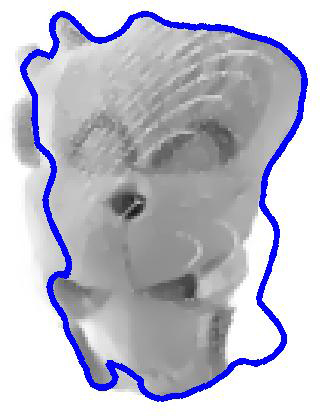}
\includegraphics[height=0.4in]{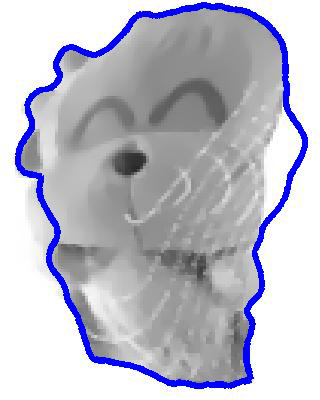}
\includegraphics[height=0.4in]{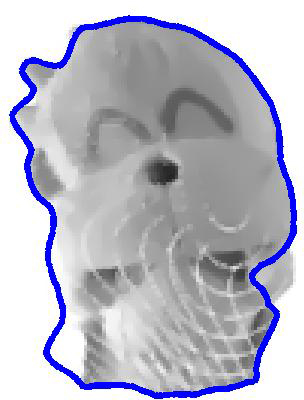}
\includegraphics[height=0.4in]{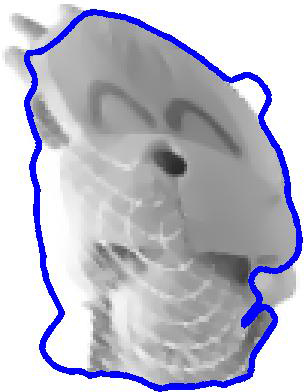}
\includegraphics[height=0.4in]{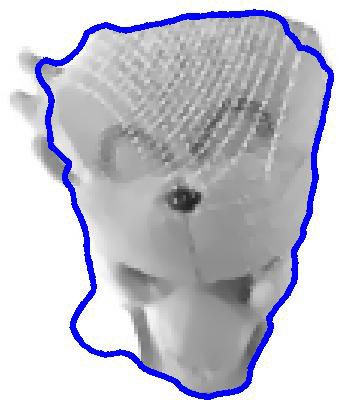}
\includegraphics[height=0.4in]{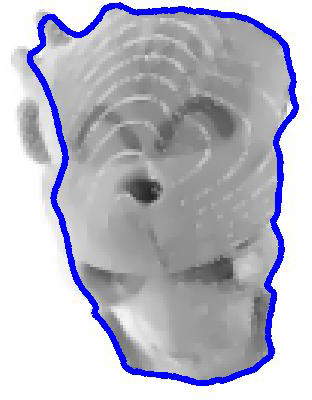}
\includegraphics[height=0.4in]{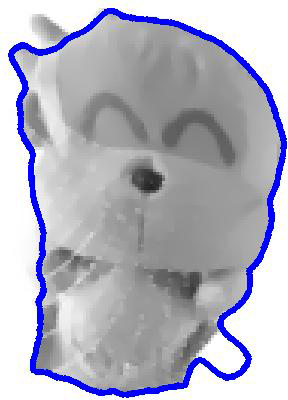}}~
\subfloat{\includegraphics[height=0.4in]{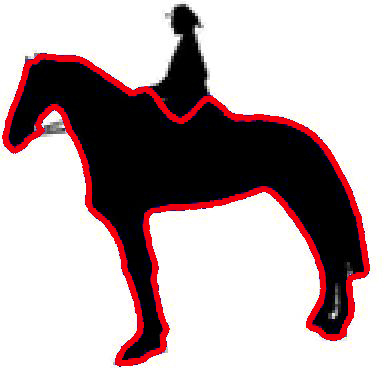}
\includegraphics[height=0.4in]{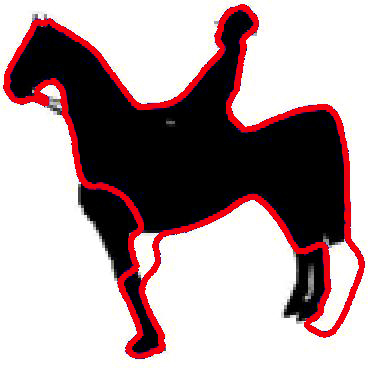}
\includegraphics[height=0.4in]{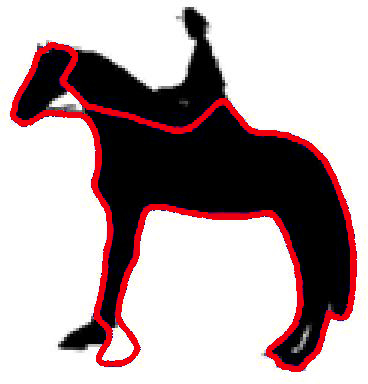}
\includegraphics[height=0.4in]{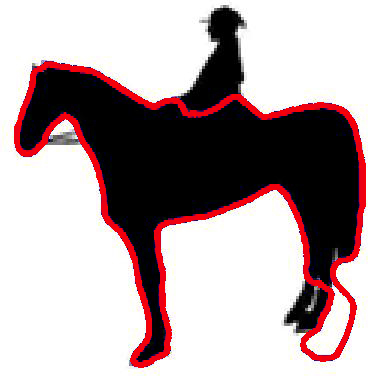}
\includegraphics[height=0.4in]{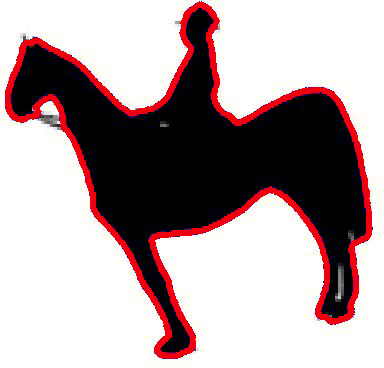}
\includegraphics[height=0.4in]{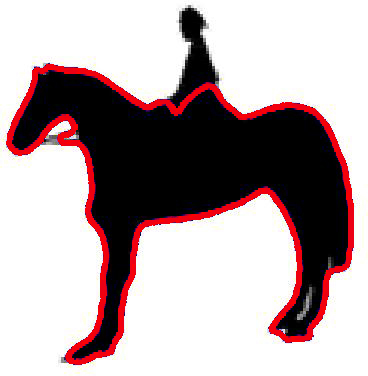}
\includegraphics[height=0.4in]{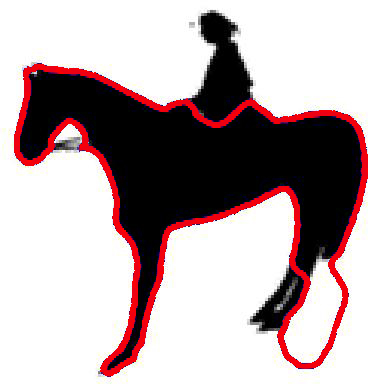}}\\
\subfloat{\includegraphics[height=0.4in]{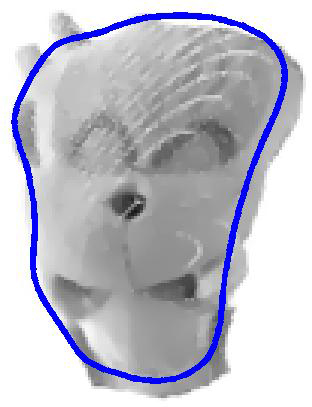}
\includegraphics[height=0.4in]{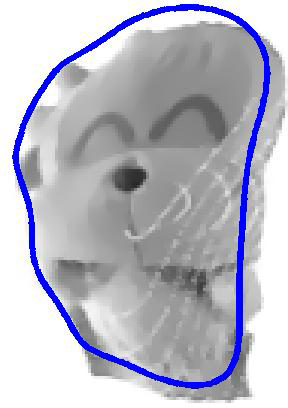}
\includegraphics[height=0.4in]{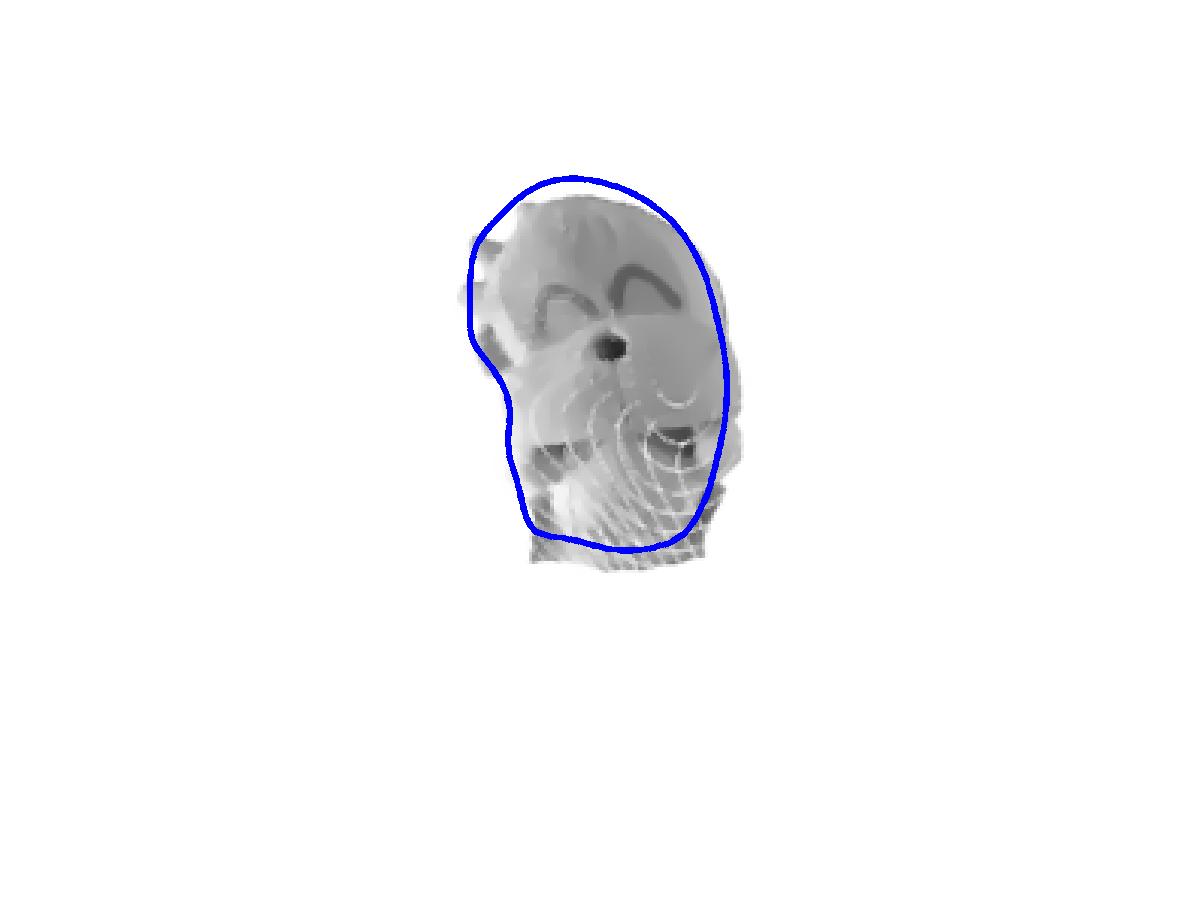}
\includegraphics[height=0.4in]{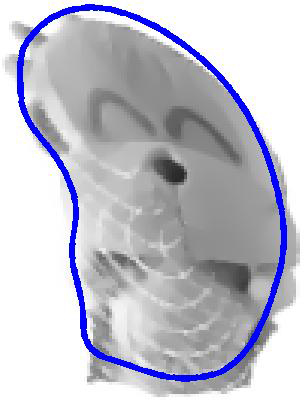}
\includegraphics[height=0.4in]{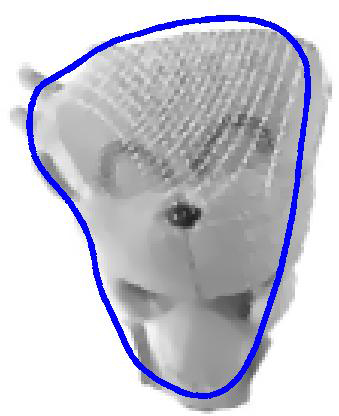}
\includegraphics[height=0.4in]{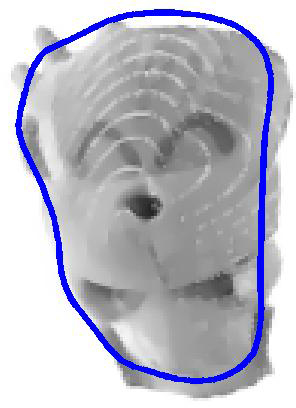}
\includegraphics[height=0.4in]{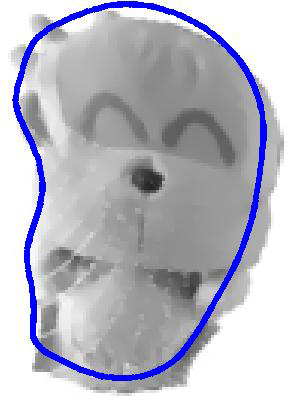}}~
\subfloat{\includegraphics[height=0.4in]{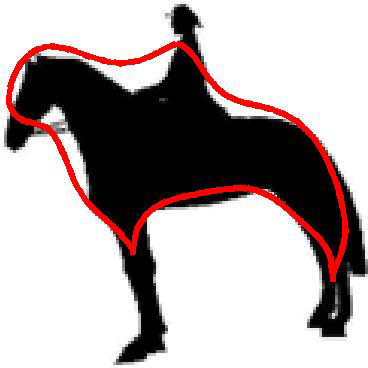}
\includegraphics[height=0.4in]{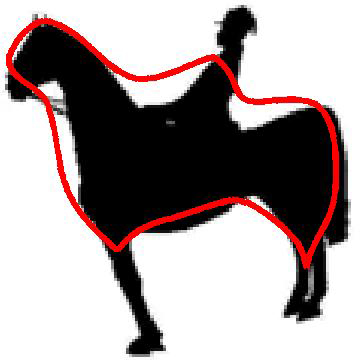}
\includegraphics[height=0.4in]{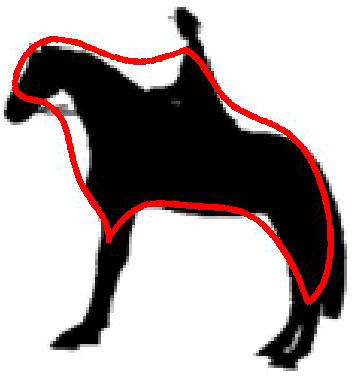}
\includegraphics[height=0.4in]{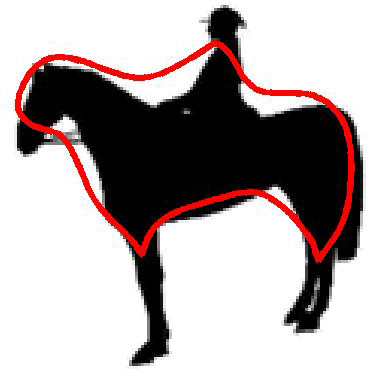}
\includegraphics[height=0.4in]{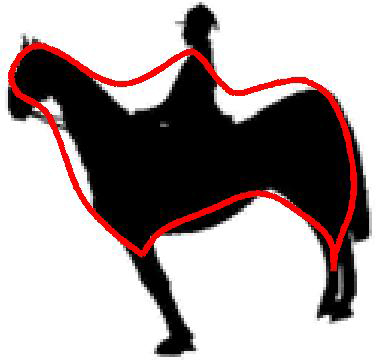}
\includegraphics[height=0.4in]{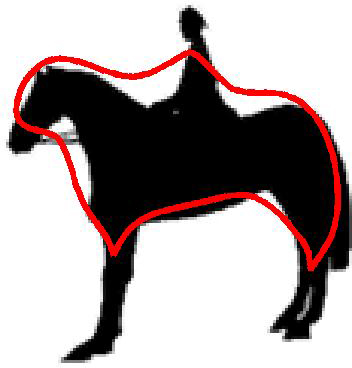}
\includegraphics[height=0.4in]{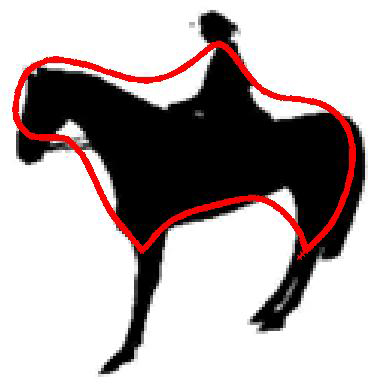}}\\
\subfloat{\includegraphics[height=0.38in]{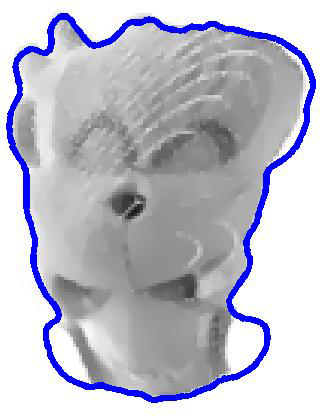}
\includegraphics[height=0.38in]{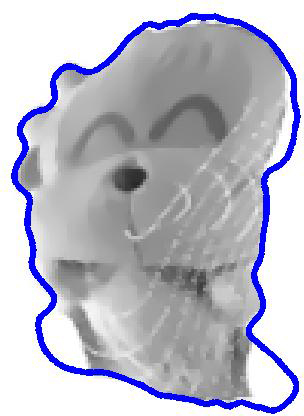}
\includegraphics[height=0.38in]{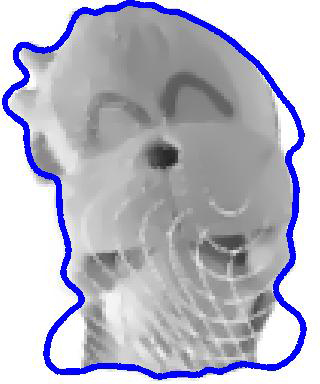}
\includegraphics[height=0.38in]{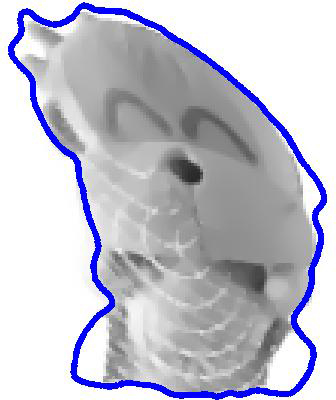}
\includegraphics[height=0.38in]{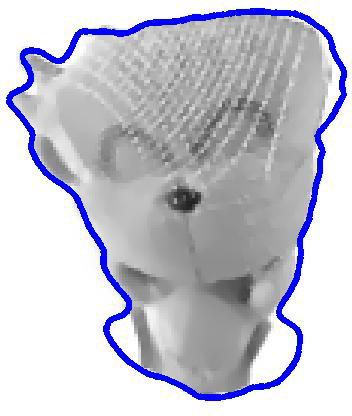}
\includegraphics[height=0.38in]{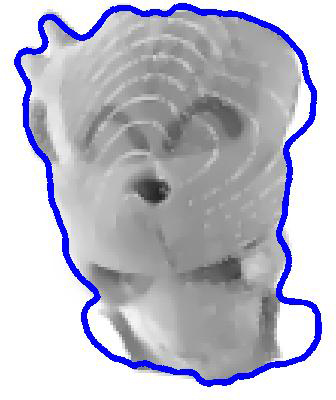}
\includegraphics[height=0.38in]{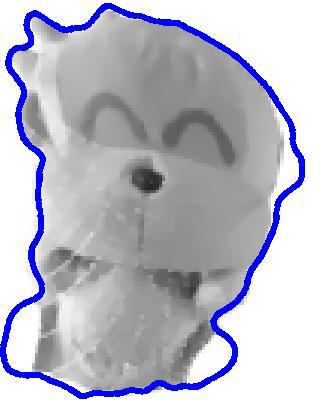}}~
\subfloat{\includegraphics[height=0.4in]{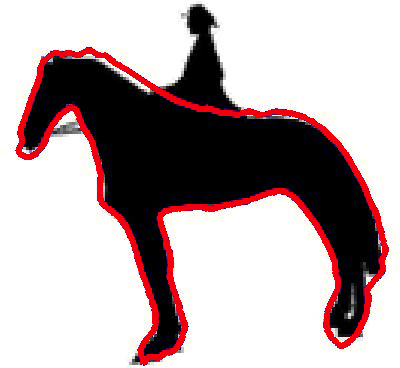}
\includegraphics[height=0.4in]{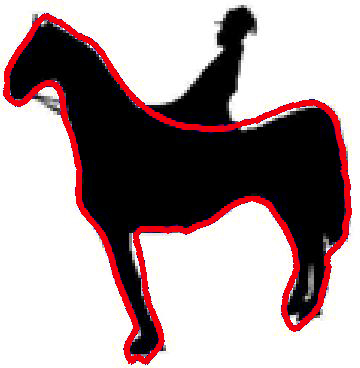}
\includegraphics[height=0.4in]{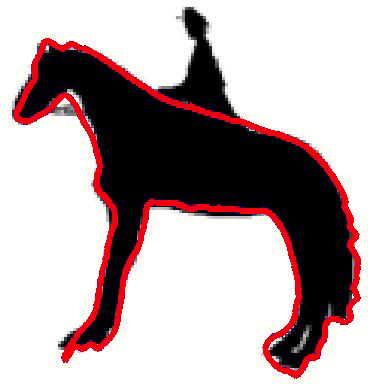}
\includegraphics[height=0.4in]{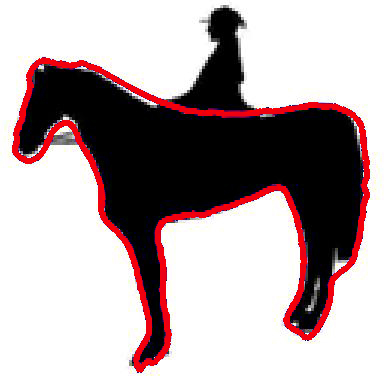}
\includegraphics[height=0.4in]{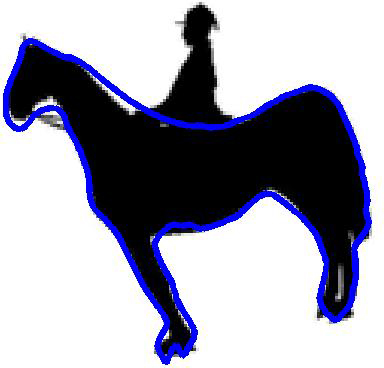}
\includegraphics[height=0.4in]{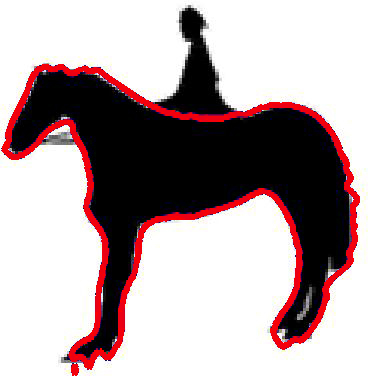}
\includegraphics[height=0.4in]{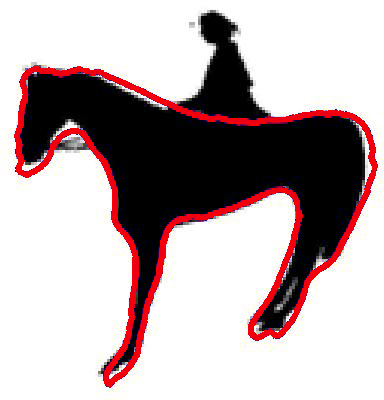}}
\caption{Examples of the results for shape extraction under nonrigid deformations, missing parts or overlapping shapes. The first row shows the results by rigid PVSE. The second and third rows show the results by $H_1$ and $H_2$ functional gradient \cite{Charpiat2007GG}. The last row shows the results of our non-rigid PVSE.}\label{Fig:CMP_horse&monkey}
\end{figure*}

\subsection{Qualitative analysis of shape evolution process for shape extraction}
In this subsection, we take a closer look at the optimization process, namely the shape evolution. Some examples of the shape evolution process for  the\ref{Fig:PVSE_NR_OC}, \ref{Fig:PVSE_NR_Exp} and \ref{Fig:PVSE_NR_Exp2}, which are sample results from this experiment. From the figures, we can observe that the shape evolution process is plausible, and it may also be viewed as the interpolation of the shape deformation. 
\begin{figure*}
\centering
\subfloat[]{\includegraphics[width=0.175\textwidth]{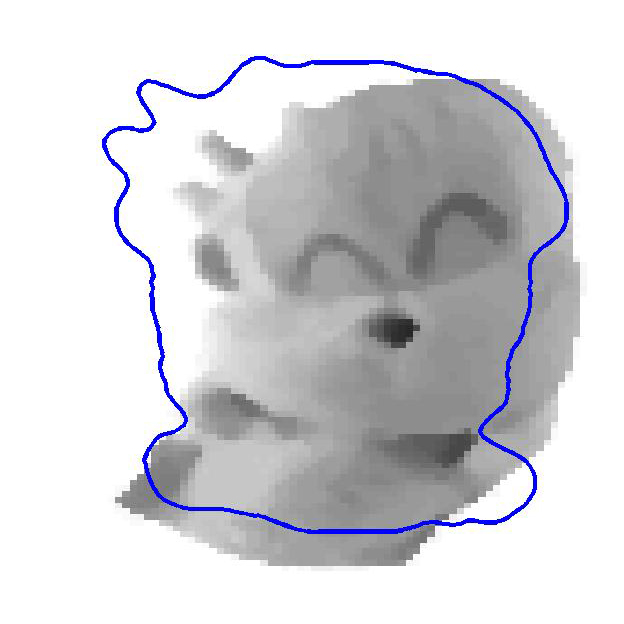}}
\subfloat[]{\includegraphics[width=0.175\textwidth]{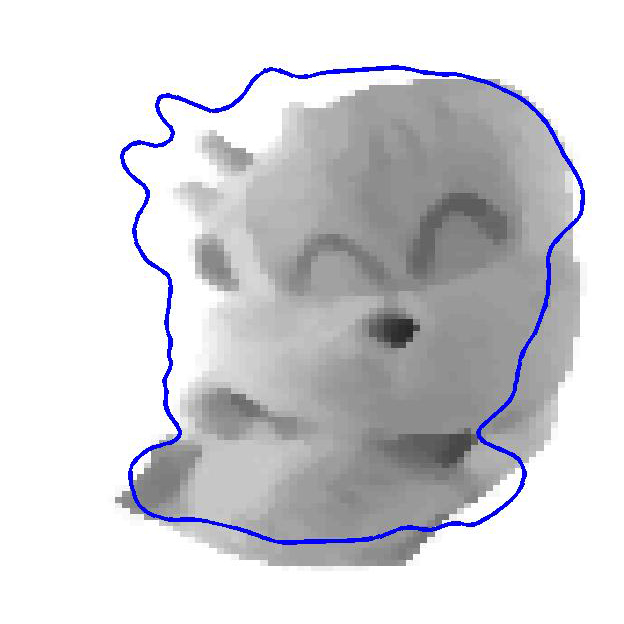}}
\subfloat[]{\includegraphics[width=0.175\textwidth]{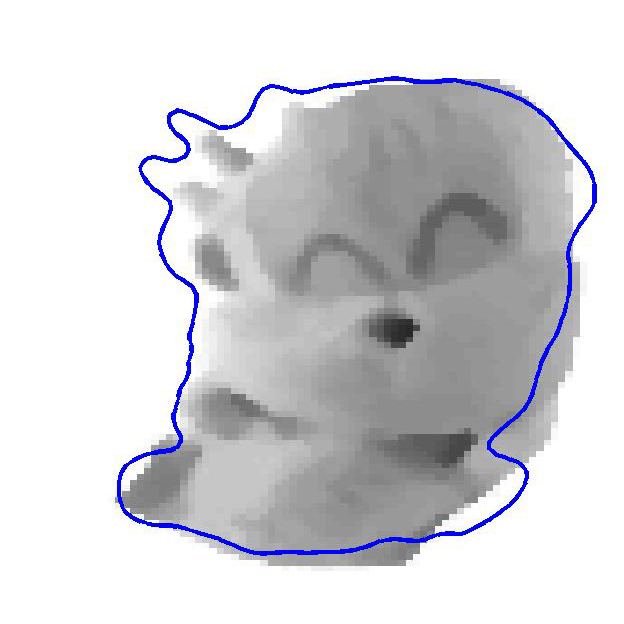}}
\subfloat[]{\includegraphics[width=0.175\textwidth]{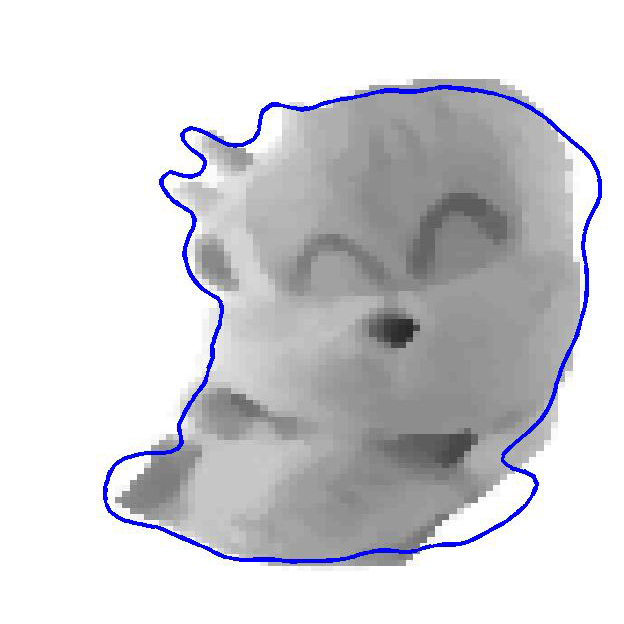}}
\subfloat[]{\includegraphics[width=0.175\textwidth]{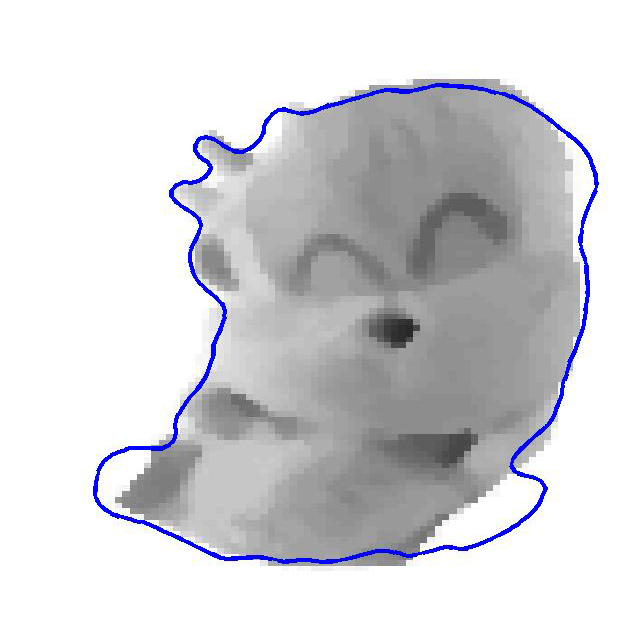}}
\caption{PVSE with prior vibrations for extracting the toy monkey shape. (a) to (e) show the curve evolution of 90 iterations until convergence at (e).}\label{Fig:PVSE_NR_OC}
\end{figure*}
\begin{figure*}
\centering
\subfloat[]{\includegraphics[width=0.175\textwidth]{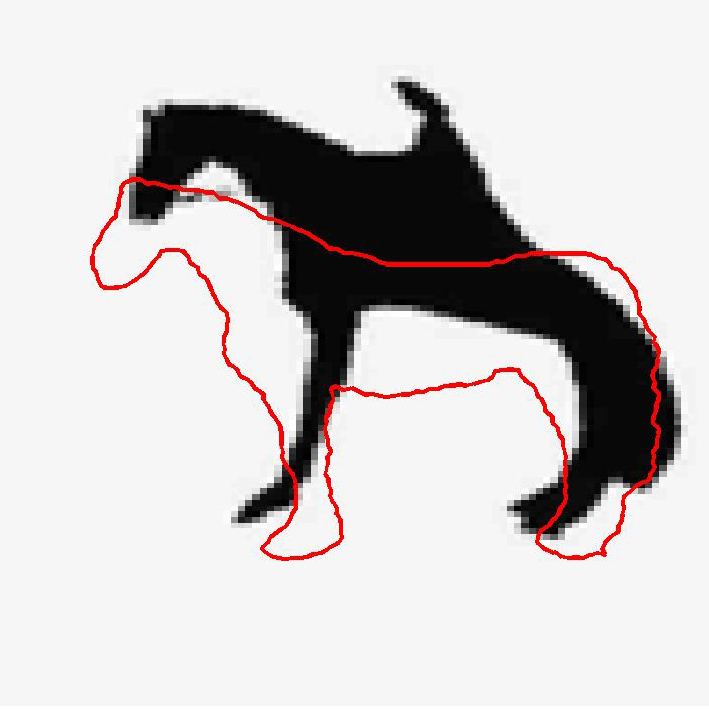}}
\subfloat[]{\includegraphics[width=0.175\textwidth]{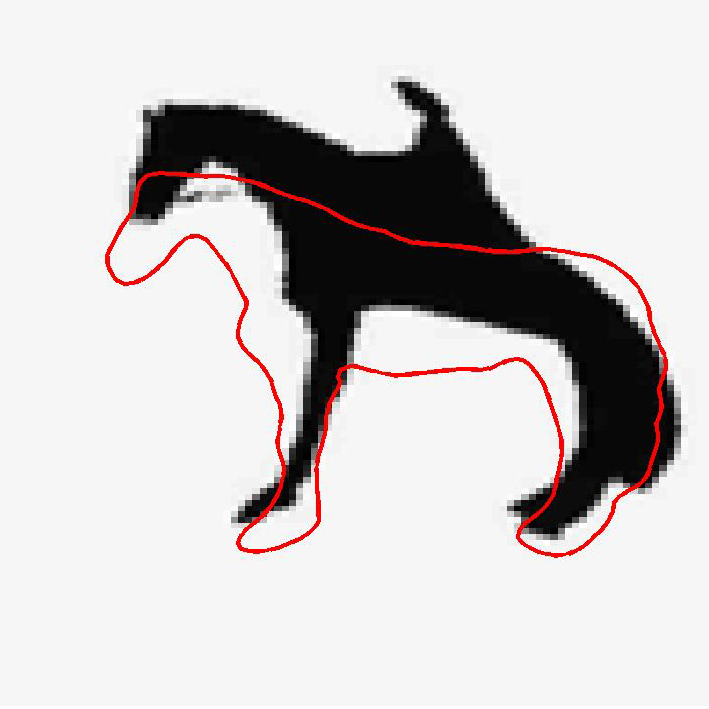}}
\subfloat[]{\includegraphics[width=0.175\textwidth]{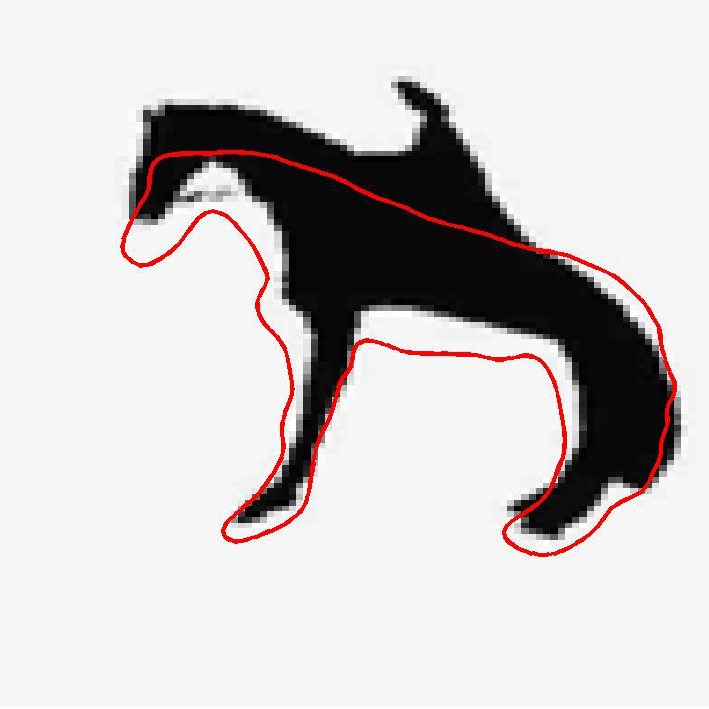}}
\subfloat[]{\includegraphics[width=0.175\textwidth]{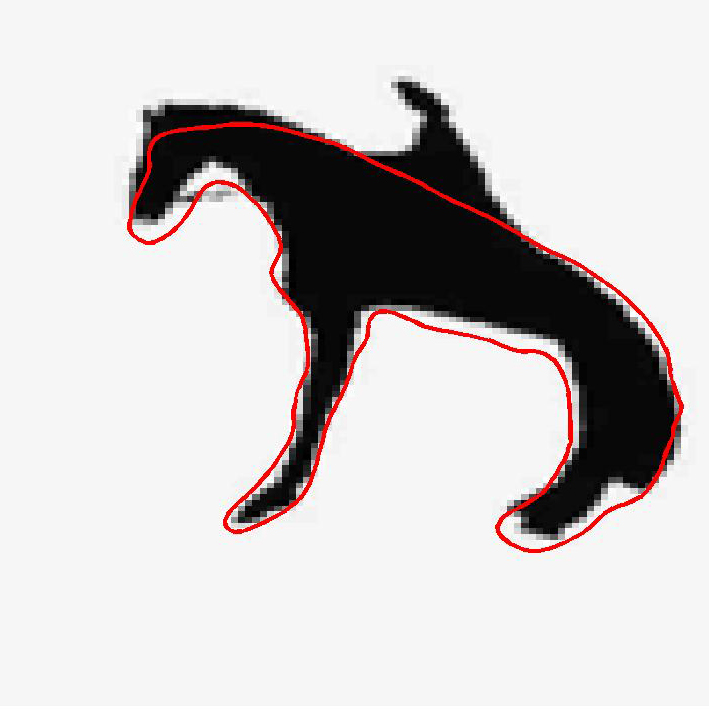}}
\subfloat[]{\includegraphics[width=0.175\textwidth]{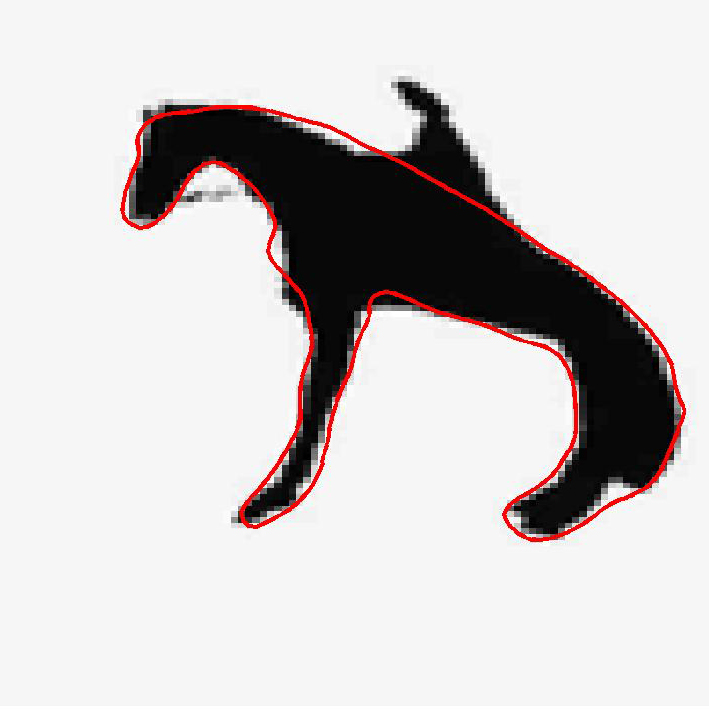}}
\caption{PVSE with prior vibrations for extracting a horse shape. (a) to (e) show the curve evolution of 150 iterations until convergence at (e).}\label{Fig:PVSE_NR_Exp}
\end{figure*}
\begin{figure*}
\centering
\subfloat[]{\includegraphics[width=0.175\textwidth]{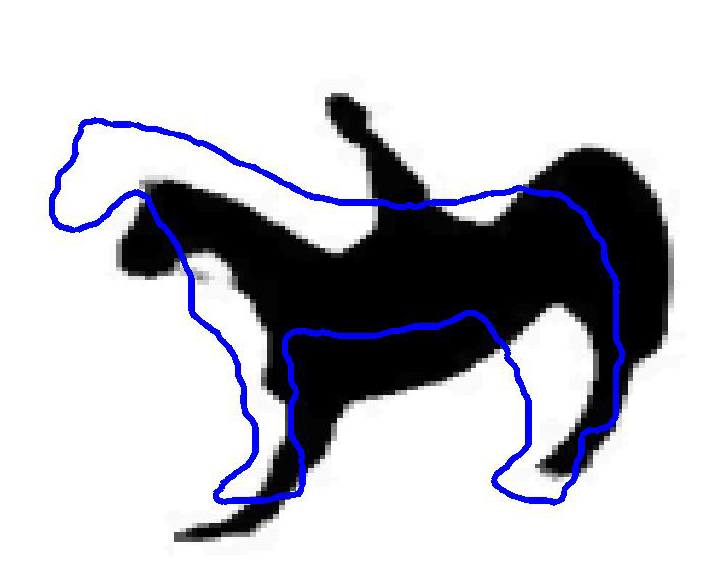}}
\subfloat[]{\includegraphics[width=0.175\textwidth]{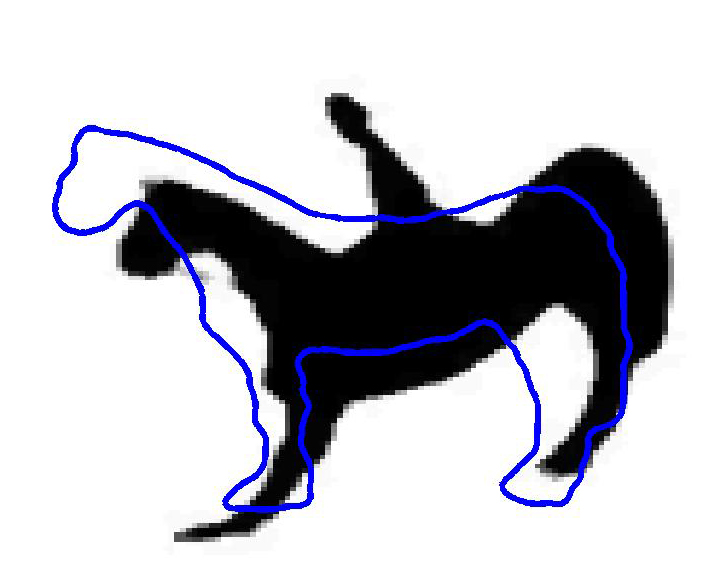}}
\subfloat[]{\includegraphics[width=0.175\textwidth]{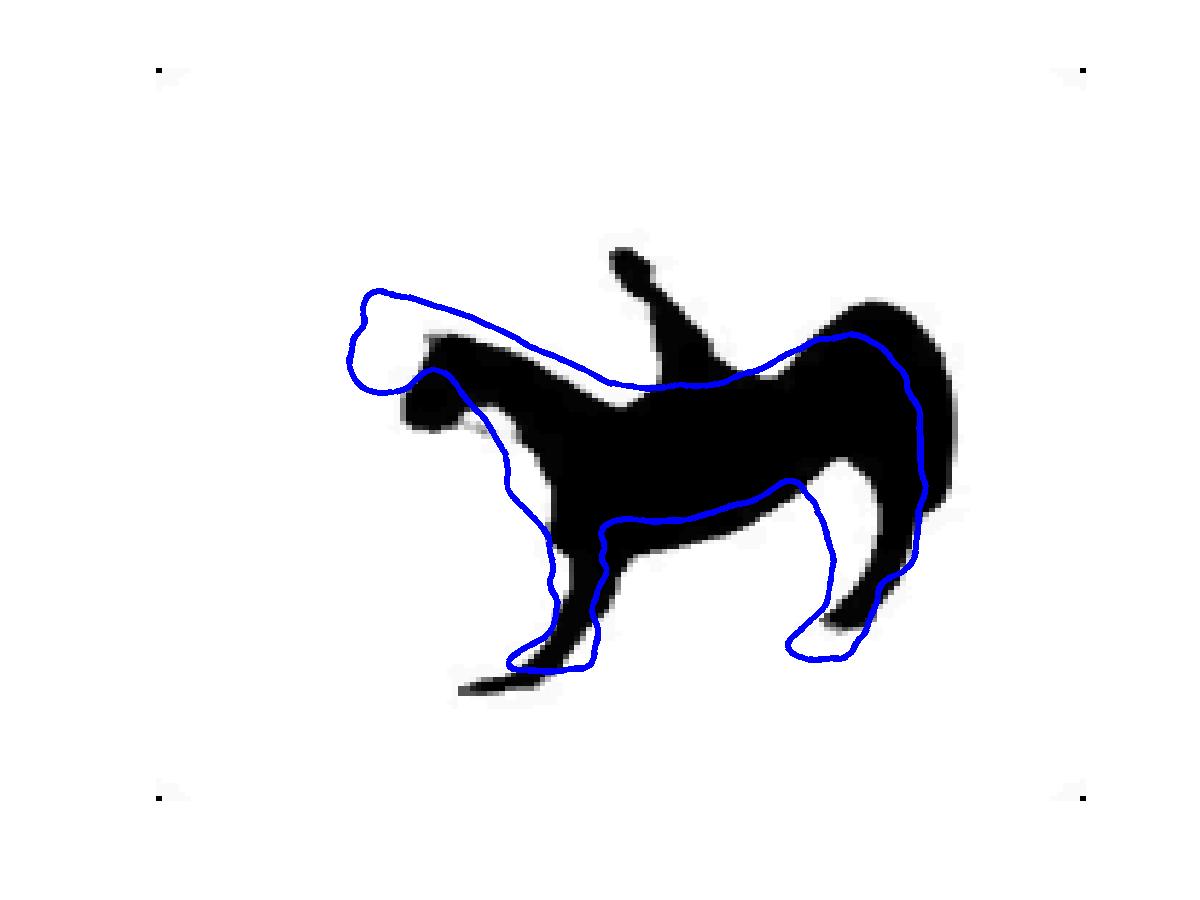}}
\subfloat[]{\includegraphics[width=0.175\textwidth]{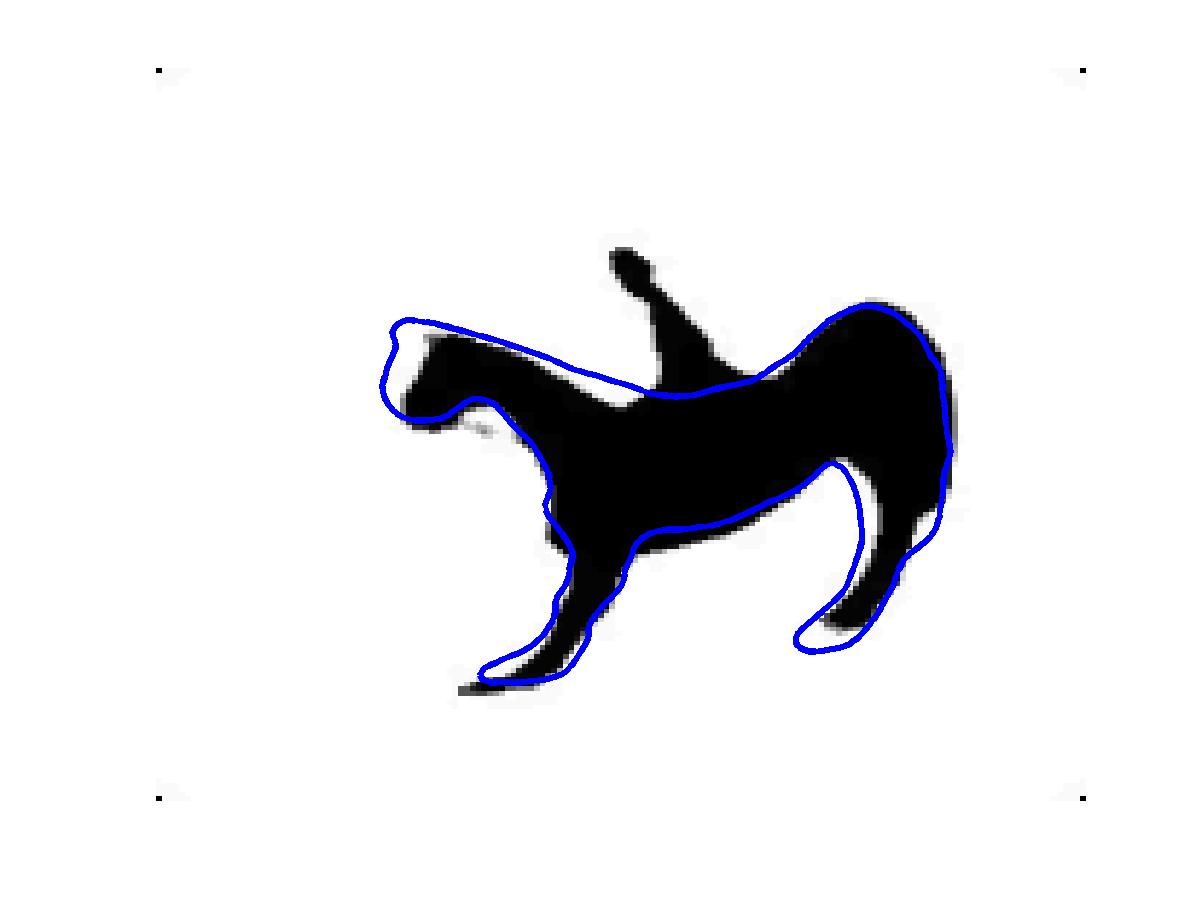}}
\subfloat[]{\includegraphics[width=0.175\textwidth]{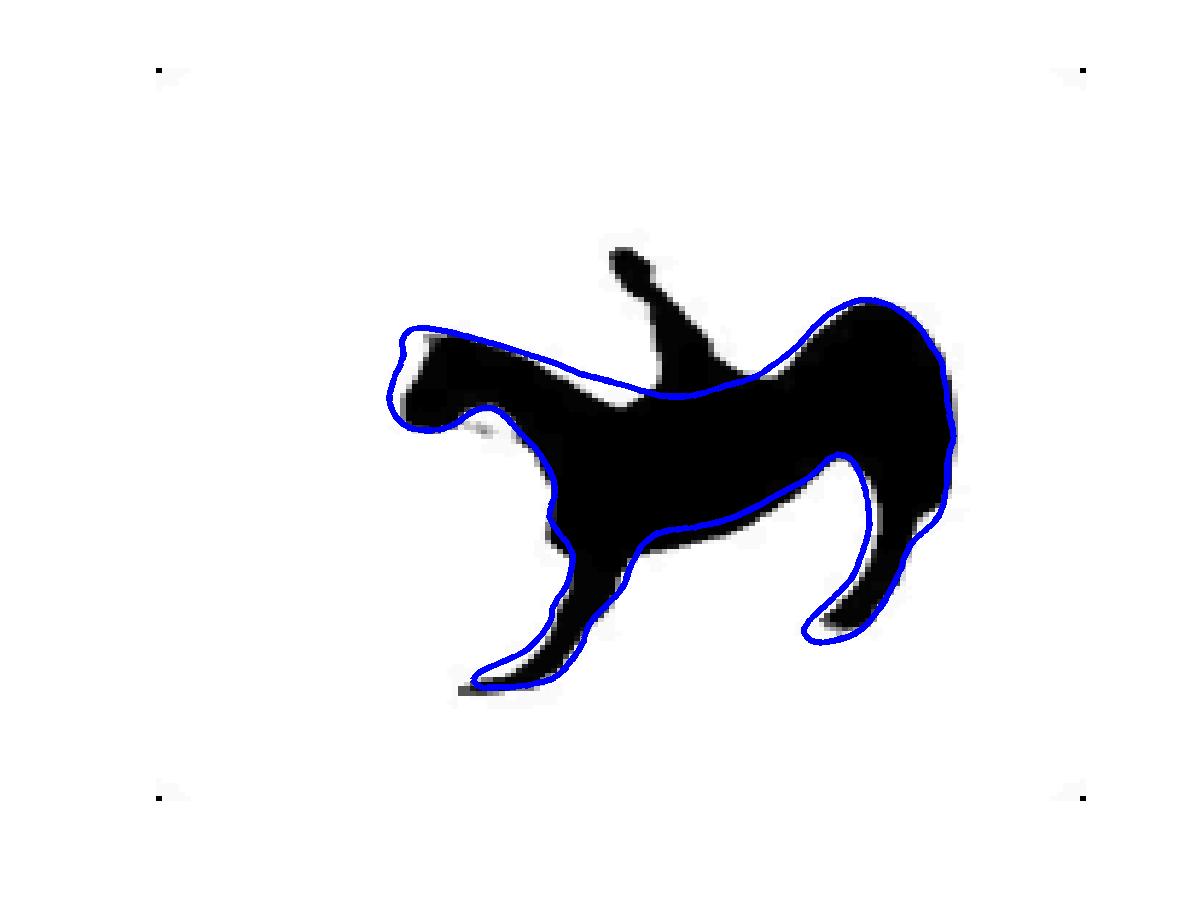}}
\caption{PVSE with prior vibrations for extracting another horse shape. (a) to (e) show the curve evolution of 500 iterations until convergence at (e).}\label{Fig:PVSE_NR_Exp2}
\end{figure*}

\begin{figure*}[thb]
\centering
\subfloat[]{\includegraphics[width=0.2\textwidth]{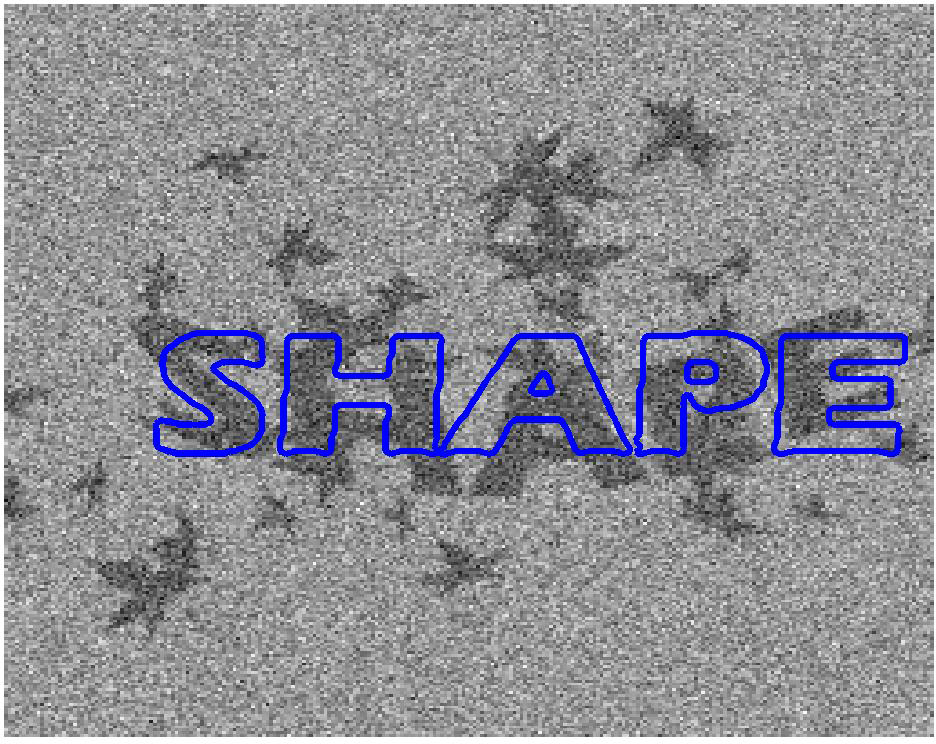}}
\subfloat[]{\includegraphics[width=0.2\textwidth]{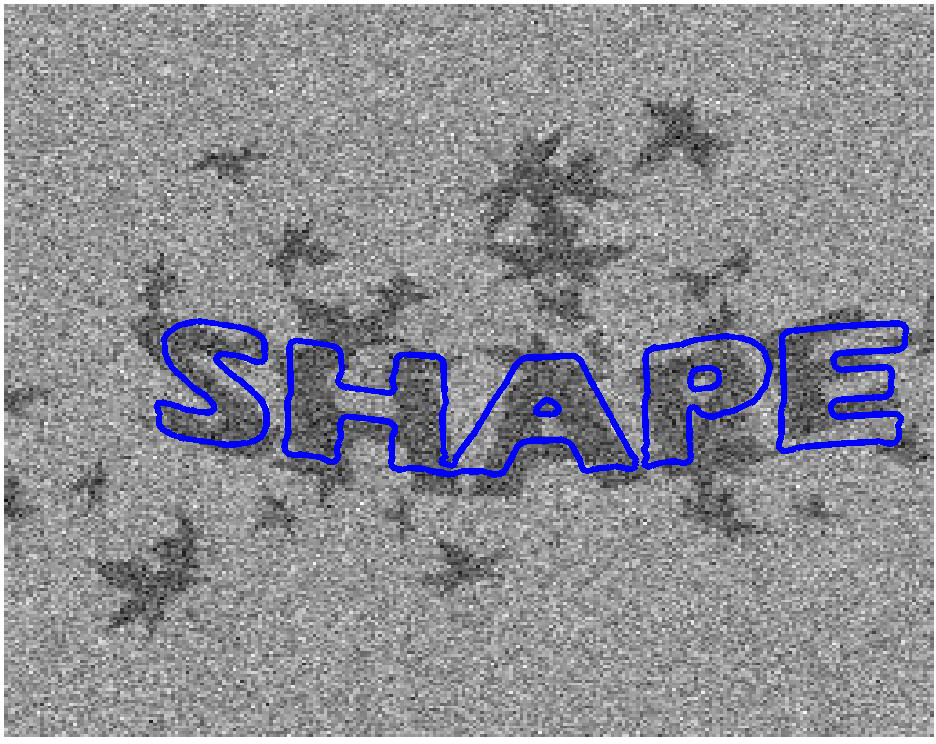}}
\subfloat[]{\includegraphics[width=0.2\textwidth]{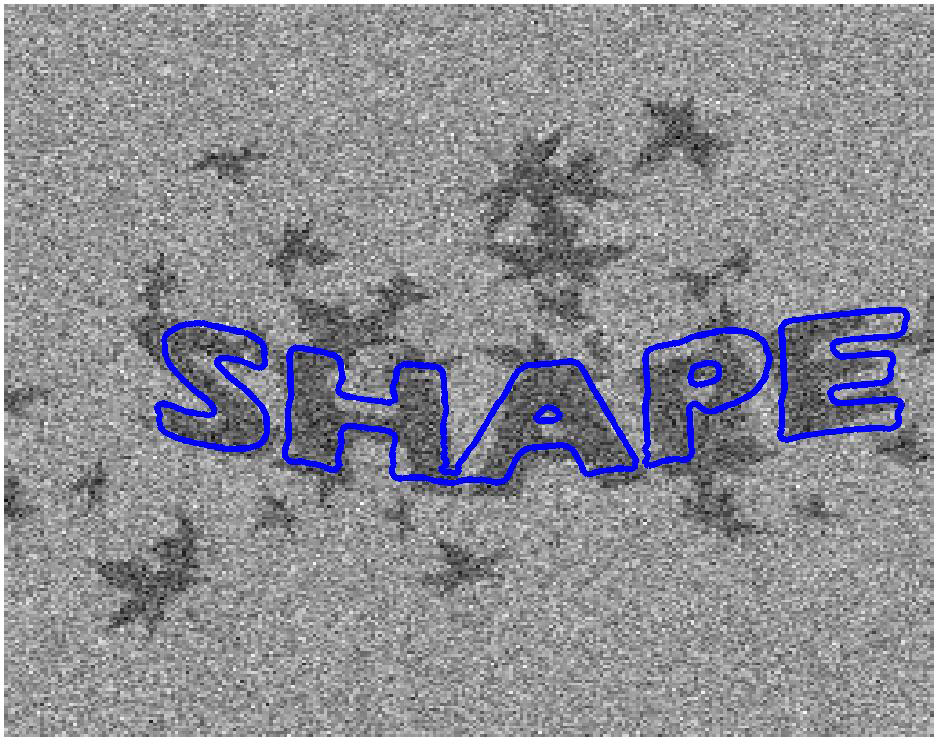}}
\subfloat[]{\includegraphics[width=0.2\textwidth]{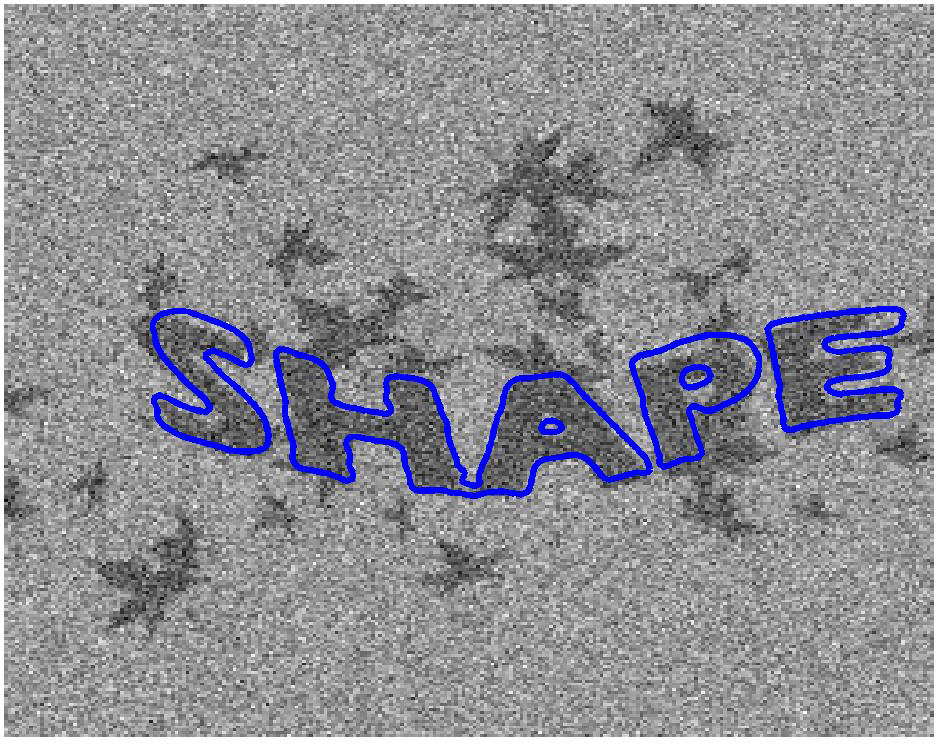}}
\subfloat[]{\includegraphics[width=0.2\textwidth]{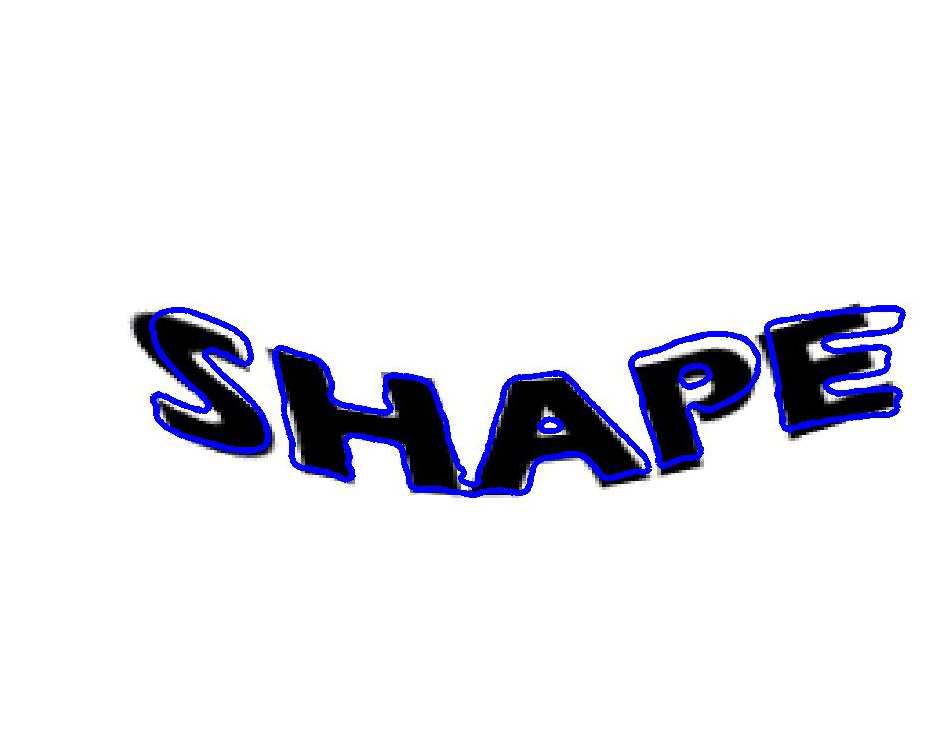}}\\
\caption{Segmentation and recovery of the \textbf{SHAPE} from heavy noise and occlusion by PVSE with the $6^{th}$-order prior vibration. The top-left is the input image. (d) is the convergent curve laying over the ground truth. (a) to (c) visualize the curve evolution of 150 iterations until convergence.}\label{Fig:SHAPE}
\end{figure*}

\begin{figure*}[thb]
\centering
\subfloat[]{\includegraphics[width=0.175\textwidth]{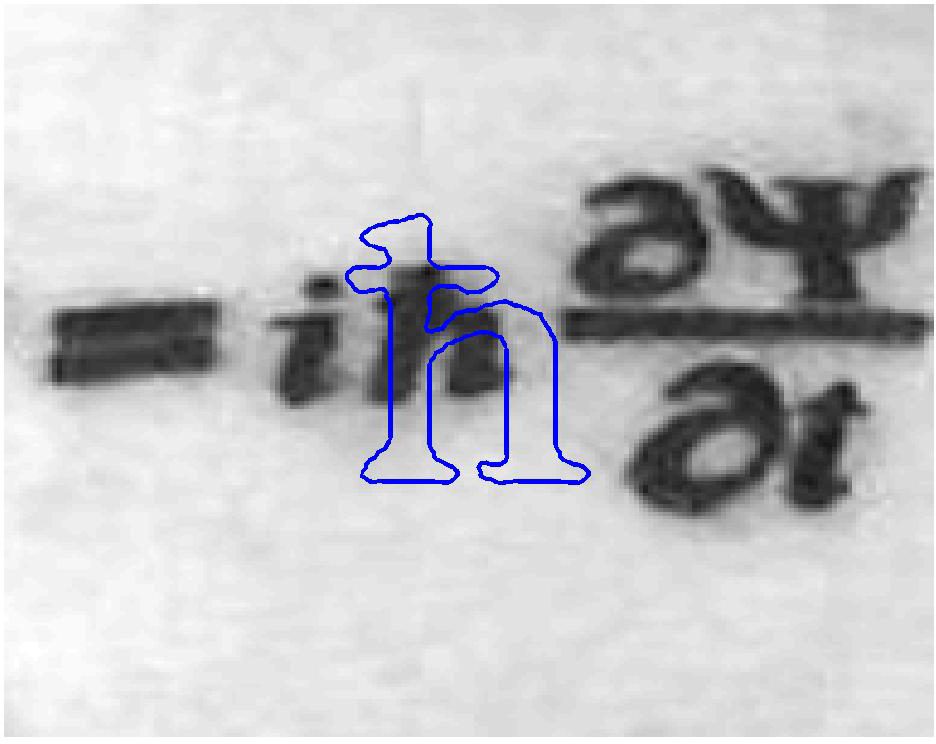}}
\subfloat[]{\includegraphics[width=0.175\textwidth]{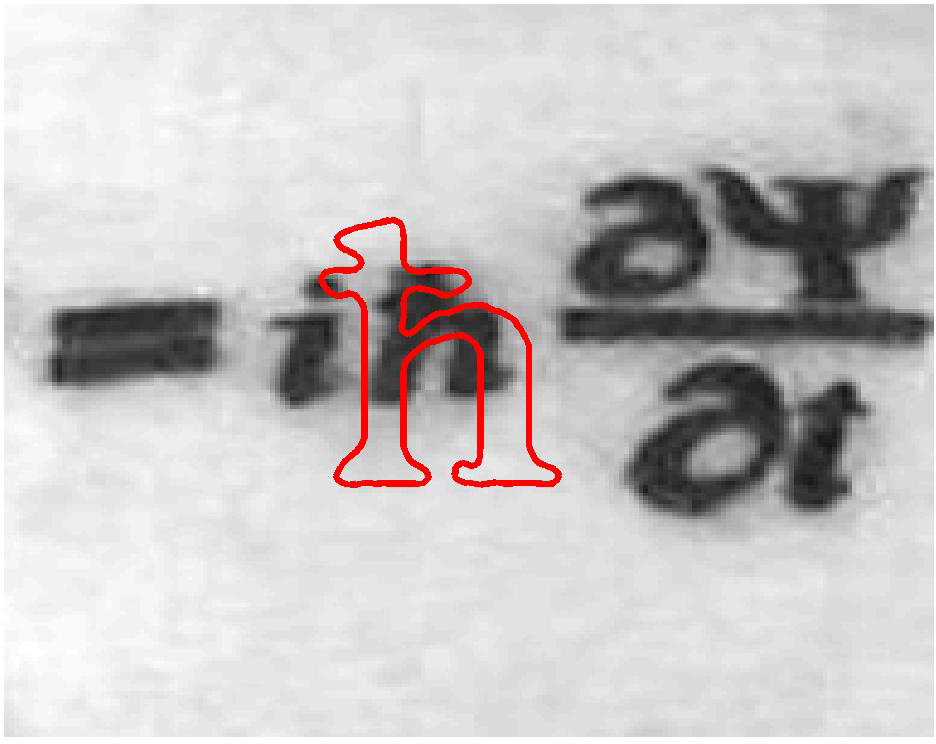}}
\subfloat[]{\includegraphics[width=0.175\textwidth]{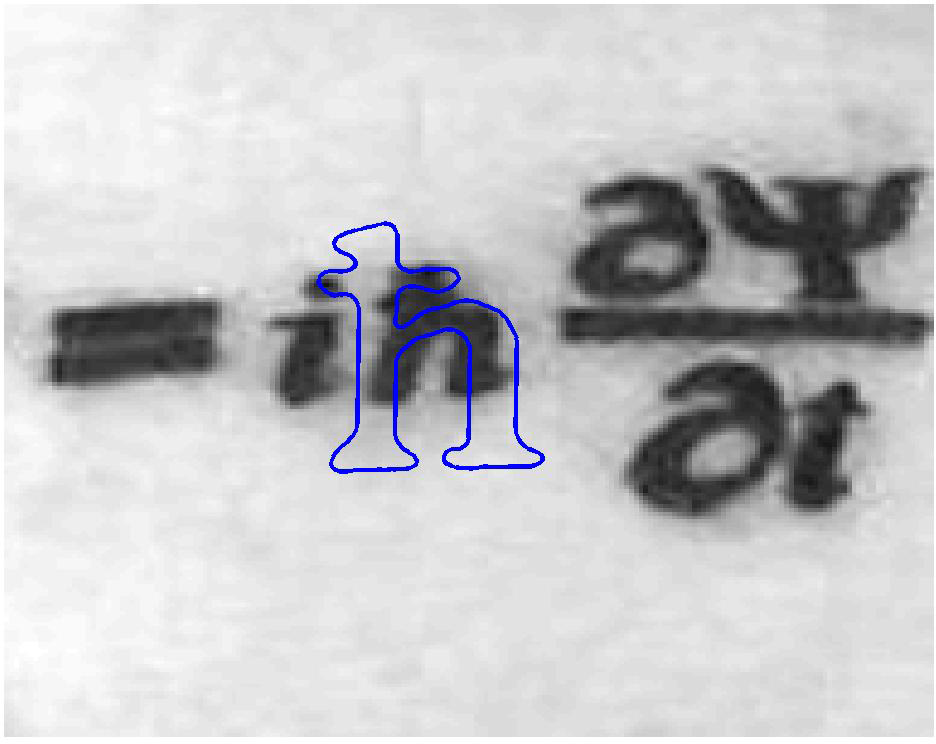}}
\subfloat[]{\includegraphics[width=0.175\textwidth]{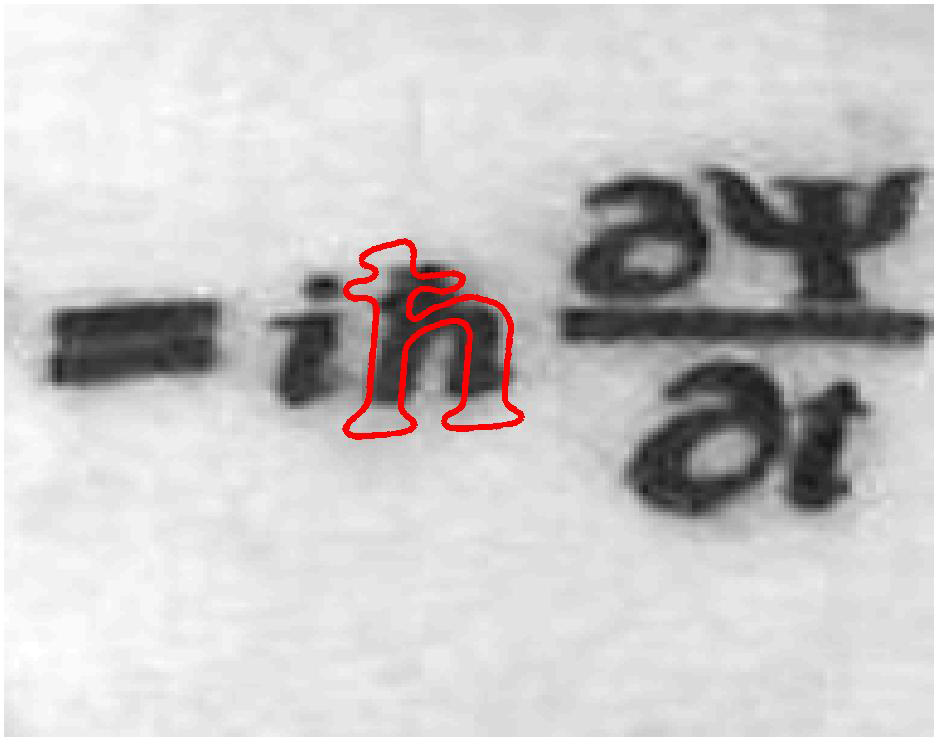}}
\subfloat[]{\includegraphics[width=0.175\textwidth]{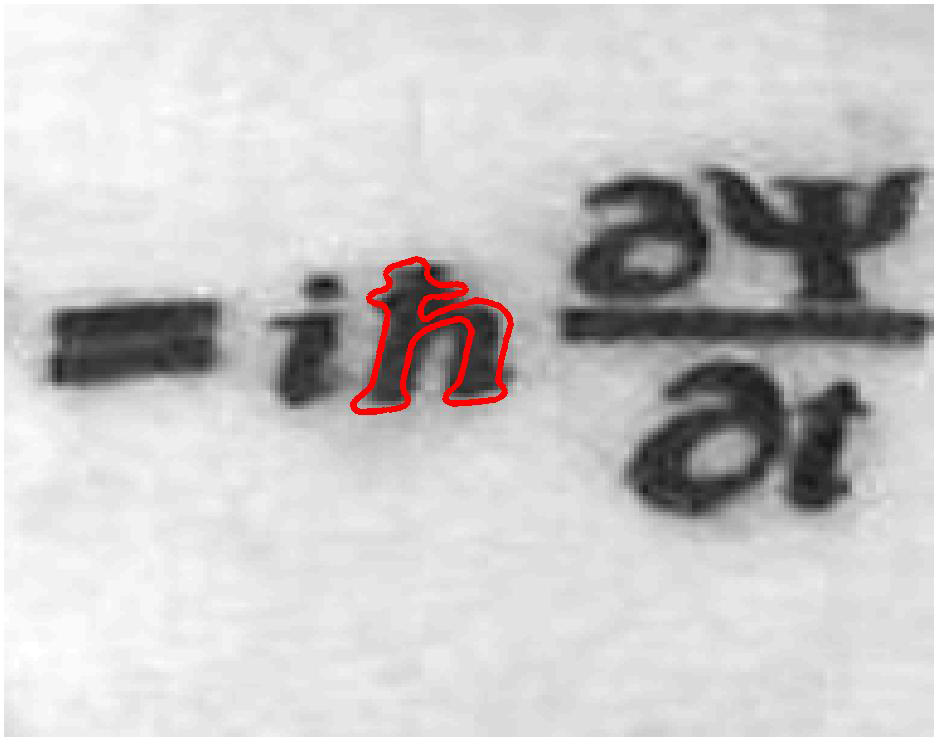}}
\caption{Shape registration and recovery with affine PVSE to find the ``$\hbar$'' in the tattoo on the skin. (a) to (e) show the curve evolution of 200 iterations until convergence at (e).}\label{Fig:Tattoo}
\end{figure*}

We also apply the PVSE to some more difficult situations. First, we apply the PVSE to shape extraction under heavy noise. We visualize the shape registration and recovery for the heavily corrupted noisy non-rigidly deformed \textbf{SHAPE} in Fig. \ref{Fig:SHAPE}. It is possible to compare the result with the ground truth in this experiment. The Jaccard coefficient for the \textbf{SHAPE} in Fig. \ref{Fig:SHAPE} is $0.73$. Despite the corner of \textbf{E}, the entire \textbf{SHAPE} has been registered and recovered well. 

We also present the results of object shape extraction from real images, such as a picture of tattoo in Fig. \ref{Fig:Tattoo} and a picture of a clownfish in Fig. \ref{Fig:clownfish}. The $\hbar$ in the tattoo has been registered and recovered by using a relatively bad initialization. Similarity PVSE has been used for this case. For the clownfish, the PVSE is applied to the red channel of this picture. There are no quantitative results for the tattoo and clownfish, as the ground-truth results are not available. For our method, the affine PVSE is used followed by a $3^{rd}$ order Prior Vibration Shape Evolution. The shape evolution process in Fig. \ref{Fig:clownfish}. The first row shows the affine PVSE, the second row shows the non-rigid PVSE during which we can observe the local shape deformation at the tail. Our method is able to register the shape satisfactorily while preserving all the local shape characteristics. If we apply the Chan-Vese model to the clownfish picture without using shape prior, we will obtain the segmentation result shown in Fig. \ref{Fig:clownfish_CV}. We also applied the shape evolution with generalized $H^1$ and $H^2$ gradients in \cite{Charpiat2007GG} for minimizing Chan-Vese active contour energy. The results are shown in Fig. \ref{Fig:clownfish_CV}. We can observe that the main shape characteristics have been corrupted during the shape evolution with generalized $H^1$ and $H^2$ gradients. Thus, these results are not satisfactory from the point of view of shape recovery. 

\begin{figure*}[!h]
\centering
\subfloat{\includegraphics[width=0.2\textwidth]{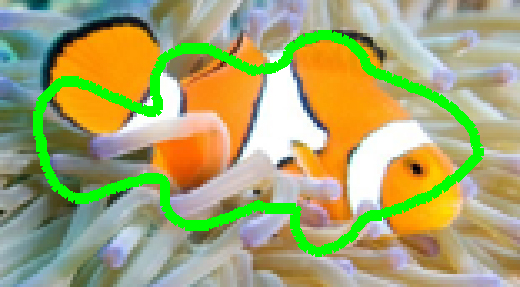}}
\subfloat{\includegraphics[width=0.2\textwidth]{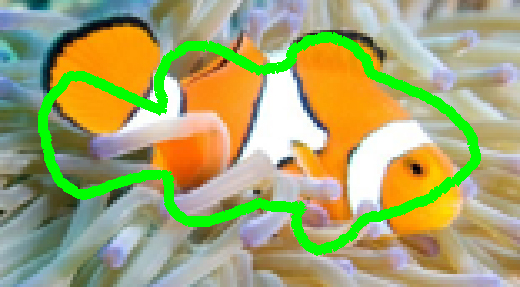}}
\subfloat{\includegraphics[width=0.2\textwidth]{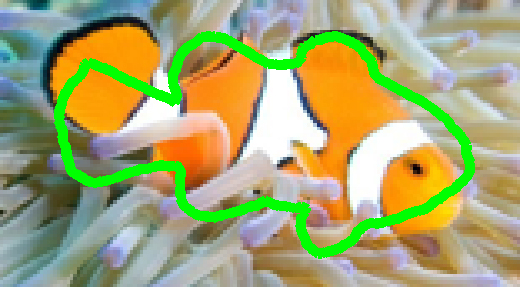}}
\subfloat{\includegraphics[width=0.2\textwidth]{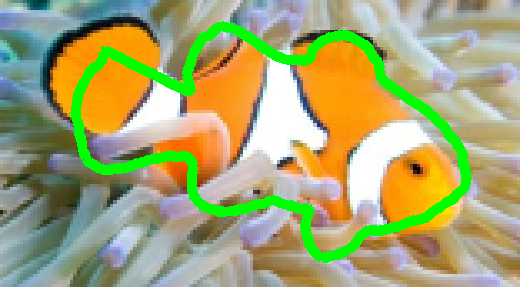}}
\subfloat{\includegraphics[width=0.2\textwidth]{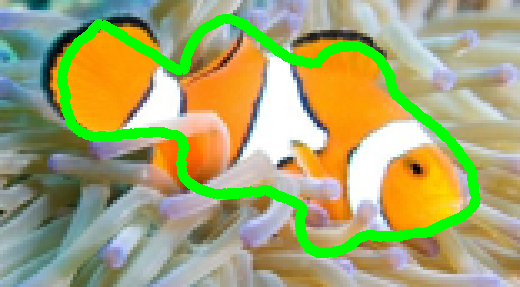}}\\
\subfloat{\includegraphics[width=0.2\textwidth]{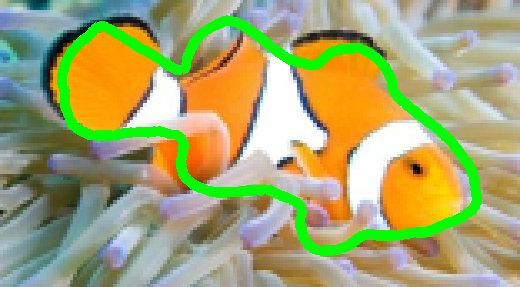}}
\subfloat{\includegraphics[width=0.2\textwidth]{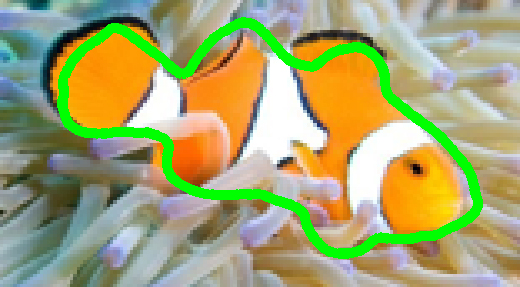}}
\subfloat{\includegraphics[width=0.2\textwidth]{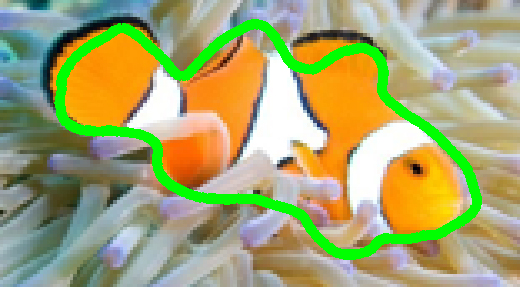}}
\subfloat{\includegraphics[width=0.2\textwidth]{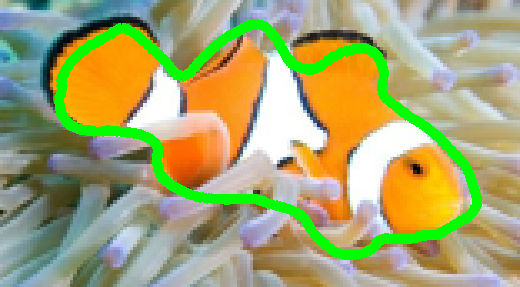}}
\subfloat{\includegraphics[width=0.2\textwidth]{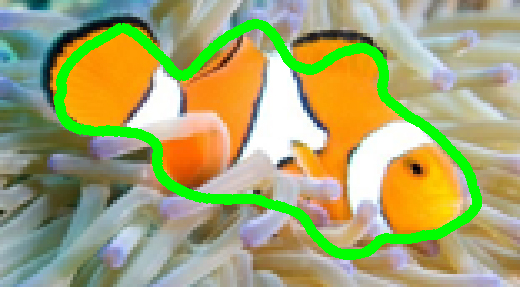}}\\
\caption{Shape registration and recovery with PVSE to find the clownfish in the coral. The first row shows the affine shape evolution until convergence. The second row shows the prior vibration shape evolution until convergence, during which we can observe the local shape deformation at the tail.}\label{Fig:clownfish}
\end{figure*}
\begin{figure*}[!h]
\centering
\subfloat{\includegraphics[width=0.2\textwidth]{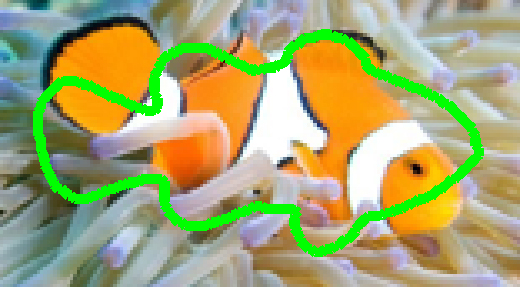}}
\subfloat{\includegraphics[width=0.2\textwidth]{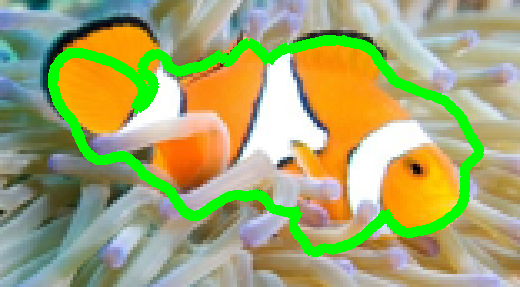}}
\subfloat{\includegraphics[width=0.2\textwidth]{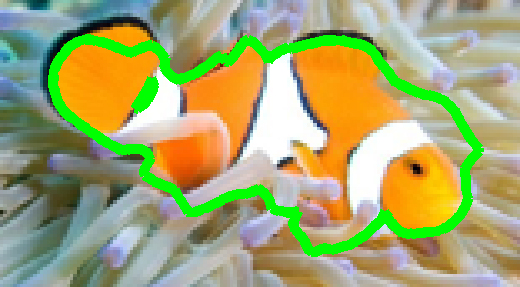}}
\subfloat{\includegraphics[width=0.2\textwidth]{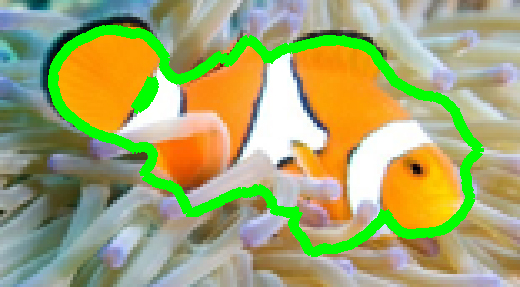}}
\subfloat{\includegraphics[width=0.2\textwidth]{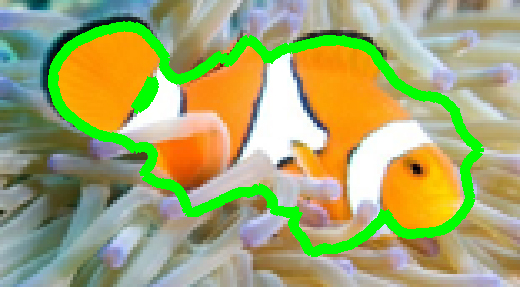}}\\
\subfloat{\includegraphics[width=0.2\textwidth]{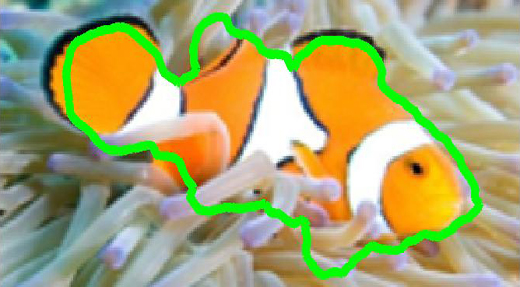}}
\subfloat{\includegraphics[width=0.2\textwidth]{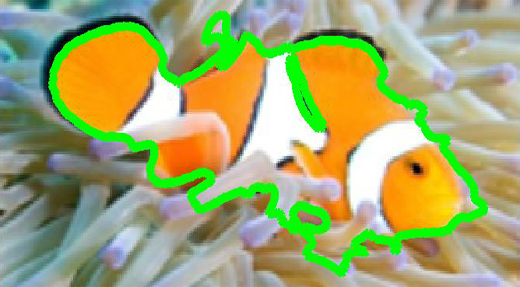}}
\subfloat{\includegraphics[width=0.2\textwidth]{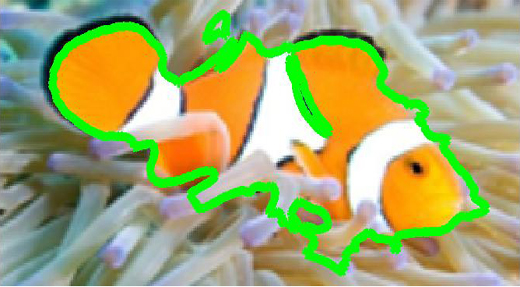}}
\subfloat{\includegraphics[width=0.2\textwidth]{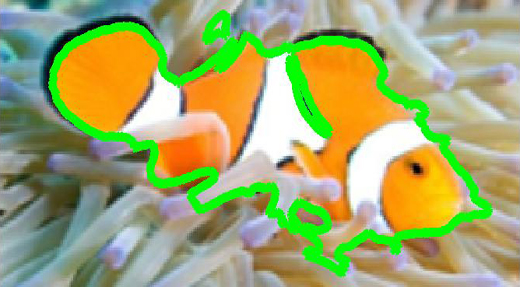}}
\subfloat{\includegraphics[width=0.2\textwidth]{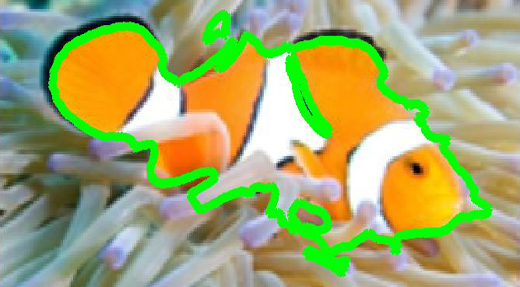}}\\
\subfloat{\includegraphics[width=0.2\textwidth]{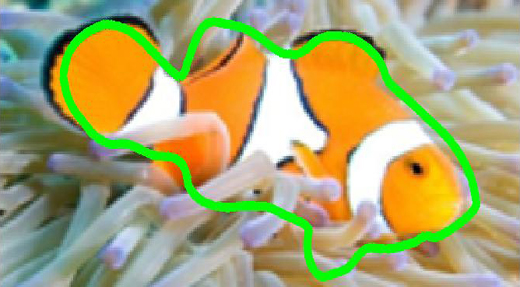}}
\subfloat{\includegraphics[width=0.2\textwidth]{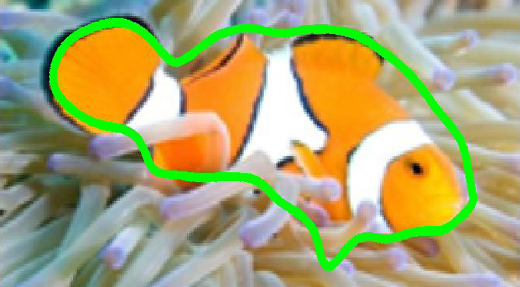}}
\subfloat{\includegraphics[width=0.2\textwidth]{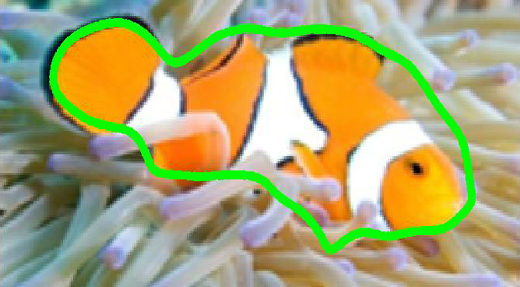}}
\subfloat{\includegraphics[width=0.2\textwidth]{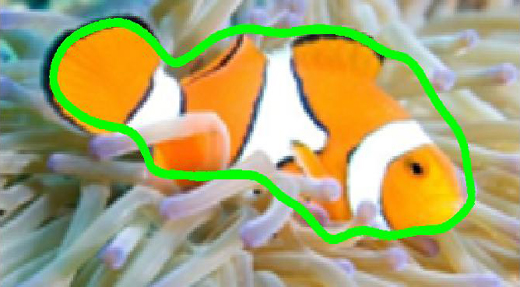}}
\subfloat{\includegraphics[width=0.2\textwidth]{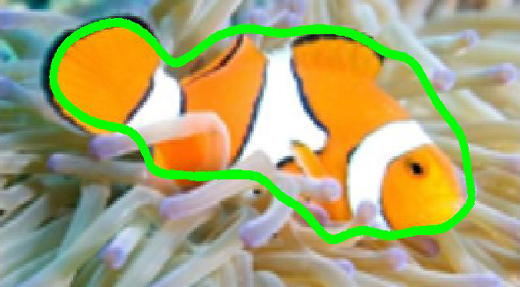}}
\caption{Comparison to other methods. The first row shows the curve evolution by the Chan-Vese active contour without shape prior. The second and third rows show the shape evolution with $H^1$ and $H^2$ gradients \cite{Charpiat2007GG} and the two methods are initialized by the converged curve of affine shape evolution.}\label{Fig:clownfish_CV}
\end{figure*}

\section{Discussions and conclusion}\label{SEC:Con}

\subsection{Discussions on the limitations}
In this work, it is assumed that the object detection has been achieved manually or automatically. Based on the detection, one may obtain an initial guess of the object shape as the initialization for the shape alignment. The rigid and non-rigid shape alignment and recovery is achieved in the case that only one sample of the object shape is available. To represent the object shape of interest, some more samples may be needed to either construct the shape model and to evaluate the shape model. To achieve shape alignment and recovery in general real images, advanced active contour models should be incorporated.

\subsection{Conclusion}
In this paper, a novel and general method is proposed for object shape extraction in images. It can handle simultaneous parts missing, shapes overlapping and shape deformation of the object of interest. The basic idea is to use the model of shape deformation to minimize the active contour energies. A curve evolution equation corresponding to the shape deformation model is derived. The curve evolution equation is the PVSE equation. The energy-minimizing PVSE equation to minimize active contour energies is derived. The novel general derivation is the calculus of prior variations. A theory of shape preservability is developed for selecting PVSE in order for shape recovery. Experiments show promising results. The developed calculus of prior variations can be easily applied for deriving PVSE equations for other shape warping models and other active contour models. The PVSE equations can be conveniently implemented by the level set method without the need of reparametrization of the contour curves.


\renewcommand{\theequation}{A.\arabic{equation}}
  \setcounter{equation}{0}  

\renewcommand{\thetheorem}{\thesection.\arabic{theorem}}
  \setcounter{theorem}{0}  
\setcounter{section}{0}
\section*{Appendix}
\renewcommand{\thesection}{A\Alph{section}}
\setcounter{subsection}{0}
\subsection{Proof of Theorem \ref{THM:Curvature_D2} and Corollary \ref{COR:F_bound}}

\begin{proof}[Proof of Theorem \ref{THM:Curvature_D2}]
Let us now consider the PVSE discretized in time as follows:
\begin{equation}
C(s,t+\Delta t)=C(s,t) + \Delta t \sum_i^n {d\theta_i\over dt}\mathbf{V}^{\theta_i}(C(s,t)),
\end{equation}
where we parameterize the curve by arclength for simplicity. Taking derivatives with respect to $s$ twice we may obtain Eq. (\ref{EQ:Kappa_PVSE}) regarding the curvature, where $C_{ss}=\kappa\mathbf{N}$, $\kappa$ is the curvature.
\begin{equation}\label{EQ:Kappa_PVSE}
\begin{split}
&C_{ss}(s,t+\Delta t)=C_{ss}(s,t)+ \Delta t \sum_i^n {d\theta_i\over dt}{\partial\over\partial s}\left(\left[{\mathbf{D}\mathbf{V}^{\theta_i}\over\mathbf{D}\mathbf{x}}\right]^{2\times2}C_s\right)\\
&=C_{ss}(s,t) \\
&+ \Delta t \sum_i^n {d\theta_i\over dt}\left\{C_s\otimes\left[{\mathbf{D}^2\mathbf{V}^{\theta_i}\over\mathbf{D}C^2}\right]^{2\times2\times2}\otimes C_s+\left[{\mathbf{D}\mathbf{V}^{\theta_i}\over\mathbf{D}\mathbf{x}}\right]C_{ss}\right\}.
\end{split}
\end{equation}


Let us consider the position $C(s,t)$, such that $\kappa(s,t)=0$. We wish to know how the $\kappa$ changes at these feature points. From Eq. (\ref{EQ:Kappa_PVSE}), we obtain the following:
\begin{equation}\label{EQ:DKappa_PVSE}
\begin{split}
&C_{ss}(s,t+\Delta t)\\
&=C_{ss}(s,t) + \Delta t \sum_i^n {d\theta_i\over dt}\left\{\mathbf{T}\otimes\left[{\mathbf{D}^2\mathbf{V}^{\theta_i}\over\mathbf{D}\mathbf{x}^2}\right]\otimes \mathbf{T}\right\},\\
\end{split}
\end{equation}
where $\mathbf{T}=C_s$ is the tangent vector of the contour.

Some rearrangements yield the expression for ${\partial C_{ss}\over\partial t}$ as follows:
\begin{equation}
{\partial C_{ss}\over\partial t} = \sum_i^n {d\theta_i\over dt}\left\{\mathbf{T}\otimes\left[{\mathbf{D}^2\mathbf{V}^{\theta_i}\over\mathbf{D}\mathbf{x}^2}\right]\otimes \mathbf{T}\right\}.
\end{equation}

Moreover, since
\begin{equation}
{\partial C_{ss}\over\partial t}={\partial \kappa\mathbf{N}\over\partial t}=\kappa_t\mathbf{N}+\underbrace{\kappa}\limits_{=0}\mathbf{N}_t=\kappa_t\mathbf{N},
\end{equation}
we have $\left\|{\partial C_{ss}\over\partial t}\right\|=\|\kappa_t\|$, which leads to the following:
\begin{equation}\label{EQ:DKappa_final}
\begin{split}
\left\|{\partial \kappa\over\partial t}\right\| &= \left\|\sum_i^n {d\theta_i\over dt}\left\{\mathbf{T}\otimes\left[{\mathbf{D}^2\mathbf{V}^{\theta_i}\over\mathbf{D}\mathbf{x}^2}\right]\otimes \mathbf{T}\right\}\right\|\\
&\leq \left\|{d\theta\over dt}\right\|\sum_i^n\left\|{\mathbf{D}^2\mathbf{V}^{\theta_i}\over\mathbf{D}\mathbf{x}^2}\right\|,
\end{split}
\end{equation}
which completes the proof.
\end{proof}

\begin{proof}[Proof of Corollary \ref{COR:F_bound}]
\begin{equation}\label{EQ:ds/dt}
\begin{split}
\left.{dp\over dt}\right|_{\kappa=0} &= \left.{dp\over ds}{ds\over d\kappa}{\partial \kappa\over\partial t}\right|_{\kappa=0}=\left.{1\over\|C_p\|\kappa_s}{\partial \kappa\over\partial t}\right|_{\kappa=0},
\end{split}
\end{equation}
where we applied $ds=\|C_p\|dp$. Note that $\left.{dp\over dt}\right|_{\kappa(p,t)=0}$ is different from the $dp\over dt$ at a fixed $p$. Rather, $\left.{dp\over dt}\right|_{\kappa(p,t)=0}$ defines on a moving $p$. Besides, $ds\over d\kappa$ is possible for nonzero $k_s$, since $\kappa:s\mapsto\kappa(s)$ is bijection.

Hence, we can establish the following bound:
\begin{equation}\label{EQ:UPB_dp/dt}
\left|{dp\over dt}\right|_{\kappa=0}\leq{1\over\|\kappa_sC_p\|}\left\|{d\theta\over dt}\right\|\sum_i^n\left\|{\mathbf{D}^2\mathbf{V}^{\theta_i}\over\mathbf{D}\mathbf{x}^2}\right\|.
\end{equation}

For the two feature points $C(p_1,t)$ and $C(p_2,t)$ at time $t$, we have the following:
\begin{equation}
\begin{split}
&\left|{d\over dt}|p_1-p_2|\right|\leq2\big|\mathrm{sign}(p_1-p_2)\big|\left|{dp\over dt}\right|_{\kappa=0}\\
&\leq{2\over\|\kappa_sC_p\|}\left\|{d\theta\over dt}\right\|\sum_i^n\left\|{\mathbf{D}^2\mathbf{V}^{\theta_i}\over\mathbf{D}\mathbf{x}^2}\right\|,
\end{split}
\end{equation}
which completes the proof.
\end{proof}

\subsection{Proof of Theorem \ref{THM:BD_PriVib}}

\begin{proof}[Proof of Theorem \ref{THM:BD_PriVib}]
To induce the norm of ${\mathbf{D}^2\mathbf{V}^\theta_{nrg}\over\mathbf{D}\mathbf{x}^2}$, we expand the derivative as follows.
\begin{equation}
\begin{split}
&\left[{\mathbf{D}^2{e}_{mn}^1\over\mathbf{D}\mathbf{x}^2},{\mathbf{D}^2{e}_{mn}^2\over\mathbf{D}\mathbf{x}^2}\right]\\
&=\left\{\left[\begin{array}{cc}
{\partial^2 {e}_{mn}^1\over\partial x^2}&{\partial^2 {e}_{mn}^1\over\partial y\partial x} \\
{\partial^2 {e}_{mn}^1\over\partial x\partial y}&{\partial^2 {e}_{mn}^1\over\partial y^2} \\
\end{array}\right],\right.\\
&\hspace{20pt}\left.\left[\begin{array}{cc}
{\partial^2 {e}_{mn}^2\over\partial x^2}&{\partial^2 {e}_{mn}^2\over\partial y\partial x} \\
{\partial^2 {e}_{mn}^2\over\partial x\partial y}&{\partial^2 {e}_{mn}^2\over\partial y^2} \\
\end{array}\right]\right\}^{[2\times2\times2]},
\end{split}
\end{equation}
where $1\leq<m\leq M$,$1\leq<n\leq N$. Therefore, we may bound the norm as follows:
\begin{equation}\label{EQ:Norm_D2V}
\begin{split}
&\left\|\left[{\mathbf{D}^2{e}_{mn}^1\over\mathbf{D}\mathbf{x}^2},{\mathbf{D}^2{e}_{mn}^2\over\mathbf{D}\mathbf{x}^2}\right]\right\|\leq\left\|{\mathbf{D}^2{e}_{mn}^1\over\mathbf{D}\mathbf{x}^2}\right\|+\left\|{\mathbf{D}^2{e}_{mn}^2\over\mathbf{D}\mathbf{x}^2}\right\|\\
&=\left\|\left[\begin{array}{cc}
{\partial^2 {e}_{mn}^1\over\partial x^2}&{\partial^2 {e}_{mn}^1\over\partial y\partial x} \\
{\partial^2 {e}_{mn}^1\over\partial x\partial y} & {\partial^2 {e}_{mn}^1\over\partial y^2} \\
\end{array}\right]\right\|+\left\|\left[\begin{array}{cc}
{\partial^2 {e}_{mn}^2\over\partial x^2}&{\partial^2 {e}_{mn}^2\over\partial y\partial x} \\
{\partial^2 {e}_{mn}^2\over\partial x\partial y}&{\partial^2 {e}_{mn}^2\over\partial y^2} \\
\end{array}\right]\right\|.
\end{split}
\end{equation}
Hence, we may consider to bound $\left\|{\mathbf{D}^2{e}_{mn}^1\over\mathbf{D}\mathbf{x}^2}\right\|_2,\left\|{\mathbf{D}^2{e}_{mn}^2\over\mathbf{D}\mathbf{x}^2}\right\|$ separately. For $\left\|{\mathbf{D}^2{e}_{mn}^1\over\mathbf{D}\mathbf{x}^2}\right\|$, if we choose the $\|\cdot\|=\|\|_\infty$, we have the following bound:
\begin{equation}
\begin{split}
&\left\|{\mathbf{D}^2{e}_{mn}^1\over\mathbf{D}\mathbf{x}^2}\right\| = \left\|\left[\begin{array}{cc}
{\partial^2 {e}_{mn}^1\over\partial x^2}&{\partial^2 {e}_{mn}^1\over\partial y\partial x} \\
{\partial^2 {e}_{mn}^1\over\partial x\partial y}&{\partial^2 {e}_{mn}^1\over\partial y^2} \\
\end{array}\right]\right\|\\
&=\max\left(\left|{\partial^2 {e}_{mn}^1\over\partial x^2}+{\partial^2 {e}_{mn}^1\over\partial y\partial x}\right|,\left|{\partial^2 {e}_{mn}^1\over\partial x\partial y}+{\partial^2 {e}_{mn}^1\over\partial y^2}\right|\right)\\
&\leq\max\left({n^2+mn\over n^2+m^2},{m^2+mn\over n^2+m^2}\right)\\
&\leq {2\over1+c}\leq 2,
\end{split}
\end{equation}
where $c=m/n$ if $m\leq n$ and $c=n/m$ if $n<m$.
Hence,
\begin{equation}\label{EQ:Ineq_emn2}
\sum\limits_{mn}\left\|{\mathbf{D}^2e^i_{mn}\over\mathbf{D}\mathbf{x}^2}\right\|=\sum\limits_{m=1}^M\sum\limits_{n=1}^N\left\|{\mathbf{D}^2e^i_{mn}\over\mathbf{D}\mathbf{x}^2}\right\|\leq 2MN,
\end{equation}
for $i=1,2$. The above together with (\ref{EQ:Norm_D2V}) completes the proof.
\end{proof}

{
\bibliographystyle{IEEEtran}
\bibliography{LevelSetActiveContours,MRFseg}
}

\end{document}